%% file: ms.tex
\documentclass[twoside]{article}

\usepackage[accepted]{aistats2019}
\usepackage[round]{natbib}

\usepackage{amsmath}
\usepackage{amssymb}
\usepackage{amsthm}
\usepackage{amsfonts}
\usepackage{enumitem}
\usepackage{bm}
\usepackage{bbm}
\usepackage{hhline}
\usepackage{mathtools}
\usepackage{stmaryrd}
\usepackage{tikz}
\usepackage{graphicx}
\usepackage{subfigure}
\usepackage{algorithm}
\usepackage{algorithmic}
\usepackage{booktabs}
\usepackage{mathrsfs}
\usepackage[titlenumbered,ruled,noend,algo2e]{algorithm2e}

\SetCommentSty{mycommfont}
\SetEndCharOfAlgoLine{}
\usepackage{hyperref}
\usepackage{cleveref}
\hypersetup{
    colorlinks,
    linkcolor={red!75!black},
    citecolor={blue!75!black},
    urlcolor={blue!80!black}
}

\def\H{\mathcal{H}}
\def\X{\mathcal{X}}
\def\Y{\mathcal{Y}}
\def\Z{\mathcal{Z}}
\def\Xin{\X_0}
\def\Xhid{\X_1}
\def\Hin{\H_1}
\def\Hhid{\H_2}
\def\K{\mathcal{K}}
\def\Kin{\K_1}
\def\Khid{\K_2}

\def\risk{\epsilon}

\DeclarePairedDelimiter\norm{\lVert}{\rVert}
\DeclarePairedDelimiter\inner{\langle}{\rangle}

\newtheorem{theorem}{Theorem}
\newtheorem{lemma}[theorem]{Lemma}
\newtheorem{definition}[theorem]{Definition}
\newtheorem{assumption}[theorem]{Assumption}

\newenvironment{proof_of}[1]{
\renewcommand{\proofname}{Proof of #1}\proof}{\endproof}

\begin{document}



\runningauthor{ Pierre Laforgue, Stephan Cl\'emen\c con,
                Florence d'Alch\'e-Buc }

\twocolumn[

\aistatstitle{Autoencoding any Data through Kernel Autoencoders}

\aistatsauthor{Pierre Laforgue \And Stephan Cl\'emen\c con \And Florence d'Alch\'e-Buc}
\medskip

% \aistatsauthor{ Author 1 \And Author 2 \And  Author 3 }

\aistatsaddress{\And {\sc LTCI, T\'el\'ecom Paris, Institut Polytechnique de Paris}\\[0.1cm] Contact: \href{mailto:pierre.laforgue1@gmail.com}{\texttt{pierre.laforgue1@gmail.com}}\And}]

% \aistatsaddress{ Institution 1 \And  Institution 2 \And Institution 3 } ]

\begin{abstract}
    This paper investigates a novel algorithmic approach to data representation based on kernel methods.
    Assuming that the observations lie in a Hilbert space $\mathcal{X}$, the introduced Kernel Autoencoder (KAE) is the composition of mappings from vector-valued Reproducing Kernel Hilbert Spaces (vv-RKHSs) that minimizes the expected reconstruction error.
    Beyond a first extension of the auto-encoding scheme to possibly infinite dimensional Hilbert spaces, KAE further allows to autoencode any kind of data by choosing $\mathcal{X}$ to be itself a RKHS.
    A theoretical analysis of the model is carried out, providing a generalization bound, and shedding light on its connection with Kernel Principal Component Analysis.
    The proposed algorithms are then detailed at length: they crucially rely on the form taken by the minimizers, revealed by a dedicated Representer Theorem.
    Finally, numerical experiments on both simulated data and real labeled graphs (molecules) provide empirical evidence of the KAE performances.
\end{abstract}

\input{1-Introduction}

\input{2-Problem}

\input{3-Theory}

\input{4-Algorithm}

\input{5-Experiments}

\input{6-Conclusion}

% \clearpage

\paragraph{Acknowledgment} This work has been funded by the Industrial Chair \emph{Machine Learning for Big Data} from T\'el\'ecom ParisTech, Paris, France.
\bibliography{ref}

\appendix
\input{7-Appendix}

\end{document}

%% file: 1-Introduction.tex
\section{INTRODUCTION}
As experienced by any practitioner, data representation is critical to the application of Machine Learning, whatever the targeted task, supervised or unsupervised.
An answer to this issue consists in feature engineering, a step that requires time-consuming interactions with domain experts.
To overcome these limitations, Representation Learning (RL) \citep{Bengio2013} aims at building automatically new features in an unsupervised fashion.
Recent applications to neural nets pre-training, image denoising and semantic hashing have renewed a strong interest in RL, now a proper research field.
Among successful RL approaches, mention has to be made of Autoencoders (AEs) \citep{Vincent2010}, and their generative variant, Deep Boltzman Machines \citep{Salakhutdinov2009}.

AEs attempt to learn a pair of encoding/decoding functions under structural constraints so as to capture the most important properties of the data \citep{Alain2014}.
If they have mostly been studied under the angle of neural networks \citep{baldi2012autoencoders} and deep architectures \citep{Vincent2010}, the concepts underlying AEs are very general and go beyond neural implementations.
In this work, we develop a general framework inspired from AEs, and based on Operator-Valued Kernels (OVKs) \citep{Senkene1973} and vector-valued Reproducing Kernel Hilbert Spaces (vv-RKHSs).
Mainly developed for supervised learning, OVKs provide a nonparametric way to tackle complex output prediction problems \citep{Alvarez2012}, including multi-task regression, structured output prediction \citep{Brouard2016}, or functional regression \citep{Kadri2016}.
This work is a first contribution to combine OVKs with AEs, enlarging the latters' applicability scope - so far restricted to $\mathbb{R}^d$~- to any data described by a similarity matrix.

We start from the simplest formulation in which a Kernel Autoencoder (KAE) is a pair of encoding/decoding functions lying in two different vv-RKHSs, and whose composition approximates the identity function.
This approach is further extended to a general framework involving the composition of an arbitrary number of mappings, defined and valued on Hilbert spaces.
A crucial application of KAEs arises if the input space is itself a RKHS: it allows to perform autoencoding on any type of data, by first mapping it to the RKHS, and then applying a KAE.
%
% As a kernelization of the KAE, this application will be denoted $\text{K}^2$AE.
%
The solutions computation, even in infinite dimensional spaces, is made possible by a Representer Theorem and the use of the kernel trick.
This unlocks new applications on structured objects for which feature vectors are missing or too complex (\textit{e.g.} in chemoinformatics).

Kernelizing an AE criterion has also been proposed by \citet{gholami2016kernel}.
But their approach differs from ours in many key aspects: it is restricted to AEs with 2 layers and composed of linear maps only; it relies on semi-supervised information; it comes with no theoretical analysis, and within a hashing perspective solely.
Despite a similar title, the work by \citet{kampffmeyer2017deep} has no connection with ours.
Authors use standard AEs, and regularize the learning by aligning the latent code with some predetermined feature map.
In the experimental section, we implement autoencoding on graphs, which cannot be done by means of standard AEs.
Graph AEs \citep{kipf2016variational} do not autoencode graphs, but $\mathbb{R}^d$ points with an additive graph characterizing the data structure.

The rest of the article is structured as follows.
The novel kernel-based framework for RL is detailed in \Cref{sec:background}.
A generalization bound and a strong connection with Kernel PCA are established
in \Cref{sec:theory}, whereas \Cref{sec:algo} describes the algorithmic approach based on a Representer Theorem.
Illustrative numerical experiments are displayed in \Cref{sec:experiments}, while concluding remarks are collected in \Cref{sec:conclusion}.
Finally, technical details are deferred to the Supplementary Material.

%% file: 2-Problem.tex
\section{THE KERNEL AUTOENCODER}\label{sec:background}
In this section, we introduce a general framework for building AEs based on vv-RKHSs.
Here and throughout the paper, the set of bounded linear operators mapping a vector space $E$ to itself is denoted by $\mathcal{L}(E)$, and the set of mappings from a set $A$ to an ensemble $B$ by $\mathcal{F}(A,B)$.
The adjoint of an operator $M$ is denoted by $M^*$.
Finally, $\llbracket n \rrbracket$ denotes the set $\{1, \ldots, n\}$ for any integer $n \in \mathbb{N}^*$.

\subsection{Background on vv-RKHSs}
Vv-RKHSs allow to cope with the approximation of functions from an input set $\X$ to some output Hilbert space $\Y$ \citep{Senkene1973, Caponnetto2008}.
Vv-RKHS can be defined from an OVK, which extends the classic notion of positive definite kernel.
An OVK is a function $\K \colon \X \times \X \to \mathcal{L}(\Y)$, that satisfies the following two properties \citep{micchelli2005learning}:
\begin{equation*}\label{eq:ovk_1}
\forall (x, x') \in \X \times \X, \qquad \K(x, x') = \K(x', x)^*,
\end{equation*}
and $\forall n \in \mathbb{N}^*, \forall~\{(x_i, y_i)\}_{1\le i \le n} \in (\X \times \Y)^n,$
\begin{equation*}\label{eq:ovk_2}
\sum_{1\le i, j \le n} \langle y_i, \K(x_i, x_j)y_j \rangle_{\Y} \ge 0.
\end{equation*}
A simple example of OVK is the {\it separable kernel} such that: $\forall~(x, x') \in \X \times \X,$ $\K(x,x')=k(x,x')A$, where $k$ is a positive definite scalar-valued kernel, and $A$ is a positive semi-definite operator on $\Y$.
Its relevance for multi-task learning has been highlighted for instance by \citet{micchelli2005learning}.

Let $\K$ be an OVK, and for $x \in \X$, let $K_{x} \colon y \in \Y \mapsto K_{x}y \in \mathcal{F}(X, Y)$ the linear operator such that:
$$
\forall x' \in \X,~(K_{x}y)(x') = \K(x', x)y.
$$
Then, there is a unique Hilbert space $\H_\K \subset \mathcal{F}(\mathcal{X}, \mathcal{Y})$ called the vv-RKHS associated to $\K$, with inner product $\inner{\cdot,\cdot}_{\H_\K}$ and norm $\norm{\cdot}_{\H_\K}$, such that $\forall x \in \X$:
\begin{itemize}
\item $K_x$ spans the space $\H_K$ ($\forall y \in \Y \colon K_x y \in \H_K$)
\item $K_x$ is bounded for the uniform norm
\item $\forall f \in \H,~f(x) = K_{x}^* f$ ({\it i.e.} reproducing property)
\end{itemize}

\subsection{Input Output Kernel Regression}\label{sec:IOKR}
Now, let us assume that the output space $\Y$ is chosen itself as a RKHS, say $\H$, associated to the positive definite scalar-valued kernel $k: \Z \times \Z \to \mathbb{R}$, with $\Z$ a non-empty set.
Working in the vv-RKHS $\H_\K$ associated to an OVK $\K: \X \times \X \to \mathcal{L}(\H)$ opens the door to a large family of learning tasks where the output set $\Z$ can be a set of complex objects such as nodes in a graph, graphs \citep{Brouard-ismb2016} or functions \citep{Kadri2016}.
Following the work of \citet{Brouard2016}, we refer to these methods as Input Output Kernel Regression (IOKR).
IOKR has been shown to be of special interest in case of Ridge Regression, where closed-form solutions are available besides classical gradient descent algorithms.
Note that in a general supervised setting, learning a function $f \in \H_\K$ is not sufficient to provide a prediction in the output set, and a pre-image problem has to be solved.
In sections \ref{subsec:OKAE} and \ref{sec:inf_dim}, a similar idea is applied at the last layer of our KAE, allowing for auto-encoding non-vectorial data while avoiding complex pre-image problems.

\subsection{The 2-layer Kernel Autoencoder (KAE)}
\begin{figure}[!t]
\begin{center}
% .95 to .90
\includegraphics[width=0.90\columnwidth, page=1]{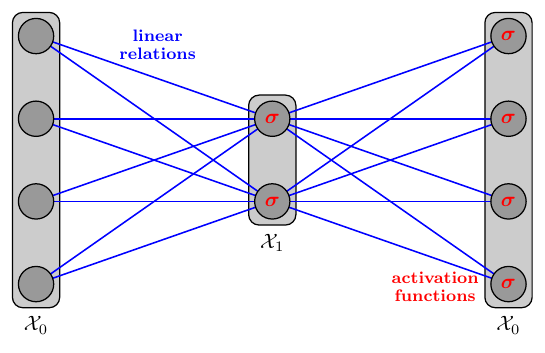}\\
\includegraphics[width=0.90\columnwidth, page=2]{fig/Autoencoder.pdf}
\caption{Standard and Kernel 2-layer Autoencoders}
\label{kAE}
\end{center}
\end{figure}
Let $S = (x_1,\; \ldots,\; x_n)$ denote a sample of $n$ independent realizations of a random vector $X$, valued in a separable Hilbert space $(\Xin,\; \|\cdot\|_{\Xin})$ with unknown probability distribution $P$, and such that there exists $M < +\infty, \|X\|_{\mathcal{X}_0} \le M$ almost surely.
On the basis of the training sample $S$, we are interested in constructing a pair of encoding/decoding mappings $\left(f: \Xin \rightarrow \X_1,\; g: \X_1 \rightarrow \Xin \right)$, where $(\X_1,\; \|\cdot\|_{\X_1})$ is the (Hilbert) \textit{representation space}.
Just as for standard AEs, we regard as good internal representations the ones that allow for an accurate recovery of the original information in expectation.
The problem to be solved states as follows:
\begin{equation}\label{eq:risk}
\min_{\substack{(f, g) \in \Hin \times \Hhid\\\|f\|_{\mathcal{H}_1} \le s,~\|g\|_{\mathcal{H}_2}\le t}}  \risk(f, g) \coloneqq \mathbb{E}_{X\sim P} \left\| X - g \circ f(X) \right\|^2_{\Xin},
\end{equation}
where $\mathcal{H}_1$ and $\mathcal{H}_2$ are two vv-RKHSs, and $s$ and $t$ two positive constants.
$\mathcal{H}_1$ is associated to an OVK $\Kin: \Xin \times \Xin \rightarrow \mathcal{L}(\Xhid)$, while $\mathcal{H}_2$ is associated to $\Khid: \Xhid \times \Xhid \rightarrow \mathcal{L}(\Xin)$.
Figure \ref{kAE} illustrates the parallel and differences between standard and kernel 2-layer Autoencoders.

Following the Empirical Risk Minimization (ERM) paradigm, the true risk \eqref{eq:risk} is replaced by its empirical version
\begin{equation}\label{eq:emp_risk}
\hat{\epsilon}_n(f,g) \coloneqq\frac 1 n \sum_{i=1}^n \vert\vert x_i-g\circ f(x_i)\vert\vert^2_{\Xin},
\end{equation}
and a penalty term $\Omega(f,g) \coloneqq \lambda \|f\|_{\Hin}^2 + \mu \|g\|_{\Hhid}^2$ is added instead of the norm constraints (see \Cref{thm:reformulation}). Solutions to the following regularized ERM problem shall be referred to as \textit{2-layer~KAE}:
\begin{equation}\label{eq:emp_pb}
\min_{(f, g) \in \Hin \times \Hhid} \hat{\epsilon}_n(f, g) + \Omega(f,g).
\end{equation}

\subsection{The Multi-layer KAE}\label{subsec:general_framework}
Like for standard AEs, the model previously described can be directly extended to more than 2 layers.
Let $L \geq 3$, and consider a collection of Hilbert spaces $\mathcal{X}_0,\; \ldots,\; \mathcal{X}_L$, with $\mathcal{X}_L = \mathcal{X}_0$.
For $0 \le l \le L-1$, the space $\mathcal{X}_l$ is supposed to be endowed with an OVK $\K_{l+1}: \mathcal{X}_l \times \mathcal{X}_l \rightarrow \mathcal{L}\left(\mathcal{X}_{l+1}\right)$, associated to a vv-RKHS $\mathcal{H}_{l+1} \subset \mathcal{F}(\mathcal{X}_l,\mathcal{X}_{l+1})$.
We then want to minimize $\epsilon(f_1, \ldots, f_L)$ over $\prod_{l=1}^L \mathcal{H}_l$.
Setting $\Omega(f_1, \ldots, f_L) \coloneqq \sum_{l=1}^L \lambda_l \|f_l\|_{\mathcal{H}_L}^2$ allows for a direct extension of \eqref{eq:emp_pb}:
\begin{equation}\label{eq:MLVVKAE_pbm}
\min_{f_l \in \mathcal{H}_l}~\frac 1 n \sum_{i=1}^n \left\| x_i-f_L \circ \ldots \circ f_1(x_i) \right\|^2_{\mathcal{X}_0} + \sum_{l=1}^L \lambda_l \|f_l\|_{\mathcal{H}_l}^2.
\end{equation}

\subsection{The General Hilbert KAE and the $\text{K}^2$AE}\label{subsec:OKAE}
\begin{figure}[!b]
\begin{center}
\includegraphics[width=0.95\columnwidth, page=4]{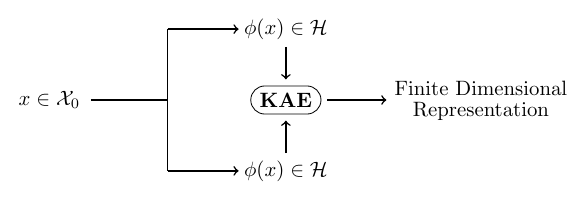}
\caption{Autoencoding any data thanks to $\text{K}^2$AE}
\label{fig:okae}
\end{center}
\end{figure}
So far, and up to the regularization term, the main difference between standard and kernel AEs is the function space on which the reconstruction criterion is optimized: respectively neural functions or RKHS ones.
But what should also be highlighted is that RKHS functions are valued in general Hilbert spaces, while neural functions are restricted to $\mathbb{R}^d$.
As shall be seen in section \ref{sec:inf_dim}, this enables KAEs to handle data from infinite dimensional Hilbert spaces (\textit{e.g.} function spaces), what standard AEs are unable to do.
To our knowledge, this first extension of the autoencoding scheme is novel.

But even more interesting is the possible extension when the input/output Hilbert space is chosen to be itself a RKHS.
Indeed, let $\mathcal{X}_0$ denote now any set (without the Hilbert assumption).
In the spirit of IOKR, let us first map $x \in \mathcal{X}_0$ to the RKHS $\mathcal{H}$ associated to some scalar kernel $k$, and its canonical feature map $\phi$.
Since the $\phi(x_i)'s$ are by definition valued in a Hilbert, KAE can be applied.
This way, we have extended the autoencoding paradigm to any set, and finite dimensional representations can be extracted from all types of data.
Again, such extension is novel to our knowledge.
Figure \ref{fig:okae} depicts the procedure, referred to as $\text{K}^2$AE, since the new criterion is a \textit{kernelization} of the KAE that reads:
\begin{equation}\label{eq:OVVKAE_pbm}
\frac 1 n \sum_{i=1}^n \left\| \phi(x_i)-f_L \circ \ldots \circ f_1(\phi(x_i)) \right\|^2_{\mathcal{H}} + \sum_{l=1}^L \lambda_l \|f_l\|_{\mathcal{H}_l}^2.
\end{equation}

%% file: 3-Theory.tex
\section{THEORETICAL ANALYSIS}\label{sec:theory}
It is the purpose of this section to investigate theoretical properties of the introduced model, its capacity to be learnt from training data with a controlled generalization error, and the connection between $\text{K}^2$AE and Kernel PCA (KPCA) namely.
%
%The contribution is two fold.
%
%First, a generalization bound is derived.
%
%Then, we emphasize on the link between $\text{K}^2$AE and Kernel Principal Component Analysis (KPCA).

\subsection{Generalization Bound}
While the algorithmic formulation aims at minimizing the regularized risk \eqref{eq:emp_pb}, the subsequent theoretical analysis focuses on the constrained problem \eqref{eq:risk}.
\Cref{thm:reformulation} relates the solutions from the two approaches to each other, so that bounds derived in the latter setting also apply to numerical solutions of the first one.
\begin{theorem}\label{thm:reformulation}
Let $V: \mathcal{H}_1 \times \ldots \times \mathcal{H}_L \rightarrow \mathbb{R}$ be an arbitrary function.
Consider the two problems:
\begin{equation}\label{eq:penalized_pb}
\min_{f_l \in \mathcal{H}_l}~\left\{V(f_1, \ldots, f_L) + \sum_{l=1}^L \lambda_l \|f_l\|_{\mathcal{H}_l}^2\right\},
\end{equation}
%
% and
%
\begin{equation}\label{eq:constrained_pb}
\min_{\substack{f_l \in \mathcal{H}_l\\\|f_l\|_{\mathcal{H}_l} \leq s_l}}~V(f_1, \ldots, f_L).
\end{equation}
Then, for any $(\lambda_1,\; \ldots,\; \lambda_L)\in \mathbb{R}_+^L$, there exists $(s_1,\; \ldots,\; s_L)\in \mathbb{R}_+^L$ such that any (respectively, local) solution to problem \eqref{eq:penalized_pb} is also a (respectively, local) solution to problem \eqref{eq:constrained_pb}.
\end{theorem}
Refer to \Cref{sec:Eq_pbm_proof} for the proof and a discussion on the converse statement.

In order to establish generalization bound results for empirical minimizers in the present setting,
we now define two key quantities involved in the proof, \textit{i.e.} Rademacher and Gaussian averages for classes of Hilbert-valued functions.
\begin{definition}
Let $\mathcal{X}$ be any measurable space, and $H$ a separable Hilbert space. Consider a class $\mathcal{C}$ of measurable functions $h: \mathcal{X} \rightarrow H$. Let $\bm{\sigma}_1,\; \ldots,\; \bm{\sigma}_n$ be $n\geq 1$ independent $H$-valued Rademacher variables and define:
\begin{equation*}\label{eq:Rademacher}
\widehat{\mathscr{R}}_n(\mathcal{C}(S)) = \mathbb{E}_{\bm{\sigma}}\left[ \sup_{h \in \mathcal{C}} \frac{1}{n} \sum_{i=1}^n \left\langle \bm{\sigma}_i, h(x_i)\right\rangle_H\right].
\end{equation*}
\end{definition}
If $H = \mathbb{R}$, it is the classical Rademacher average (see \textit{e.g.} \citet{mohri2012foundations} p.34), while, when $H = \mathbb{R}^p$, it corresponds to the expectation of the supremum of the sum of the Rademacher averages over the $p$ components of $h$ (see Definition 2.1 in \citet{maurer2016bounds}).
If $H$ is an infinite dimensional Hilbert space with countable orthonormal basis $(e_k)_{k \in \mathbb{N}}$, we have:
\begin{equation*}
\widehat{\mathscr{R}}_n(\mathcal{C}(S)) = \mathbb{E}_{\bm{\sigma}}\left[ \sup_{h \in \mathcal{C}} \frac{1}{n} \sum_{i=1}^n \sum_{k=1}^\infty \sigma_{i, k} \left\langle h(x_i), e_k\right\rangle_H\right].
\end{equation*}
The Gaussian counterpart of $\widehat{\mathscr{R}}_n(\mathcal{C}(S))$, obtained by replacing Rademacher random variables/processes with standard $H$-valued Gaussian ones, is denoted by $\widehat{\mathscr{G}}_n(\mathcal{C}(S))$ throughout the paper.

For the sake of simplicity, results in the rest of the subsection are derived in the 2-layer case solely, with $\mathcal{X}_1$ finite dimensional (\textit{i.e.} $\mathcal{X}_1 = \mathbb{R}^p$), although the approach remains valid for deeper architectures.

Let $\mathcal{H}_{1, s} \coloneqq \left\{f \in \mathcal{H}_1: \|f\|_{\mathcal{H}_1} \le s\right\}$, and similarly $\mathcal{H}_{2, t} \coloneqq \left\{g \in \mathcal{H}_2: \|g\|_{\mathcal{H}_2} \le t,\sup_{y \in \mathbb{R}^p}\|g(y)\|_{\mathcal{X}_0} \le M\right\}$.
We shall use the notation
$\mathcal{H}_{s, t} \subset \mathcal{F}(\mathcal{X}_0, \mathcal{X}_0)$ to mean the space of composed functions $\mathcal{H}_{1, s} \circ \mathcal{H}_{2, t} \coloneqq \left\{ h \in \mathcal{F}(\mathcal{X}_0, \mathcal{X}_0): \exists(f,g) \in \mathcal{H}_{1, s} \times \mathcal{H}_{2, t},~h = g \circ f\right\}$. To simplify the notation, $\epsilon$ (and $\hat{\epsilon}_n$) may be abusively considered as a functional with one or two arguments: $\epsilon(f, g) = \epsilon(g \circ f) = \mathbb{E}_{X \sim P}\left\| X - g \circ f (X)\right\|_{\mathcal{X}_0}^2$. Finally, let $\hat{h}_n$ denote the minimizer of $\hat{\epsilon}_n$ over $\mathcal{H}_{s, t}$, and $\epsilon^*$ the infimum of $\epsilon$ on the same functional space.

% The following assumptions on $\mathcal{K}_1$ and $\mathcal{K}_2$ are needed to establish the bound stated below.
%
\begin{assumption}\label{hyp:trace_1}
There exists $K<+\infty$ such that:
$$
\forall x \in \mathcal{X}_0,\;\; \text{\normalfont \textbf{Tr}}\big(\mathcal{K}_1(x, x)\big) \le Kp.
$$
\end{assumption}

\begin{assumption}\label{hyp:trace_2}
There exists $L<+\infty$ such that for all $y,\;  y'$ in $\mathbb{R}^p$:
$$
\text{\normalfont \textbf{Tr}}\big( \mathcal{K}_2(y, y) - 2 \mathcal{K}_2(y, y') + \mathcal{K}_2(y', y') \big) \le L^2~\|y - y'\|_{\mathbb{R}^p}^2.
$$
\end{assumption}

\begin{theorem}\label{thm:gen_bound}
Let $\mathcal{K}_1$ and $\mathcal{K}_2$ be OVKs satisfying Assumptions \ref{hyp:trace_1} and \ref{hyp:trace_2} respectively. Then, there exists a universal constant $C_0<+\infty$ such that, for any $0 < \delta < 1$, we have with probability at least $1 - \delta$:
\begin{equation*}\label{eq:gen_bound}
\epsilon(\hat{h}_n) - \epsilon^* \le C_0 L M s t \sqrt{\frac{K p}{n}} + 24M^2 \sqrt{\frac{\log(2)/\delta}{2n}}.
\end{equation*}
\end{theorem}
%

% The exess of risk of the empirical minimizer can be studied by the classic bound
% $\epsilon(\hat{h}_n) - \epsilon^*\leq 2 \sup_{h \in \mathcal{H}_{s, t}} |\hat{\epsilon}_n(h) - \epsilon(h)|$.
% %
% Standard arguments (McDiarmid's inequality, symmetrization) allow to upper bound with high probability this latter quantity in terms of Rademacher complexity.
% %
% Corollary 4 in \citet{maurer2016vector} then permit to restrict the analysis to $\widehat{\mathscr{G}}_n(\mathcal{H}_{s, t}(S))$.
% %
% By means of an extension of Theorem 2 in \citet{maurer2014chain}, proved in the Supplementary Material, $\widehat{\mathscr{G}}_n((\mathcal{H}_{1, s} \circ \mathcal{H}_{2, t})(S))$ is further decomposed into quantities depending only on $\mathcal{H}_{1, s}$ or $\mathcal{H}_{2, t}$.
% %
% Finally, each individual term is bounded using the stipulated assumptions.
% %
% Technical details are deferred to \cref{sec:lemma_proof}.
%
The proof relies on a Rademacher bound, which is in turn upper bounded using Corollary 4 in \citet{maurer2016vector}, an extension of Theorem 2 in \citet{maurer2014chain} proved in the Supplementary Material, and several intermediary results derived from the stipulated assumptions.
Technical details are deferred to \Cref{sec:lemma_proof}.

Attention should be paid to the fact that constants in \Cref{thm:gen_bound} appear in a very interpretable fashion: the less spread the input (the smaller the constant $M$), the more restrictive the constraints on the functions (the smaller $K$, $L$, $s$ and $t$), and the smaller the internal dimension $p$, the sharper the bound.

%%%%%%%%%%%%%%%%%%%%%%%%
%  LINK WITH THE KPCA  %
%%%%%%%%%%%%%%%%%%%%%%%%

\subsection{$\text{K}^2$AE and Kernel PCA: a Connection}
Just as \citet{bourlard1988auto} have shown a mere equivalence between PCA and standard 2-layer AEs, a similar link can be established between 2-layer $\text{K}^2$AE and Kernel PCA \citep{scholkopf1997kernel,scholkopf1998nonlinear}.
Throughout the analysis, a 2-layer $\text{K}^2$AE is considered, with decomposable kernels made of linear scalar kernels and identity operators.
Also, there is no penalization (\textit{i.e.} $\lambda_1 = \lambda_2 = 0$).
We want to autoencode data into $\mathbb{R}^p$, after a first embedding through the feature map $\phi$, like in \eqref{eq:OVVKAE_pbm}.

% Finite dim

\subsubsection{Finite Dimensional Feature Map}
Let us assume first that $\phi$ is valued in $\mathbb{R}^d$, with \mbox{$p < d < n$}.
Let $\Phi = (\phi(x_1), \ldots, \phi(x_n))^\top \in \mathbb{R}^{n \times d}$ denote the matrix storing the $\phi(x_i)^\top$ to autoencode in rows.
Note that $K_\phi = \Phi \Phi^\top \in \mathbb{R}^{n \times n}$ corresponds to the Gram matrix associated to $\phi$.
As shall be seen in \Cref{thm:RT}, the optimal $f$ and $g$ have a specific form, so that they only depend on two coefficient matrices, $A \in \mathbb{R}^{n \times p}$ and $B \in \mathbb{R}^{n \times d}$ respectively.
Equipped with this notation, one has: $Y = f_A(\Phi) = \Phi \Phi^\top A \in \mathbb{R}^{n \times p}$, and $\widetilde{\Phi} = g_B(Y) = YY^\top B \in \mathbb{R}^{n \times d}$.
Without penalization, the goal is then to minimize in $A$ and $B$:
\begin{equation}\label{eq:fro}
\|\Phi - \widetilde{\Phi}\|_\text{Fr}^2.
\end{equation}
$\widetilde{\Phi}$ being at most of rank $p$, we know from Eckart-Young Theorem that the best possible $\widetilde{\Phi}$ is given by $\Phi^* = U \Sigma_p V^\top$, where $\left(U \in \mathbb{R}^{n \times d}, \Sigma \in \mathbb{R}^{d \times d}, V^\top \in \mathbb{R}^{d \times d}\right)$ is the \textit{thin} Singular Value Decomposition (SVD) of $\Phi$ such that $\Phi = U \Sigma V^\top$, and $\Sigma_p$ is equal to $\Sigma$, but with the $d -p$ smallest singular values zeroed.

Let us now prove that there exists a couple of matrices $(A^*, B^*)$ such that $g_{B^*} \circ f_{A^*}(\Phi) = \Phi^*$.
One can verify that $( A^* = U_p {\overline{\Sigma}_p}^{-3/2}, B^* = U V^\top)$, with $U_p \in \mathbb{R}^{n \times p}$ storing only the $p$ largest eigenvectors of $K_\phi$, and $\overline{\Sigma}_p \in \mathbb{R}^{p \times p}$ the $p \times p$ top left block of $\Sigma_p$, satisfy it.
Finally, the optimal encoding returned is $Y^* = f_{A^*}(\Phi) = \left(\sqrt{\sigma_1}u_1, \ldots, \sqrt{\sigma_p}u_p \right)$, with $u_1, \ldots, u_p$ the $p$ largest eigenvectors of $K_\phi$, while the KPCA's new representation is $\left(\sigma_1 u_1, \ldots, \sigma_p u_p \right)$.
Precisely, this $1/\sqrt{\sigma_i}$ rescaling may be seen as a KPCA whitening, and the encoding returned by KAE is actually known as the Kernel PCA Map (see \textit{e.g.} Section 2.2.6 in \citet{smola1998learning}).
Notice also that the algorithmic resolution of Problem \eqref{eq:fro} under different structural constraints (low rank assumption, sparsity) is studied in \citet{smola2000sparse}.

We have shown that a specific instance of $\text{K}^2$AE can be solved explicitly using a SVD, and that the optimal coding returned is close to the one output by KPCA.

% Infinite dim

\subsubsection{Infinite Dimensional Feature Map}
Let us assume now that $\phi$ is valued in a general Hilbert space $\mathcal{H}$.
$\Phi$ is now seen as the linear operator from $\mathcal{H}$ to $\mathbb{R}^n$ such that $\forall \alpha \in \mathcal{H},~\Phi\alpha = \left(\left\langle \alpha, \phi(x_1)\right\rangle_\mathcal{H}, \ldots, \left\langle \alpha, \phi(x_n)\right\rangle_\mathcal{H}\right) \in \mathbb{R}^n$.
Since \Cref{thm:reformulation} makes no assumption on the dimensionality, everything stated in the finite dimensional scenario applies, except that $B \in \mathcal{L}(\mathcal{H}, \mathbb{R}^n)$, and that we minimize the Hilbert-Schmidt norm:  $\|\Phi - \widetilde{\Phi}\|_{HS}^2$.
We then need an equivalent of Eckart-Young Theorem.
It still holds since its proof only requires the existence of an SVD for any operator, which is granted in our case since we deal with compact operators (they have finite rank $n$).
The end of the proof is analogous to the finite dimensional case.

%% file: 4-Algorithm.tex
\section{THE KAE ALGORITHMS} \label{sec:algo}
This section describes at length the algorithms we propose to solve problems \eqref{eq:MLVVKAE_pbm} and \eqref{eq:OVVKAE_pbm}.
They raise two major issues as their objective functions are non-convex, and their search spaces are infinite dimensional.
However, this last difficulty is solved by \Cref{thm:RT}.

%%%%%%%%%%%%%%%%%
%  REPRESENTER  %
%%%%%%%%%%%%%%%%%

\subsection{A Representer Theorem}
\begin{theorem}\label{thm:RT}
Let $L_0 \in \llbracket L \rrbracket$, and $V: \mathcal{X}_{L_0}^n \times \mathbb{R}_+^{L_0} \rightarrow \mathbb{R}$ a~function of $n + L_0$ variables, strictly increasing in each of its $L_0$ last arguments.
Suppose that $(f_1^*, \ldots, f_{L_0}^*)$ is a solution to the optimization problem:
\begin{align*}
\min_{f_l \in \mathcal{H}_l}V\Big(& (f_{L_0} \circ \ldots \circ f_1)(x_1), \ldots, (f_{L_0}\circ\ldots\circ f_1)(x_n),\\
& \left\| f_1 \right\|_{\mathcal{H}_1}, \ldots, \left\|f_{L_0}\right\|_{\mathcal{H}_{L_0}}\Big).
\end{align*}
Let ${x_i^*}^{(l)} \coloneqq f_l^* \circ \ldots \circ f_1^*(x_i)$, with ${x_i^*}^{(0)} \coloneqq x_i$.
Then, $\exists \left(\varphi_{1, 1}^*, \ldots, \varphi_{1, n}^*, \ldots, \varphi_{L_0, n}^* \right) \in \mathcal{X}_1^n \times \ldots \times \mathcal{X}_{L_0}^n:$
\begin{equation*}
\forall~ l \in \llbracket L_0 \rrbracket,\quad f_l^*(\cdot) = \sum_{i=1}^n \mathcal{K}_l\left(~\cdot~,{x_i^*}^{(l-1)}\right) \varphi_{l, i}^*.
\end{equation*}
\end{theorem}
\begin{proof}
Refer to \Cref{sec:RT_proof}
\end{proof}
%
% This Theorem is quite general as condition on $V$ is not restrictive.
%
% Therefore, it can be applied equivalently to problem \eqref{eq:MLVVKAE_pbm} or \eqref{eq:OVVKAE_pbm}, and regardless of the number of composed functions considered.
This Theorem exhibits a very specific structure for the minimizers, as each layer's support vectors are the images of the original points by the previous layer.

%%%%%%%%%%%%%%%%
%  FINITE DIM  %
%%%%%%%%%%%%%%%%

\subsection{Finite Dimension Case}\label{sec:finite_dim}
In this section, let us assume that $\mathcal{X}_l = \mathbb{R}^{d_l}$ for $l \in \llbracket L \rrbracket$.
The objective function of \eqref{eq:MLVVKAE_pbm}, viewed as a function of $(f_L \circ \ldots \circ f_1)(x_1), \ldots, (f_L \circ \ldots \circ f_1)(x_n), \|f_1\|_{\mathcal{H}_1}, \ldots, \|f_L\|_{\mathcal{H}_L}$ satisfies the condition on $V$ involved in \Cref{thm:RT}.
After applying it (with $L_0 = L$), problem \eqref{eq:MLVVKAE_pbm} boils down to the problem of finding the $\varphi_{l, i}^*$'s, which are finite dimensional.
This crucial observation shows that our problem can be solved in a computable manner.
However, its convexity still cannot be ensured (see \Cref{sec:non_cvx}).

The objective only depending on the $\varphi_{l, i}$'s, problem \eqref{eq:MLVVKAE_pbm} can be approximately solved by Gradient Descent (GD).
We now specify the gradient derivation in the decomposable OVKs case, \textit{i.e.} for any layer $l$ there exists a scalar kernel $k_l$ and $A_l \in \mathcal{L}(X_l)$ positive semidefinite such that $\mathcal{K}_l(x, x') = k_l(x, x')A_l$.
All detailed computations can be found in \Cref{sec:grad_detail}.
Let $\Phi_l \coloneqq (\varphi_{l, 1}, \ldots, \varphi_{l, n})^\top \in \mathbb{R}^{n \times d_l}$ storing the coefficients $\varphi_{l, i}$ in rows, and $K_l \in \mathbb{R}^{n \times n}$ such that $\left[K_l\right]_{i, i'} = k_l\Big(x_i^{(l-1)}, x_{i'}^{(l-1)}\Big)$.
Let $(l_0, i_0) \in  \llbracket L \rrbracket \times \llbracket n \rrbracket$, the gradient of the distortion term reads:

\begin{align}\label{eq:disto_grad}
\Bigg(\nabla_{\varphi_{l_0, i_0}}&~\frac{1}{n}\sum_{i=1}^n \|x_i - f_L \circ \ldots \circ f_1(x_i)\|_{\mathcal{X}_0}^2\Bigg)^\top\\
= &~-\frac{2}{n}\sum_{i=1}^n \left(x_i - {x_i}^{(L)}\right)^\top\textbf{Jac}_{{x_i}^{(L)}}(\varphi_{l_0, i_0}).\nonumber
\end{align}
%
% Assuming the matrices $\textbf{Jac}_{{x_i}^{(L)}}(\varphi_{l_0, i_0})$ are known (their computations are detailed in section \ref{sec:jacs}), the distortion part of the gradient is computable.
%
On the other hand, $\|f_l\|_{\mathcal{H}_l}^2$ may be rewritten as:
\begin{equation}\label{eq:norms_detail}
\|f_l\|_{\mathcal{H}_l}^2 = \sum_{i, i' = 1}^n k_l\left({x_i}^{(l-1)}, {x_{i'}}^{(l-1)}\right)\left\langle \varphi_{l, i},A_l\varphi_{l, i'}\right\rangle_{\mathcal{X}_l},
\end{equation}
so that it may depend on $\varphi_{l_0, i_0}$ in two ways:
1) if $l_0=l$, there is a direct dependence of the second quadratic term,
2) but note also that for $l_0 < l$, the $\varphi_{l_0, i}$ have an influence on the ${x_i}^{(l-1)}$ and so on the first term.
This remark leads to the following formulas:
\begin{equation}\label{eq:own_norm}
\nabla_{\Phi_l}~\|f_l\|_{\mathcal{H}_l}^2 = 2~K_l \Phi_l A_l,
\end{equation}
with $\nabla_{\Phi_l}F \coloneqq \left((\nabla_{\varphi_{l, 1}}F)^\top,\dots, (\nabla_{\varphi_{l, n}}F)^\top\right)^\top \in \mathbb{R}^{n \times d_l}$ storing the gradients of any real-valued function $F$ with respect to the $\varphi_{l, i}$ in rows.
And when $l_0 < l$:
\begin{align}\label{eq:cross_norm}
&\Big(\nabla_{\varphi_{l_0, i_0}}\|f_l\|_{\mathcal{H}_l}^2\Big)^\top = 2\sum_{i,i'=1}^n \bigg\{\\
&[N_l]_{i, i'}\left(\nabla^{(1)} k_l\left({x_i}^{(l-1)}, {x_{i'}}^{(l-1)}\right)\right)^\top\textbf{Jac}_{{x_i}^{(l-1)}}(\varphi_{l_0, i_0})\bigg\},\nonumber
\end{align}
where $\nabla^{(1)} k_l\left(x, x'\right)$ denotes the gradient of $k_l(\cdot,\cdot)$ with respect to the $1^{st}$ coordinate evaluated~in $\left(x, x'\right)$, and $N_l$ the $n \times n$ matrix such that $[N_l]_{i, i'} = \left\langle \varphi_{l, i}~,~A_l~\varphi_{l, i'}\right\rangle_{\mathcal{X}_l}$.
Again, assuming the matrices $\textbf{Jac}_{{x_i}^{(L)}}(\varphi_{l_0, i_0})$ are known, the norm part of the gradient is computable.
Combining expressions \eqref{eq:disto_grad}, \eqref{eq:own_norm} and \eqref{eq:cross_norm} using the linearity of the gradient leads readily to the complete formula.

If $n$, $L$, and $p$ denote respectively the number of samples, the number of layers, and the size of the largest latent space, the algorithm complexity is no more than $\mathcal{O}(n^2 L p)$ for objective evaluation, and $\mathcal{O}(n^3 L^2 p^3)$ for gradient derivation.
Hence, it appears natural to consider stochastic versions of GD.
But as shown by equation \eqref{eq:cross_norm}, the norms gradients involve the computation of many Jacobians.
Selecting a mini-batch does not affect these terms, which are the most time consuming.
Thus, the expected acceleration due to stochasticity must not be so important.
Nevertheless, a \textit{doubly stochastic} scheme where both the points on which the objective is evaluated, as well as the coefficients to be updated, are chosen randomly at each iteration, might be of high interest since it would dramatically decrease the number of Jacobians computed.
However, this approach goes beyond the scope of this paper, and is left for future work.

%%%%%%%%%%%%%%%%%%
%  INFINITE DIM  %
%%%%%%%%%%%%%%%%%%

\subsection{General Hilbert Space Case}\label{sec:inf_dim}
In this section, $\mathcal{X}_0$ (and so $\mathcal{X}_L$) are supposed to be infinite dimensional.
Despite this relaxation, KAEs remains computable.
As \Cref{thm:RT} makes no assumption on the dimensionality of $\mathcal{X}_0$, it can be applied.
The only difference is that coefficients $\varphi_{L, i}$'s $\in \mathcal{X}_L^n$ are infinite dimensional, preventing from the use of a global GD.
But assuming the $\varphi_{L, i}$'s to be fixed, a GD can still be performed on the $\varphi_{l, i}$'s, $l \in \llbracket L - 1 \rrbracket$.
On the other hand, if one assumes these coefficients fixed, the optimal $\varphi_{L, i}$'s are the solutions to a Kernel Ridge Regression (KRR).
Consequently, a hybrid approach alternating GD and KRR is considered.
Two issues remain to be addressed:
1) how to compute the KRR in $\mathcal{X}_L$,
2) how to propagate the gradients through $\mathcal{X}_L$.

From now, $A_L$ is assumed to be the identity operator.
If the $\varphi_{l, i}$'s, $l \in \llbracket L - 1 \rrbracket$ are fixed, then the best $\varphi_{L,i}$'s shall satisfy \citep{micchelli2005learning} for all $i \in \llbracket n \rrbracket$:
\begin{equation}\label{eq:KRR}
\sum_{i'=1}^n \left(\mathcal{K}_L\left({x_i}^{(L-1)},{x_{i'}}^{(L-1)}\right) + n\lambda_L \delta_{ii'}\right)\varphi_{L, i'} = x_i.
\end{equation}
In particular, the computation of $N_L$ becomes explicit (\Cref{sec:lin_sys}) as long as we know the dot products $\left\langle x_j, x_{j'}\right\rangle_{\mathcal{X}_0}$.
In the case of the $\text{K}^2$AE, these dot products are the input Gram matrix $K_{in}$.
Let $N_{\text{KRR}}$ be the function that computes $N_L$ from the $\varphi_{l, i}$'s, $l \in \llbracket L - 1 \rrbracket$, and $K_{in}$.
What is remarkable is that knowing $N_L$ (and not each $\varphi_{L, i}$ individually) is enough to propagate the gradient through the infinite dimensional layer.

Indeed, let us assume now that $N_L$ is fixed.
All spaces but $\mathcal{X}_L$ remaining finite dimensional, changes in the gradients only occur where the last layer is involved, namely for the distortion and for $\|f_L\|_{\mathcal{H}_L}^2$.
As for the gradients of $\|f_L\|_{\mathcal{H}_L}^2$, equation \eqref{eq:cross_norm} remain true.
If $N_L$ is given, there is no difficulty.
As for the distortion, the use of the differential (see \Cref{sec:differential}) gives:
\begin{align}\label{eq:differential}
&{\nabla_{\varphi_{l_0, i_0}}\left\|x_i - x_i^{(L)}\right\|_{\mathcal{X}_0}^2}^\top = -2 \sum_{i'=1}^n \bigg\{\\
&\Big\langle x_i - x_i^{(L)}, \varphi_{L, i'}\Big\rangle_{\mathcal{X}_L} \left(\nabla_{\varphi_{l_0, i_0}}k_L\left(x_i^{(L-1)},x_{i'}^{(L-1)}\right)\right)^\top\bigg\}.\nonumber
\end{align}
It is a direct extension of \eqref{eq:disto_grad}, where $\textbf{Jac}_{{x_i}^{(L)}}(\varphi_{l_0, i_0})$, has been replaced using the definition of ${x_i}^{(L)}$.
Using again \eqref{eq:KRR}, $\langle x_i - x_i^{(L)}, \varphi_{L, i'}\rangle_{\mathcal{X}_L}$ can be rewritten as $n\lambda_L\left\langle\varphi_{L, i} , \varphi_{L, i'}\right\rangle_{\mathcal{X}_L} = n \lambda_L [N_L]_{i, i'}$, and infinite dimensional objects are dealt with.
The crux of the algorithm is that infinite dimensional coefficients $\varphi_{L,i}$'s are never computed, but only their scalar products.
Not knowing the $\varphi_{L,i}$'s is of no importance, as we are interested in the encoding function, which does not rely on them.
Let $T$ be a number of epochs, and $\gamma_{t}$ a step size rule, the approach is summarized in \Cref{alg:md}.

\begin{algorithm}[t]
\SetKwInOut{Input}{input}
\SetKwInOut{Init}{init}
\SetKwInOut{Parameter}{Param}
\caption{General Hilbert KAE and $\text{K}^2$AE}
\Input{Gram matrix $K_{in}$}
\Init{$\Phi_1 = \Phi_1^{init}, \ldots, \Phi_{L-1} = \Phi_{L-1}^{init},$\\
      $N_L = N_{KRR}\left(\Phi_1, \ldots, \Phi_{L-1}, K_{in}\right)$}
\For{epoch $t$ from $1$ to $T$}{
    \tcp{inner coefficients gradient update}
    \For{layer $l$ from $1$ to $L-1$}{
        $\Phi_l = \Phi_l - \gamma_{t}~\nabla_{\Phi_{l}}\left(\hat{\epsilon}_n + \Omega~|~N_L, K_{in}\right)$
        }
    \tcp{outer coefficient dot products update}
    $N_L = N_{KRR}\left(\Phi_1, \ldots, \Phi_{L-1}, K_{in}\right)$
    }
\Return{$\Phi_1, \ldots, \Phi_{L-1}$}
\label{alg:md}
\end{algorithm}

%% file: 5-Experiments.tex
\section{NUMERICAL EXPERIMENTS}\label{sec:experiments}
Numerical experiments have been run in order to assess the ability of KAEs to provide relevant data representations.
We used decomposable OVKs with the identity operator as $A$, and the Gaussian kernel as $k$.
First, we present insights on the interesting properties of the KAEs via a 2D example.
Then, we describe more involved experiments on the NCI dataset to measure the power of KAEs.
\begin{figure*}[!ht]
\vspace{-0.5cm}
\centering
\subfigure[2D Input Data\label{fig:cc_inp}]{\includegraphics[width=0.325\textwidth]{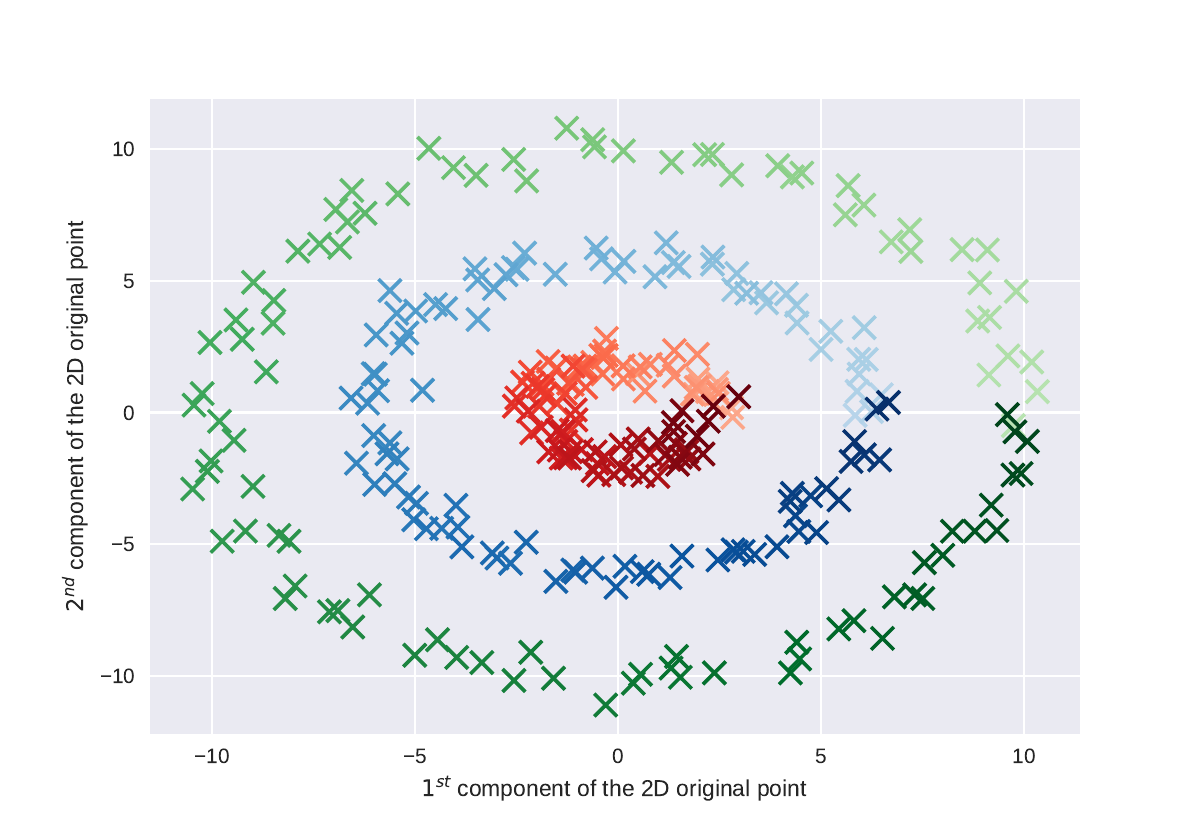}}
\subfigure[AE Reconstruction\label{fig:ae_reco}]{\includegraphics[width=0.325\textwidth]{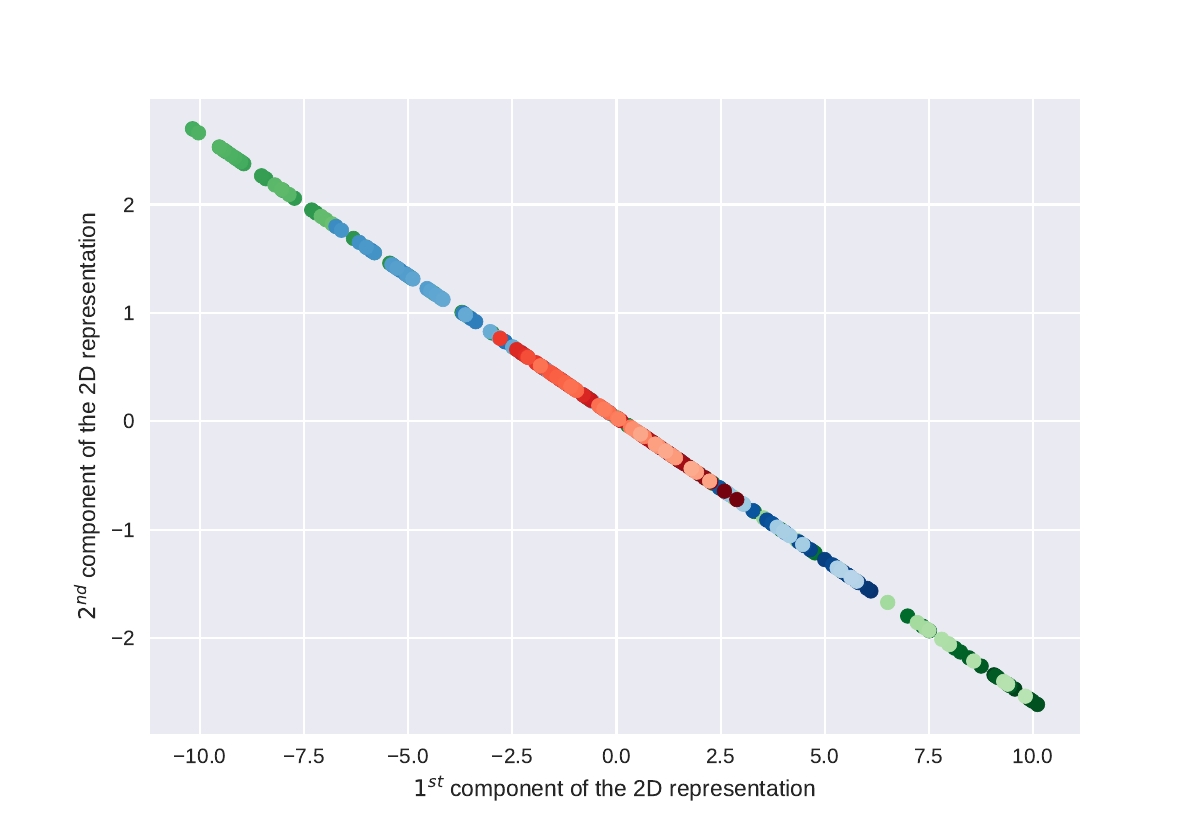}}
\subfigure[KAE Reconstruction\label{fig:kae_reco}]{\includegraphics[width=0.325\textwidth]{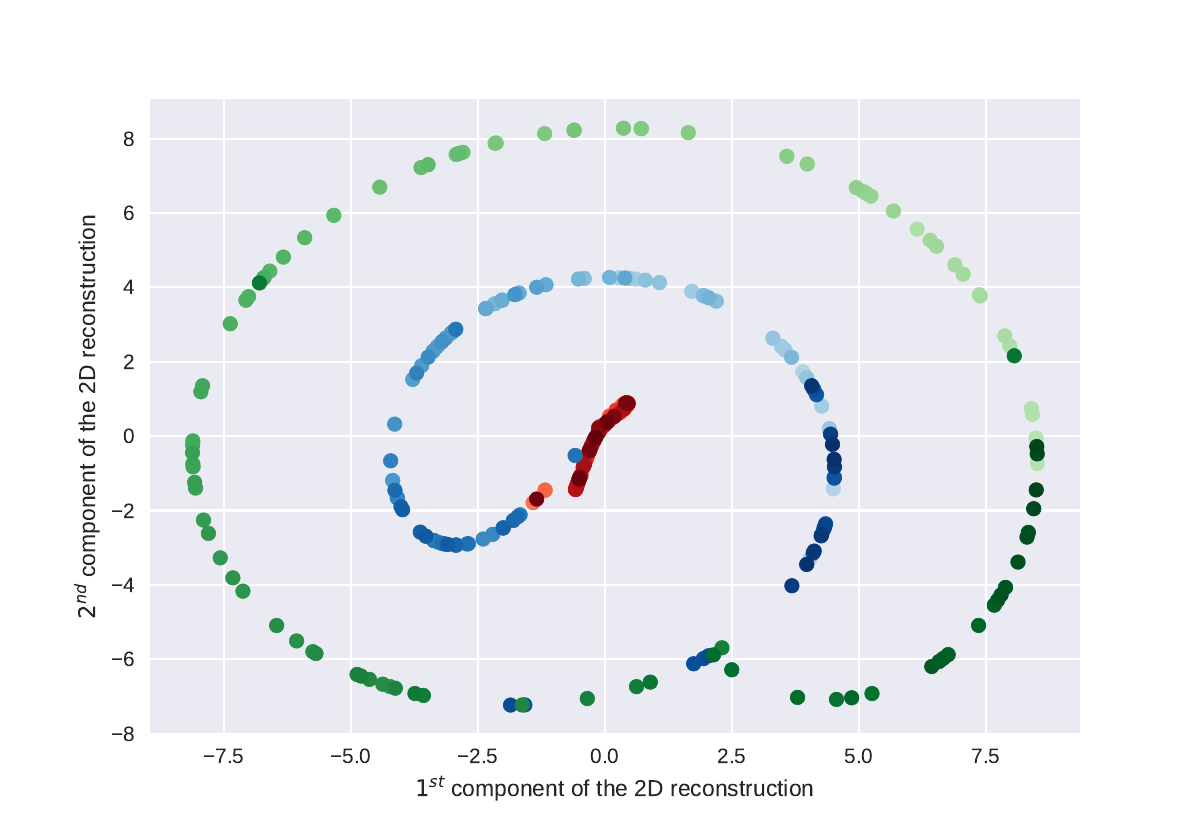}}\\
\subfigure[1D Representation (5 it.)\label{fig:cc_1D_enc0}]{\includegraphics[width=0.325\textwidth]{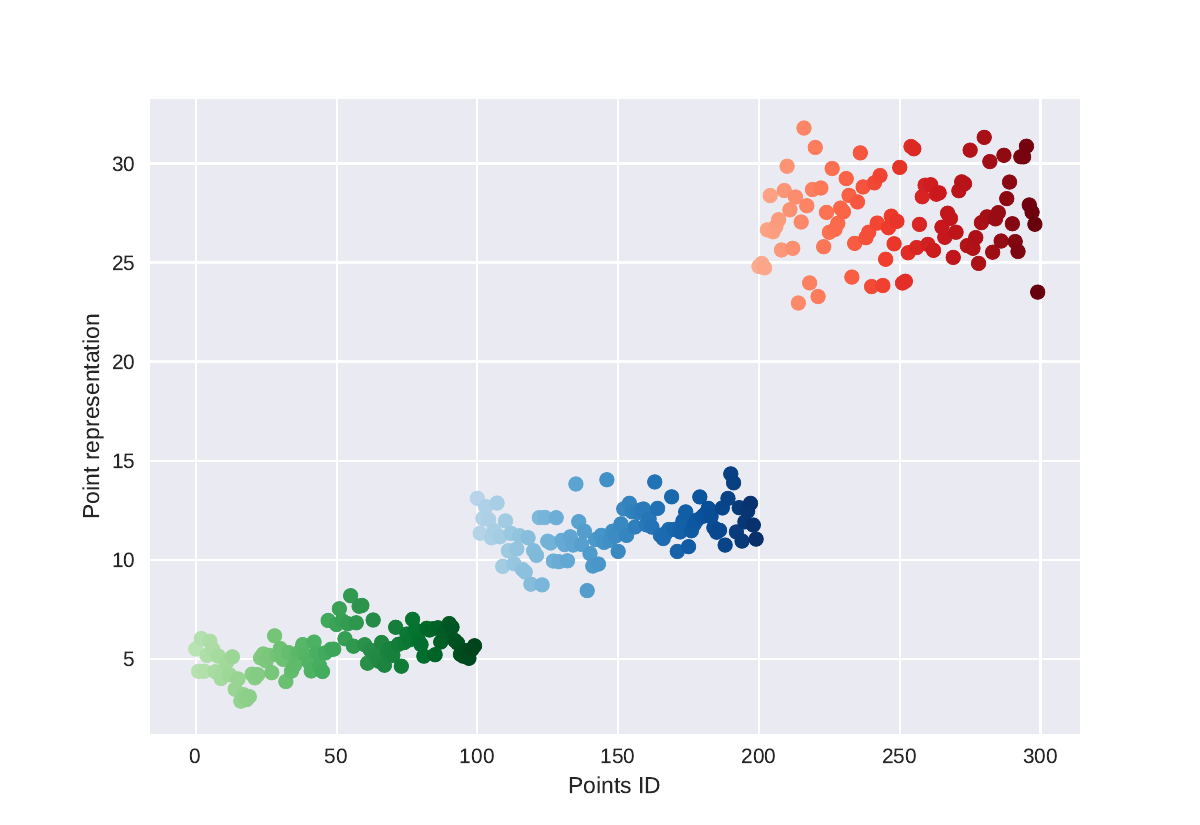}}
\subfigure[1D Representation (20 it.)\label{fig:cc_1D_enc1}]{\includegraphics[width=0.325\textwidth]{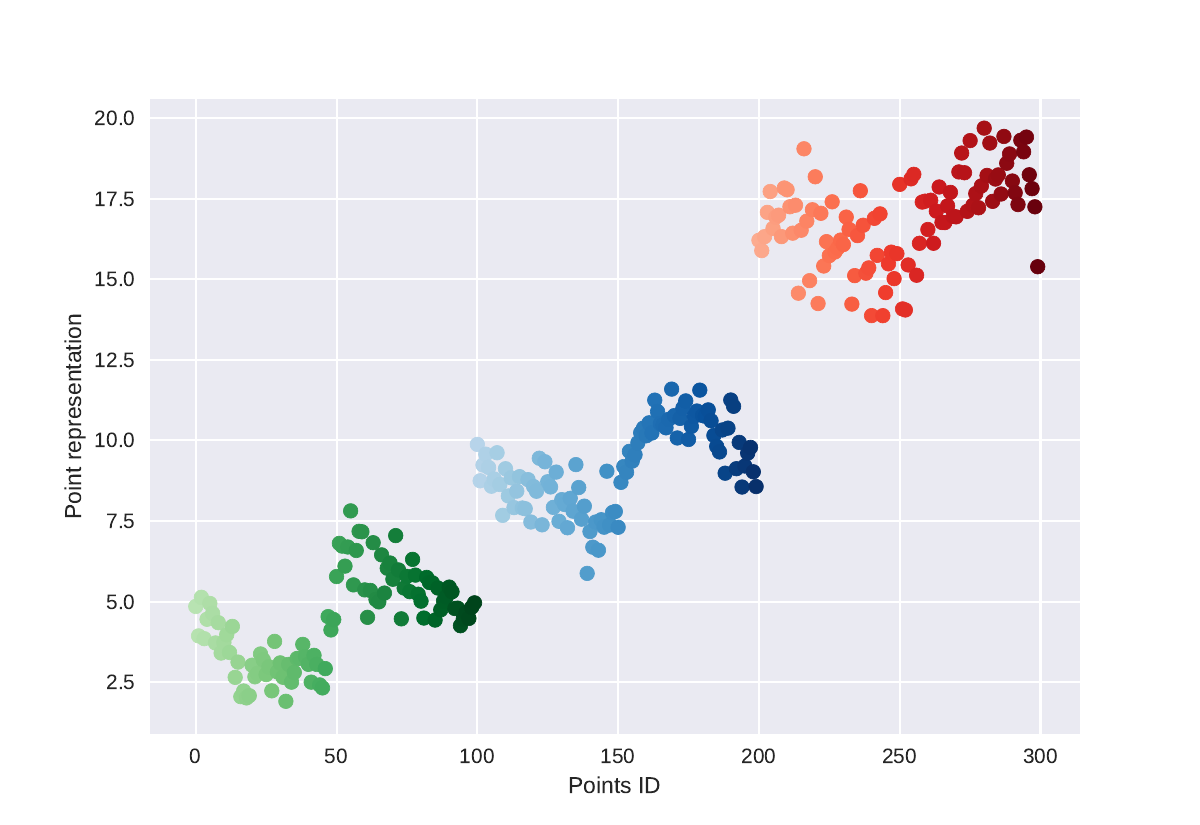}}
\subfigure[2D Re-Representation\label{fig:cc_2D_out}]{\includegraphics[width=0.325\textwidth]{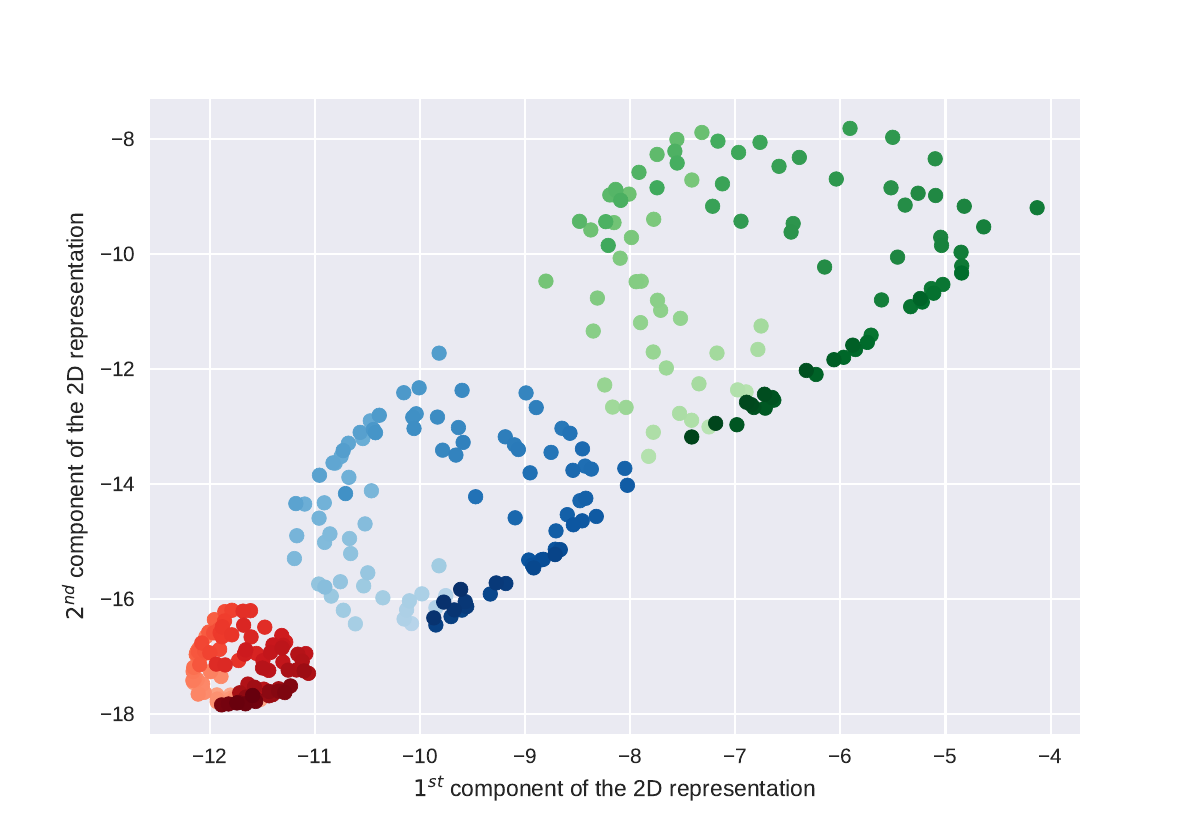}}
\caption{KAE Performance on Noisy Concentric Circles}
\label{fig:conc_circles}
\end{figure*}

\subsection{Behavior on a 2D problem}
Let us first consider three noisy concentric circles such as in \Cref{fig:cc_inp}.
Although the main strength of KAEs is to perform autoencoding on complex data (\Cref{subsec:OKAE}), they can still be applied on real-valued points.
Figures \ref{fig:ae_reco} and \ref{fig:kae_reco} show the reconstructions obtained after fitting respectively a 2-1-2 standard and kernel AE.
Since the latent space is of dimension 1, the 2D reconstructions are manifolds of the same dimension, hence the curve aspect.
What is interesting though is that the KAE learns a much more complex manifold than the standard AE.
Due to its linear limitations (the nonlinear activation functions did not help much in this case), the standard AE returns a line, far from the original data, while the KAE outputs a more complex manifold, much closer to the initial data.

Apart from a good reconstruction, we are interested in finding representations with attractive properties.
The 1D feature found by the previous KAE is interesting, as it is a discriminative one with respect to the original clusters: points from different circles are mapped around different values (\Cref{fig:cc_1D_enc0}).
Interestingly, after a few iterations, some variability is introduced around these \textit{cluster values}, so that all codes shall not be mapped back to the same point (\Cref{fig:cc_1D_enc1}).

Finally, a KAE with 1 hidden layer of size 2 gives the internal representation shown in \Cref{fig:cc_2D_out}.
This new 2D representation has a disentangling effect: the circle structure is kept so as to preserve the intra-cluster specificity, while the inter-cluster differentiation is ensured by the circles' dissociation.
These visual 2D examples give interesting insights on the good properties of the KAE representations: discrimination, disentanglement (see further experiments in \Cref{sec:add_expe}).

\subsection{Representation Learning on Molecules}
\begin{table*}[t]
\vspace{-0.2cm}
\caption{NMSEs on Molecular Activity for Different Types of Cancer}
\label{tab:NCI}
\vskip 0.15in
\begin{center}
\begin{small}
\begin{sc}
\begin{tabular}{cccccc}
\toprule
          & KRR      & KPCA 10 + RF & KPCA 50 + RF & $\text{K}^2$AE 10 + RF & $\text{K}^2$AE 50 + RF\\
\midrule
Cancer 01 & 0.02978  & 0.03279  & 0.03035  & 0.03097  & \textbf{0.02808}\\
Cancer 02 & 0.03004  & 0.03194  & 0.02978  & 0.03099  & \textbf{0.02775}\\
Cancer 03 & 0.02878  & 0.03155  & 0.02914  & 0.02989  & \textbf{0.02709}\\
Cancer 04 & 0.03003  & 0.03274  & 0.03074  & 0.03218  & \textbf{0.02924}\\
Cancer 05 & 0.02954  & 0.03185  & 0.02903  & 0.03065  & \textbf{0.02754}\\
Cancer 06 & 0.02914  & 0.03258  & 0.03083  & 0.03134  & \textbf{0.02838}\\
Cancer 07 & 0.03113  & 0.03468  & 0.03207  & 0.03257  & \textbf{0.03018}\\
Cancer 08 & 0.02899  & 0.03162  & 0.02898  & 0.03065  & \textbf{0.02770}\\
Cancer 09 & 0.02860  & 0.02992  & 0.02804  & 0.02872  & \textbf{0.02627}\\
Cancer 10 & 0.02987  & 0.03291  & 0.03111  & 0.03170  & \textbf{0.02910}\\ \bottomrule
\end{tabular}
\end{sc}
\end{small}
\end{center}
\vskip -0.15in
\end{table*}

We now present an application of KAEs in the context of chemoinformatics.
The motivation is triple.
First, such complex data cannot be handled by standard AEs.
Second, kernel methods being prominent in the field, data are often stored as Gram matrices, suiting perfectly our framework.
Third, finding a compressed representation of a molecule is a problem of highest interest in Drug Discovery.
We considered two different problems, one supervised, one unsupervised.

As for the supervised one, we exploited the dataset of \cite{Su2010} from the NCI-Cancer database: it consists in a Gram matrix comparing 2303 molecules by the mean of a Tanimoto kernel (a linear path kernel built using the presence or absence of sequences of atoms in the molecule), as well as the molecules activities in the presence of 59 types of cancer.
The dataset containing no vectorial representations of the molecules (but only Gram matrices), only kernel methods were possible to benchmark.
As a good representation is supposed to facilitate ulterior learning tasks, we assess the goodness of the representations through the regression scores obtained by Random Forests (RFs) from scikit-learn \citep{scikit-learn} fed with it.

2-layer $\text{K}^{2}$AEs with respectively 5, 10, 25, 50 and 100 internal dimension were run, as well as Kernel Principal Component Analyses (KPCAs) with the same number of components.
Finally, these representations were given as inputs to RFs.
KRR was also added to the comparison.
The Normalized Mean Squared Errors (NMSEs), averaged on 10 runs, for 5 strategies and on the first 10 cancers are stored in \Cref{tab:NCI} (the complete results are available in \Cref{sec:add_nci}).
A visualization with all strategies is also proposed in \Cref{fig:bar_plot}.
Clearly, methods combining a data representation step followed by a prediction one performs better.
But the good performance of our approach should not be attributed to the use of RFs only, since the same strategy run with KPCA leads to worse results.
Indeed, the $\text{K}^2$AE 50 + RF strategy outperforms all other procedures on all problems, managing to extract compact and useful feature vectors from the molecules.

\begin{table}[!ht]
\caption{MSREs on Test Metabolites}
\label{tab:reco}
\vskip 0.15in
\begin{center}
\begin{small}
\begin{sc}
\begin{tabular}{cccc}
\toprule
Dimension & AE (sigmoid) & AE (relu) & KAE \\
\midrule
5   & 99.81 & 96.62 & \textbf{76.38} \\
10  & 87.36 & 84.02 & \textbf{65.76} \\
25  & 72.31 & 68.77 & \textbf{51.63} \\
50  & 63.00 & 58.29 & \textbf{40.72} \\
100 & 55.43 & 48.63 & \textbf{36.27} \\ \bottomrule
\end{tabular}
\end{sc}
\end{small}
\end{center}
\vskip -0.15in
\end{table}

The data for the unsupervised problem is taken from \citet{Brouard-ismb2016}. It is composed of two sets (a train set of size 5579, and a test set of size 1395), each one containing metabolites under the form of 4136-long binary vectors (called fingerprints), as well as a Gram matrix comparing them.
2-layer standard AEs from Keras \citep{chollet2015keras} with sigmoid and relu activation functions, and 2-layer KAEs with internal layer of size 5, 10, 25, 50 and 100, were trained.
In absence of a supervised task, we measured the Mean Squared Reconstruction Errors (MSREs) induced on the test set, and stored them in \Cref{tab:reco}.
Again, the KAE approach shows a systematic improvement.

%% file: 6-Conclusion.tex
\section{CONCLUSION}\label{sec:conclusion}
We introduce a new framework for AEs, based on vv-RKHSs and OVKs.
The use of RKHS functions enables KAEs to handle data from possibly infinite dimensional Hilbert spaces, and then to extend the autoencoding scheme to any kind of data.
A generalization bound and a strong connection to KPCA are established, while
the underlying optimization problem is tackled by a Representer Theorem and the kernel trick.
Beyond a detailed description, the behavior of the algorithm is carefully studied on simulated data, and yields relevant performances on graph data, that standard AEs are typically unable to handle.
Further research may consider a semi-supervised approach, that would ideally tailor the representation according to the future targeted.

%% file: 7-Appendix.tex
\onecolumn
\section{TECHNICAL PROOFS}

%%%%%%%%%%%%%%%%%%%%%%%%%%%
%  EQUIVALENCE THM PROOF  %
%%%%%%%%%%%%%%%%%%%%%%%%%%%

\subsection{Proof of \Cref{thm:reformulation}}\label{sec:Eq_pbm_proof}
Let $(\lambda_1,\; \ldots,\; \lambda_L) \in \mathbb{R}_+^L$ and $(f_1^*,\; \ldots,\; f_L^*)$ a solution to problem \eqref{eq:penalized_pb}.
Let $s_l = \|f_l^*\|_{\mathcal{H}_l}^2~\forall~ l \in \llbracket L \rrbracket$.
We shall prove that $(f_1^*,\; \ldots,\; f_L^*)$ is also a solution to problem \eqref{eq:constrained_pb} for this choice of $(s_1, \ldots, s_L)$.
Consider $(f_1,\; \ldots,\; f_L)$ satisfying problem \eqref{eq:constrained_pb}'s constraints.
$\forall~l \in \llbracket L \rrbracket,~\|f_l\|_{\mathcal{H}_l}^2 \leq s_l = \|f_l^*\|_{\mathcal{H}_l}^2$. Hence, we have $\sum_{l=1}^L \lambda_l \|f_l\|_{\mathcal{H}_l}^2 \leq \sum_{l=1}^L \lambda_l \|f_l^*\|_{\mathcal{H}_l}^2$.
On the other hand, by definition of the $f_l^*$'s, it holds :
\begin{equation*}
V(f_1, \ldots, f_L) + \sum_{l=1}^L \lambda_l \|f_l\|_{\mathcal{H}_l}^2~\geq~V(f_1^*, \ldots, f_L^*) + \sum_{l=1}^L \lambda_l \|f_l^*\|_{\mathcal{H}_l}^2.
\end{equation*}
Thus, we necessarily have: $V(f_1, \ldots, f_d) ~\geq~V(f_1^*, \ldots, f_d^*)$.

A similar argument can be used for local solutions, details are left to the reader.
\qed

Although this result may appear rather simple, we thought it was worth mentioning as our setting is particularly unfriendly: the objective function $V$ is not assumed to be convex, and the variables $f_l$ are infinite dimensional.
As a consequence, in absence of additional assumptions the converse statement (that solutions to problem \eqref{eq:constrained_pb} are also solutions to problem \eqref{eq:penalized_pb} for a suitable choice of $\lambda_l$'s) is not guaranteed.
The proof indeed rely on the existence of Lagrangian multipliers, which has been shown when the variables are finite dimensional (KKT conditions), or when the objective function is assumed to be convex \citep{bauschke2011convex}, but is not ensured in our case.

%%%%%%%%%%%%%%%%%%%%%%%%%%%%%%%
%  PROOF OF RADEMACHER BOUND  %
%%%%%%%%%%%%%%%%%%%%%%%%%%%%%%%

\subsection{Proof of \Cref{thm:gen_bound}}\label{sec:lemma_proof}

The technical proof is structured as follows.

\input{8-KAE_Rademacher}

%%%%%%%%%%%%%%%%%%%%%%%%%%%
%  REPRESENTER THM PROOF  %
%%%%%%%%%%%%%%%%%%%%%%%%%%%

\subsection{Proof of \Cref{thm:RT}}\label{sec:RT_proof}

\begin{lemma}\label{lem:1}
See Theorem 3.1 in \citet{micchelli2005learning}.
Let $\mathcal{X}$ be a measurable space, $\mathcal{Y}$ a real Hilbert space with inner product $\left\langle \cdot,\cdot\right\rangle_\mathcal{Y}$, $\mathcal{K} : \mathcal{X} \times \mathcal{X} \rightarrow \mathcal{L}(\mathcal{Y})$ an operator-valued kernel, $\mathcal{H} \subset \mathcal{F}(\mathcal{X}, \mathcal{Y})$ the corresponding vv-RKHS, with inner product $\langle \cdot,\cdot\rangle_\mathcal{H}$.
We have the reproducing property : $\langle y, f(x)\rangle_\mathcal{Y} = \langle \mathcal{K}_xy,f\rangle_\mathcal{H}$, with the notation $\mathcal{K}_xy = \mathcal{K}(\cdot,x)y : \mathcal{X} \rightarrow \mathcal{Y}$.
Suppose also that the linear functionals $L_{x_i}f = f(x_i), f \in \mathcal{H}, i \in \llbracket n \rrbracket$ are linearly independent.
Then the unique solution to the variational problem:
\begin{equation*}
\min_{f \in \mathcal{H}}\Big\{\|f\|_\mathcal{H}^2 : f(x_i) = y_i,~i \in \llbracket n \rrbracket\Big\},
\end{equation*}
is given by :
\begin{equation*}
\hat{f} = \sum_{i=1}^n \mathcal{K}_{x_i}c_i,
\end{equation*}
where $\{c_i,~i \in \llbracket n \rrbracket\} \subset \mathcal{Y}^n$ is the unique solution of the linear system of equations :
\begin{equation*}
\sum_{i=1}^n \mathcal{K}(x_k, x_i)c_i = y_k, \qquad k \in \llbracket n \rrbracket.
\end{equation*}
\end{lemma}
\begin{proof}
Let $f \in \mathcal{H}$ such that $f(x_i) = y_i~~\forall~i \in \llbracket n \rrbracket$, and set $g = f - \hat{f}$.
We have :
\begin{equation*}
\|f\|_\mathcal{H}^2 = \|\hat{f}\|_\mathcal{H}^2 + \|g\|_\mathcal{H}^2 + 2\langle \hat{f},g\rangle_\mathcal{H}.
\end{equation*}
Observe also that :
\begin{equation*}
\langle \hat{f},g\rangle_\mathcal{H} = \left\langle \sum_{i=1}^n \mathcal{K}_{x_i}c_i,g\right\rangle_\mathcal{H} = \sum_{i=1}^n \langle \mathcal{K}_{x_i}c_i , g\rangle_\mathcal{H} = \sum_{i=1}^n \langle c_i , g(x_i)\rangle_\mathcal{Y} = 0.
\end{equation*}
Finally, we have :
\begin{equation*}
\|f\|_\mathcal{H}^2 = \|\hat{f}\|_\mathcal{H}^2 + \|g\|_\mathcal{H}^2 \geq \|\hat{f}\|_\mathcal{H}^2.
\end{equation*}
\end{proof}

\begin{proof_of}{\Cref{thm:RT}}
We shall use the following shortcut notation:
\begin{equation*}
\xi(f_1^*, \ldots, f_{L_0}^*, \mathcal{S}) \coloneqq V\Big((f_{L_0}\circ\ldots\circ f_1)(x_1),\ldots,(f_{L_0}\circ\ldots\circ f_1)(x_n), \left\|f_1\right\|_{\mathcal{H}_1}, \ldots, \left\|f_{L_0}\right\|_{\mathcal{H}_{L_0}}\Big).
\end{equation*}
Let $l_0 \in \llbracket L_0 \rrbracket$. Let $g_{l_0} \in \mathcal{H}_{l_0}$ such that :
\begin{equation*}
g_{l_0}\left({x_i^*}^{(l_0-1)}\right) = f_{l_0}^*\left({x_i^*}^{(l_0-1)}\right), \qquad \forall~i \in \llbracket n \rrbracket.
\end{equation*}
By definition, we have :
\begin{equation*}
\xi(f_1^*, \ldots,f_{l_0}^*, \ldots, f_{L_0}^*, \mathcal{S}) \leq \xi(f_1^*, \ldots,g_{l_0}, \ldots, f_{L_0}^*, \mathcal{S}),
\end{equation*}
thus we necessarily have :
\begin{equation*}
\|f_{l_0}^*\|_{\mathcal{H}_{l_0}}^2 \leq \|g_{l_0}\|_{\mathcal{H}_{l_0}}^2.
\end{equation*}
Therefore $f_{l_0}^*$ is a solution to the problem :
\begin{equation*}
\min_{f \in \mathcal{H}_{l_0}} \left\{~\|f\|_{\mathcal{H}_{l_0}}^2 : f\left({x_i^*}^{(l_0-1)}\right) =f_{l_0}^*\left({x_i^*}^{(l_0-1)}\right),~i \in \llbracket n \rrbracket~\right\}.
\end{equation*}
From \Cref{lem:1}, there exists $\left(\varphi_{l_0, 1}^*, \ldots, \varphi_{l_0, n}^*\right) \in \mathcal{X}_{l_0}^n$, such that :
\begin{equation*}
f_{l_0}^*(\cdot) = \sum_{i=1}^n \mathcal{K}_{l_0}\left(~\cdot~, {x_i^*}^{(l_0-1)}\right)\varphi_{l_0, i}^*.
\end{equation*}
\end{proof_of}

%%%%%%%%%%%%%%%%%%%%%%%%%
%  NON-CONVEXITY PROOF  %
%%%%%%%%%%%%%%%%%%%%%%%%%

\subsection{Non-convexity of the Problem}\label{sec:non_cvx}
\subsubsection{Functional Setting}
We prove that problem \eqref{eq:emp_pb} is not convex by showing that the objective function $(f, g) \mapsto \hat{\epsilon}_n(g \circ f) + \Omega(f, g)$ is not.
We denote this application by $\mathcal{O}$ and suppose it is.
If it were convex, one would have :
\begin{equation}\label{eq:cvx}
\mathcal{O}\left(\kappa(f, g) + (1 - \kappa)(f', g')\right) ~\leq~\kappa \mathcal{O}(f, g)
+ (1 - \kappa)\mathcal{O}(f', g'),
\end{equation}
for any $\kappa \in [0, 1]$ and any functions $f, f', g, g' \in \mathcal{H}_1^2 \times \mathcal{H}_2^2$.
Now, consider the particular case where we want to encode a single point ($n=1$) from $\mathcal{X}_0 = \mathbb{R}$ to $\mathcal{X}_1 = \mathbb{R}$, using one single hidden layer ($L=2$).
Let $x_1 = 1$, and assume that both kernels are linear : $\mathcal{K}_1(x, x') = xx'$, $\mathcal{K}_2(y, y') = yy'$.
$f : x \mapsto \mathcal{K}_1(x, x_1)\varphi = \varphi x$ and $f' : x \mapsto \mathcal{K}_1(x, x_1)\varphi' = \varphi'x$ are elements of $\mathcal{H}_1$ for any coefficients $\varphi, \varphi'$.
In the same way, $g : y \mapsto \mathcal{K}_2(y, f(x_1))\psi = \psi f(1) y$ and $g' : y \mapsto \mathcal{K}_2(y, f'(x_1))\psi' = \psi'f'(1) y$ are elements of $\mathcal{H}_2$ for any $\psi, \psi' \in \mathbb{R}^2$.

Therefore, $\mathcal{O}(f,g)$ depends only on $\varphi$ and $\psi$.
Let $\mathcal{P}$ denote the application from $\mathbb{R}^2$ to $\mathbb{R}$ such that $\mathcal{O}(f, g) = \mathcal{P}(\varphi, \psi)$.
Then, one has also $\mathcal{O}(f', g') = \mathcal{P}(\varphi', \psi')$.
And finally, it holds :
\begin{align*}
\mathcal{O}\left(\kappa(f, g) + (1 - \kappa)(f', g')\right) &= \mathcal{O}\left(\kappa f + (1 - \kappa)f', \kappa g + (1 - \kappa)g')\right),\\
&= \mathcal{P}\left(\kappa \varphi + (1 - \kappa)\varphi', \kappa \psi + (1 - \kappa)\psi')\right),\\
\mathcal{O}\left(\kappa(f, g) + (1 - \kappa)(f', g')\right) &= \mathcal{P}\left(\kappa(\varphi, \psi) + (1 - \kappa)(\varphi', \psi')\right).
\end{align*}

So if \eqref{eq:cvx} were true, in particular it would be true for the specific $f, f', g, g'$ functions we just defined.
Hence, the following would hold for any $\varphi, \varphi', \psi, \psi' \in \mathbb{R}^4$ :
\begin{equation*}
\mathcal{P}\left(\kappa(\varphi, \psi) + (1 - \kappa)(\varphi', \psi')\right) ~\leq~ \kappa \mathcal{P}(\varphi, \psi) + (1 - \kappa)\mathcal{P}(\varphi', \psi').
\end{equation*}
This is exactly the convexity of $\mathcal{P}$ in $(\varphi, \psi)$.
So the convexity of the objective function in the functional setting (problem \eqref{eq:emp_pb}) implies the convexity of the objective function in the parametric setting (obtained after application of \Cref{thm:RT}).
In the following section we show that the latest does not even hold, which allows to conclude that neither problem is convex.

\subsubsection{Parametric Setting}
As a reminder, we have :
\begin{align*}
&f(x) = \mathcal{K}_1(x, x_1)\varphi = \varphi x, \qquad &&f(1) = \varphi,\\
&g(y) = \mathcal{K}_2(y, f(x_1))\psi = \varphi\psi y,  &&g(f(1)) = \varphi^2\psi.
\end{align*}

Our problem reads :
\begin{equation*}
\min_{\varphi \in \mathbb{R},~\psi \in \mathbb{R}}~~\mathcal{P}(\varphi, \psi) \overset{def}{=} \left(1 - \varphi^2\psi\right)^2 + \lambda\varphi^2 + \mu \psi^2,
\end{equation*}
or equivalently :
\begin{equation*}
\min_{\varphi \in \mathbb{R},~\psi \in \mathbb{R}}~~1 + \lambda\varphi^2 + \mu\psi^2 - 2\varphi^2\psi + \varphi^4\psi^2.
\end{equation*}

Let us find the critical points and analyze them. We have :
\begin{eqnarray*}
\frac{\partial \mathcal{P}}{\partial \varphi}(\varphi, \psi) &=& 2\lambda\varphi - 4\varphi\psi + 4\varphi^3\psi^2,\\
\frac{\partial \mathcal{P}}{\partial^2 \varphi}(\varphi, \psi) &=&  2\lambda - 4\psi + 12\varphi^2\psi^2,\\
\frac{\partial \mathcal{P}}{\partial \psi}(\varphi, \psi) &=& 2\mu\psi -2 \varphi^2 + 2\varphi^4\psi,\\
\frac{\partial \mathcal{P}}{\partial^2 \psi}(\varphi, \psi) &=& 2\mu + 2\varphi^4,\\
\frac{\partial \mathcal{P}}{\partial \varphi \partial \psi}(\varphi, \psi) &=& -4\varphi + 8\varphi^3\psi.
\end{eqnarray*}

The two following equivalence relationships hold true:
\begin{align*}
\frac{\partial \mathcal{P}}{\partial \varphi}(\varphi^*, \psi^*) &= \left(2\lambda - 4\psi^* + 4{\varphi^*}^2{\psi^*}^2\right)\varphi^* = 0 &&\Leftrightarrow \qquad \varphi^* = 0 \text{~~or~~} {\varphi^*}^2 = \frac{2\psi^* - \lambda}{2 {\psi^*}^2}, \qquad\qquad&\\[.3cm]
\frac{\partial \mathcal{P}}{\partial \psi}(\varphi^*, \psi^*) &= 2\mu\psi^* - 2{\varphi^*}^2 + 2 {\varphi^*}^4\psi^* = 0 &&\Leftrightarrow \qquad \psi^* = \frac{{\varphi^*}^2}{{\varphi^*}^4 + \mu}.
\end{align*}
Obviously, the point $(\varphi^*, \psi^*) = (0, 0)$ is always critical.
Notice that :
\begin{equation*}
\textbf{Hess}_{(0, 0)}\mathcal{P} = \left(\begin{matrix}2\lambda & 0\\0&2\mu\end{matrix}\right) \succ 0.
\end{equation*}

\begin{figure*}[th]
\begin{center}
\includegraphics[width=\textwidth]{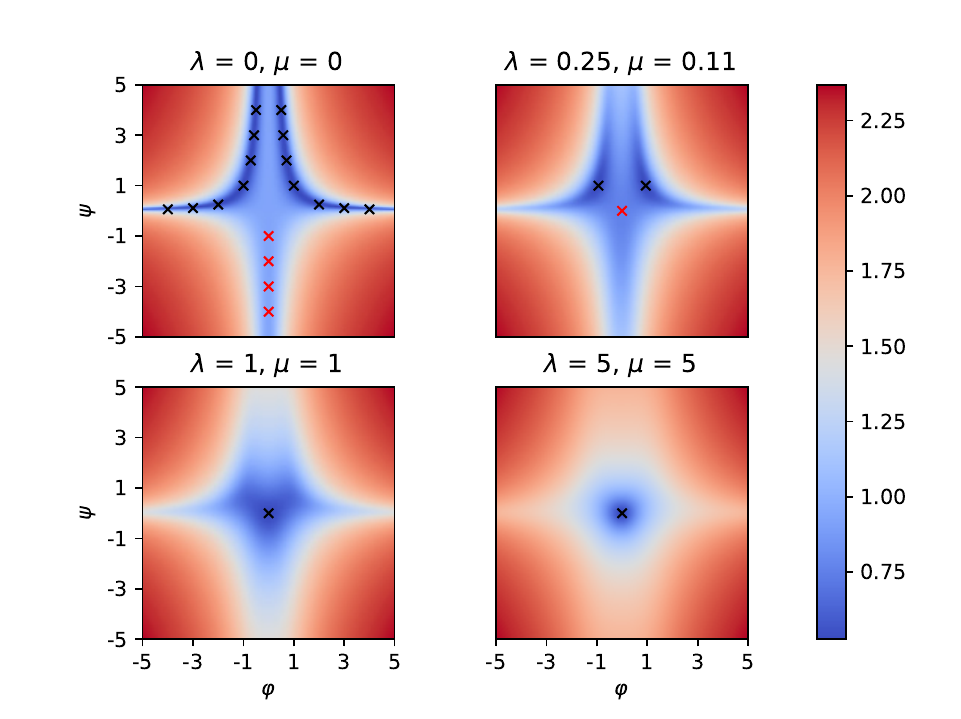}
\caption{Heatmaps of $\mathcal{P}$ for different values of $\lambda$ and $\mu$}
\label{fig:heatmaps}
\end{center}
\end{figure*}

Thus $(0,0)$ is a local minimum and $\mathcal{P}(0, 0) = 1$.
To prove that it is not a global minimizer, it is enough to find a couple $(\varphi, \psi)$ such that $\mathcal{P}(\varphi, \psi) < 1$.
For example $\mathcal{P}(1, 1) = \lambda + \mu$.
As soon as $\lambda + \mu < 1$, the objective $\mathcal{P}$ is not invex, and a fortiori non-convex.

\Cref{fig:heatmaps} shows the heatmaps of $\mathcal{P}$ with respect to $\varphi$ and $\psi$ for different regularization settings.
Note that in the non-regularized setting ($\lambda = \mu =0$), every point $(0, \psi)$ with $\psi < 0$ is a local minimizer but not a global one.
They are represented by red crosses.
On the other hand, we have also an infinite number of global minima, namely every couple satisfying $\varphi^2\psi = 1$.
See the black crosses on the top left figure.
When the regularization parameters remain small enough, $(0,0)$ is a local minimizer but not a global one (top right figure).
Finally, the higher the regularization, the smoother the objective, even if convexity can never be verified (bottom figures).

% OLD FIGURE
%%%%%%%%%%%%
%\begin{figure}
%\centering
%\includegraphics[width=\textwidth]{fig/non_cvx.png}
%\caption{Heatmaps of $\mathcal{P}$ for different values of $\lambda$ and $\mu$}
%\label{fig:heatmaps}
%\end{figure}

% OLD FIGURE 2
%%%%%%%%%%%%%%
% \begin{figure}[!ht]
% \begin{center}
% \subfloat[$\lambda=0,~\mu=0$\label{fig:00}]{\includegraphics[width=2.5in]{fig/non_cvx_1.png}}\\
% \subfloat[$\lambda=0.25,~\mu=0.11$\label{fig:smallreg}]{\includegraphics[width=2.5in]{fig/non_cvx_2.png}}\\
% \subfloat[$\lambda=1,~\mu=1$\label{fig:strongreg1}]{\includegraphics[width=2.5in]{fig/non_cvx_3.png}}\\
% \subfloat[$\lambda=5,~\mu=5$\label{fig:strongreg2}]{\includegraphics[width=2.5in]{fig/non_cvx_4.png}}
% \caption{Heatmaps of $\mathcal{P}$ for different values of $\lambda$ and $\mu$}
% \label{fig:heatmaps}
% \end{center}
% \end{figure}

% OLD FIGURE 3
%%%%%%%%%%%%%%
% \begin{figure*}[ht]
% \vskip 0.2in
% \begin{center}
% \subfigure[$\lambda=0,~\mu=0$\label{fig:00}]{\includegraphics[width=\columnwidth]{fig/non_cvx_1.png}}
% \subfigure[$\lambda=0.25,~\mu=0.11$\label{fig:smallreg}]{\includegraphics[width=\columnwidth]{fig/non_cvx_2.png}}\\
% \subfigure[$\lambda=1,~\mu=1$\label{fig:strongreg1}]{\includegraphics[width=\columnwidth]{fig/non_cvx_3.png}}
% \subfigure[$\lambda=5,~\mu=5$\label{fig:strongreg2}]{\includegraphics[width=\columnwidth]{fig/non_cvx_4.png}}
% \caption{Heatmaps of $\mathcal{P}$ for different values of $\lambda$ and $\mu$}
% \label{fig:heatmaps}
% \end{center}
% \vskip -0.2in
% \end{figure*}

%%%%%%%%%%%%%%%%%%%%%%%%%%%%%%%%%
%  GRADIENT DERIVATION DETAILS  %
%%%%%%%%%%%%%%%%%%%%%%%%%%%%%%%%%

\clearpage
\section{Gradient Derivation Details}\label{sec:grad_detail}
\subsection{Detail of Equation \eqref{eq:norms_detail}}
\begin{align*}
\|f_l\|_{\mathcal{H}_l}^2 &= \left\langle f_l, f_l\right\rangle_{\mathcal{H}_l},\\
&= \left\langle \sum_{i=1}^n  \mathcal{K}_l\left(~.~, {x_i}^{(l-1)}\right)\varphi_{l, i}~,~\sum_{i'=1}^n \mathcal{K}_l\left(~.~, {x_{i'}}^{(l-1)}\right)\varphi_{l, i'}\right\rangle_{\mathcal{H}_l},\\
&= \sum_{i, i' = 1}^n \left\langle \mathcal{K}_l\left(~.~, {x_i}^{(l-1)}\right)\varphi_{l, i}~,~\mathcal{K}_l\left(~.~, {x_{i'}}^{(l-1)}\right)\varphi_{l, i'}\right\rangle_{\mathcal{H}_l},\\
&= \sum_{i, i' = 1}^n \left\langle \varphi_{l, i}~,~\mathcal{K}_l\left({x_i}^{(l-1)}, {x_{i'}}^{(l-1)}\right)\varphi_{l, i'}\right\rangle_{\mathcal{X}_l},\\
\|f_l\|_{\mathcal{H}_l}^2 &= \sum_{i, i' = 1}^n k_l\left({x_i}^{(l-1)}, {x_{i'}}^{(l-1)}\right)\left\langle \varphi_{l, i}~,~A_l~\varphi_{l, i'}\right\rangle_{\mathcal{X}_l}.
\end{align*}
\qed

\subsection{Detail of Equation \eqref{eq:cross_norm}}\label{sec:proof_cross_norm}
\begin{align*}
\Big(\nabla_{\varphi_{l_0, i_0}}~\|f_l\|_{\mathcal{H}_l}^2\Big)^\top &= \sum_{i,i'=1}^n [N_l]_{i, i'} \left(\nabla_{\varphi_{l_0, i_0}}~k_l\left({x_i}^{(l-1)}, {x_{i'}}^{(l-1)}\right)\right)^\top,\\
&= \sum_{i,i'=1}^n [N_l]_{i, i'}~\bigg[~\left(\nabla^{(1)} k_l\left({x_i}^{(l-1)}, {x_{i'}}^{(l-1)}\right)\right)^\top\textbf{Jac}_{{x_i}^{(l-1)}}(\varphi_{l_0, i_0})\\
&\phantom{= \sum_{i,i'=1}^n [N_l]_{i, i'}~\bigg[~}+\left(\nabla^{(2)} k_l\left({x_i}^{(l-1)}, {x_{i'}}^{(l-1)}\right)\right)^\top\textbf{Jac}_{{x_{i'}}^{(l-1)}}(\varphi_{l_0, i_0})~\bigg],\\
&=\sum_{i,i'=1}^n [N_l]_{i, i'}\left(\nabla^{(1)} k_l\left({x_i}^{(l-1)}, {x_{i'}}^{(l-1)}\right)\right)^\top\textbf{Jac}_{{x_i}^{(l-1)}}(\varphi_{l_0, i_0})\\
&\phantom{=}+\sum_{i',i=1}^n [N_l]_{i', i}\left(\nabla^{(1)} k_l\left({x_{i'}}^{(l-1)}, {x_i}^{(l-1)}\right)\right)^\top\textbf{Jac}_{{x_{i'}}^{(l-1)}}(\varphi_{l_0, i_0}),\\
\Big(\nabla_{\varphi_{l_0, i_0}}~\|f_l\|_{\mathcal{H}_l}^2\Big)^\top &= 2\sum_{i,i'=1}^n [N_l]_{i, i'}\left(\nabla^{(1)} k_l\left({x_i}^{(l-1)}, {x_{i'}}^{(l-1)}\right)\right)^\top\textbf{Jac}_{{x_i}^{(l-1)}}(\varphi_{l_0, i_0}),
\end{align*}
where $\nabla^{(1)} k_l\left(x, x'\right)$ (respectively $\nabla^{(2)} k_l\left(x, x'\right)$) denotes the gradient of $k_l(\cdot,\cdot)$ with respect to the $1^{st}$ (respectively $2^{nd}$) coordinate evaluated in $\left(x, x'\right)$.
\qed

\subsection{Detail of Jacobians Computation}\label{sec:jacs}
All previously written gradients involve Jacobian matrices.
Their computation is to be detailed in this subsection.
First note that $\textbf{Jac}_{{x_i}^{(l)}}(\varphi_{l_0, i_0})$ only makes sense if $l_0 \leq l$.
Indeed, ${x_i}^{(l)}$ is completely independent from $\varphi_{l_0, i_0}$ otherwise.
Let us first detail ${x_i}^{(l)}$ and use the linearity of the Jacobian operator :
\begin{equation*}
\textbf{Jac}_{{x_i}^{(l)}}(\varphi_{l_0, i_0})
= \sum_{i'=1}^n \textbf{Jac}_{k_l\left({x_i}^{(l-1)}, {x_{i'}}^{(l-1)}\right)A_l~\varphi_{l, i'}}(\varphi_{l_0, i_0}).
\end{equation*}
Just as in the norm gradient case (see \Cref{sec:finite_dim}), there are two different outputs depending on whether $l = l_0$ (this gives an initialization), or $l > l_0$ (this leads to a recurrence formula).

% 3.3.1. Own Jacobian

\underline{Own Jacobian $(l = l_0)$ :}
\begin{align*}
\textbf{Jac}_{{x_i}^{(l)}}(\varphi_{l, i_0}) &= \sum_{i'=1}^n \textbf{Jac}_{k_l\left({x_i}^{(l-1)}, {x_{i'}}^{(l-1)}\right)A_l~\varphi_{l, i'}}(\varphi_{l, i_0}),\\
&= \sum_{i'=1}^n k_l\left({x_i}^{(l-1)}, {x_{i'}}^{(l-1)}\right) \textbf{Jac}_{A_l~\varphi_{l, i'}}(\varphi_{l, i_0}),\\
\textbf{Jac}_{{x_i}^{(l)}}(\varphi_{l, i_0})&= [K_l]_{i, i_0}~A_l.
\end{align*}

% 3.3.2. Higher Jacobian

\underline{Higher Jacobian $(l > l_0)$ :}
\begin{align*}
\textbf{Jac}_{{x_i}^{(l)}}(\varphi_{l_0, i_0}) &= \sum_{i'=1}^n \textbf{Jac}_{k_l\left({x_i}^{(l-1)}, {x_{i'}}^{(l-1)}\right)A_l~\varphi_{l, i'}}(\varphi_{l_0, i_0}),\\
&= \sum_{i'=1}^n A_l~\varphi_{l, i'} \left(\nabla_{\varphi_{l_0, i_0}}~k_l\left({x_i}^{(l-1)}, {x_{i'}}^{(l-1)}\right)\right)^\top,\\
&= A_l \sum_{i'=1}^n \varphi_{l, i'} \bigg[~\left(\nabla^{(1)} k_l\left({x_i}^{(l-1)}, {x_{i'}}^{(l-1)}\right)\right)^\top\textbf{Jac}_{{x_i}^{(l-1)}}(\varphi_{l_0, i_0})\\
&\phantom{= A_l \sum_{i'=1}^n \varphi_{l, i'} \bigg[~}+ \left(\nabla^{(1)} k_l\left({x_{i'}}^{(l-1)}, {x_i}^{(l-1)}\right)\right)^\top\textbf{Jac}_{{x_{i'}}^{(l-1)}}(\varphi_{l_0, i_0})~\bigg],\\
&=~~A_l \left[\sum_{i'=1}^n \varphi_{l, i'} \left(\nabla^{(1)} k_l\left({x_i}^{(l-1)}, {x_{i'}}^{(l-1)}\right)\right)^\top\right]\textbf{Jac}_{{x_i}^{(l-1)}}(\varphi_{l_0, i_0})\\
&~~+A_l\left[\sum_{i'=1}^n \varphi_{l, i'} \left(\nabla^{(1)} k_l\left({x_{i'}}^{(l-1)}, {x_i}^{(l-1)}\right)\right)^\top\textbf{Jac}_{{x_{i'}}^{(l-1)}}(\varphi_{l_0, i_0})\right],\\
\textbf{Jac}_{{x_i}^{(l)}}(\varphi_{l_0, i_0})&=A_l\Bigg[\Phi_l^\top \Delta_l\left({x_i}^{(l-1)}\right)\textbf{Jac}_{{x_i}^{(l-1)}}(\varphi_{l_0, i_0})\\
&~~~~~~~~~~~+\sum_{i'=1}^n \varphi_{l, i'} \left(\nabla^{(1)} k_l\left({x_{i'}}^{(l-1)}, {x_i}^{(l-1)}\right)\right)^\top\textbf{Jac}_{{x_{i'}}^{(l-1)}}(\varphi_{l_0, i_0})\Bigg],
\end{align*}
with $\Delta_l(x) \coloneqq \left(\left(\nabla^{(1)} k_l\left(x, {x_1}^{(l-1)}\right)\right)^\top,\ldots, \left(\nabla^{(1)} k_l\left(x, {x_n}^{(l-1)}\right)\right)^\top\right)^\top$ the $n \times d_{l-1}$ matrix storing the $\nabla^{(1)} k_l\left(x, {x_i}^{(l-1)}\right)$ in rows.
These matrices are computed on \Cref{sec:delta_detail} (especially for $x={x_{i}}^{(l-1)})$.
Assuming these quantities are known, we have an expression of $\textbf{Jac}_{{x_i}^{(l)}}(\varphi_{l_0, i_0})$ that only depends on the $\textbf{Jac}_{{x_{i'}}^{(l-1)}}(\varphi_{l_0, i_0})$.
Thus we can unroll the recurrence until $l=l_0$ and, using the previous subsection, compute $\textbf{Jac}_{{x_i}^{(l)}}(\varphi_{l_0, i_0})$ for every couple $(l, l_0)$ such that $l > l_0$.\\

An interesting remark can be made on the two-terms structure of the Jacobians.
Indeed, the first term corresponds to the chain rule on ${x_i}^{(l)} = f_l\left({x_i}^{(l-1)}\right)$ assuming that $f_l$ is constant : $\frac{\partial f_l\left({x_i}^{(l-1)}\right)}{\partial \varphi_{l_0, i_0}} = \frac{\partial f_l\left({x_i}^{(l-1)}\right)}{\partial{x_i}^{(l-1)}}\cdot\frac{\partial {x_i}^{(l-1)}}{\partial \varphi_{l_0, i_0}}$ (notation abuse on $\partial$ in order to preserve understandability).
On the contrary, the second term corresponds to a chain rule assuming that ${x_i}^{(l-1)}$ does not vary with $\varphi_{l_0, i_0}$, but that $f_l$ does, through the influence of $\varphi_{l_0, i_0}$ on the supports of $f_l$, namely the ${x_{i'}}^{(l-1)}$.

\subsection{Detail of the $\Delta_l$ Matrices Computation}\label{sec:delta_detail}
In this section we derive the quantities $\nabla^{(1)} k_l\left({x_i}^{(l-1)}, {x_{i'}}^{(l-1)}\right)$ and more specifically the matrices $\Delta_l\left({x_i}^{(l-1)}\right)$ for $l \in \llbracket L \rrbracket$ and $i \in \llbracket n \rrbracket$.
Note that all previously computed quantities are independent from the kernel chosen.
Actually, the $\Delta_l\left({x_i}^{(l-1)}\right)$ matrices encapsulate all the kernel specificity of the algorithm.
Thus, tailoring a new algorithm by changing the kernels only requires computing the new $\Delta_l$ matrices.
This flexibility is a key asset of our approach, and more generally a crucial characteristic of kernel methods.
In the following, we describe the $\Delta_l$ derivation for two popular kernels : the Gaussian and the polynomial ones.

% 3.4.1. Gaussian kernel

\underline{Gaussian kernel :}
\begin{equation*}
\nabla^{(1)} k_l(x,x') = \nabla_x \Big(\exp\left(-\gamma_l \|x - x'\|_{\mathcal{X}_{l-1}}^2\right)\Big) = -2 \gamma_l~e^{-\gamma_l \|x - x'\|_{\mathcal{X}_{l-1}}^2}~(x - x').
\end{equation*}
\begin{align*}
\Delta_l\left({x_i}^{(l-1)}\right) &= \bigg[\left(\nabla^{(1)} k_l\left({x_i}^{(l-1)}, {x_1}^{(l-1)}\right)\right)^\top,\ldots, \left(\nabla^{(1)} k_l\left({x_i}^{(l-1)}, {x_n}^{(l-1)}\right)\right)^\top\bigg]^\top,\\
&= -2 \gamma_l\bigg[e^{-\gamma_l \left\|{x_i}^{(l-1)} - {x_1}^{(l-1)}\right\|_{\mathcal{X}_{l-1}}^2}\left({x_i}^{(l-1)} - {x_1}^{(l-1)}\right)^\top,\ldots\\
&\phantom{= -2 \gamma_l\bigg[}\ldots, e^{-\gamma_l \left\|{x_i}^{(l-1)} - {x_n}^{(l-1)}\right\|_{\mathcal{X}_{l-1}}^2}\left({x_i}^{(l-1)} - {x_n}^{(l-1)}\right)^\top\bigg]^\top,\\
\Delta_l\left({x_i}^{(l-1)}\right) &= -2\gamma_l~\tilde{K}_{l,i} \circ \left(\tilde{X}_i^{(l-1)} - X^{(l-1)}\right),
\end{align*}
where :
\begin{itemize}
\item $X^{(l-1)} \coloneqq \left(\left({x_1}^{(l-1)}\right)^\top, \ldots, \left({x_n}^{(l-1)}\right)^\top \right)^\top \in \mathbb{R}^{n \times d_{l-1}}$ stores the level $l-1$ representations of the $x_i$'s in rows
\item $\tilde{X}_i^{(l-1)} \coloneqq \left(\left({x_i}^{(l-1)}\right)^\top, \ldots, \left({x_i}^{(l-1)}\right)^\top \right)^\top \in \mathbb{R}^{n \times d_{l-1}}$ stores the level $l-1$ representation of $x_i$ $n$ times in rows
\item $\tilde{K}_{l, i} \in \mathbb{R}^{n \times n}$ is the $k_l$ Gram matrix between $X^{(l-1)}$ and $\tilde{X}_i^{(l-1)}$ \big(\textit{i.e.} $[\tilde{K}_{l, i}]_{s,t} = k_l\left({x_i}^{(l-1)}, {x_t}^{(l-1)}\right)$\big)
\item $\circ$ denotes the Hadamard (termwise) product for two matrices of the same shape
\end{itemize}
In practice, it is important to note that computing the $\Delta_l$ matrices with the Gaussian kernel needs not new calculations, but only uses already computed quantities : the level $l-1$ representations and their Gram matrix.

% 3.4.2. Polynomial kernel

\underline{Polynomial kernel :}
\begin{equation*}
\nabla^{(1)} k_l(x,x') = \nabla_x~\Big(\left(a \left\langle x, x'\right\rangle + b\right)^c\Big) = ca \Big(\left(a \left\langle x, x'\right\rangle + b\right)^{c-1}\Big)x'.
\end{equation*}
\begin{align*}
\Delta_l\left({x_i}^{(l-1)}\right) &= \bigg[\left(\nabla^{(1)} k_l\left({x_i}^{(l-1)}, {x_1}^{(l-1)}\right)\right)^\top,...~, \left(\nabla^{(1)} k_l\left({x_i}^{(l-1)}, {x_n}^{(l-1)}\right)\right)^\top\bigg]^\top,\\
&= ca \bigg[ \left(a \left\langle {x_i}^{(l-1)}, {x_1}^{(l-1)}\right\rangle + b\right)^{c-1}\left({x_1}^{(l-1)}\right)^\top,\ldots\\
&\qquad\quad\ldots, \left(a \left\langle {x_i}^{(l-1)}, {x_n}^{(l-1)}\right\rangle + b\right)^{c-1}\left({x_n}^{(l-1)}\right)^\top \bigg]^\top,\\
\Delta_l\left({x_i}^{(l-1)}\right) &= ca~\left(\tilde{K}_{l, i}\right)^{\frac{c-1}{c}} \circ X^{(l-1)},
\end{align*}
where we keep the notations introduced in the Gaussian kernel example for $X^{(l-1)}$, $\tilde{K}_{l, i}$ and $\circ$.
Note that the exponent on $\tilde{K}_{l, i}$ must be understood as a termwise power, and not a matrix multiplication power.

In practice, it is important to note that computing the $\Delta_l$ matrices with the polynomial kernel only requires a slight and cheap new calculation : putting the -~already computed~- Gram matrix at layer $l-1$ to the termwise power $(c-1)/c$.

\subsection{Detail of $N_L$ Computation}\label{sec:lin_sys}
\begin{align}\label{eq:F}
\left\langle x_j, x_{j'}\right\rangle_{\mathcal{X}_0} &= \Bigg\langle \sum_{i=1}^n \left(\mathcal{K}_L\left(x_j^{(L-1)},x_i^{(L-1)}\right) + n\lambda_L \delta_{ij}\right)\varphi_{L, i}~,\nonumber\\
&\quad\quad\sum_{i'=1}^n \left(\mathcal{K}_L\left(x_{j'}^{(L-1)},x_{i'}^{(L-1)}\right) + n\lambda_L \delta_{i'j'}\right)\varphi_{L, i'}\Bigg\rangle_{\mathcal{X}_0},\nonumber\\
&= \sum_{i,i'=1}^n \Bigg\langle \left(k_L\left(x_j^{(L-1)},x_i^{(L-1)}\right) + n\lambda_L \delta_{ij}\right)\varphi_{L, i}~,\nonumber\\
&\qquad\qquad~~\left(k_L\left(x_{j'}^{(L-1)},x_{i'}^{(L-1)}\right) + n\lambda_L \delta_{i'j'}\right)\varphi_{L, i'}\Bigg\rangle_{\mathcal{X}_0},\nonumber\\
\left\langle x_j, x_{j'}\right\rangle_{\mathcal{X}_0} &= \sum_{i,i'=1}^n \left(k_L\left(x_j^{(L-1)},x_i^{(L-1)}\right) + n\lambda_L \delta_{ij}\right)\nonumber\\
&\qquad\quad~~\left(k_L\left(x_{j'}^{(L-1)},x_{i'}^{(L-1)}\right) + n\lambda_L \delta_{i'j'}\right)\left\langle \varphi_{L,i},\varphi_{L,i'}\right\rangle_{\mathcal{X}_0}.
\end{align}
As a reminder, $N_L$ denotes the matrix such that $[N_L]_{i, i'} = \left\langle \varphi_{L,i},\varphi_{L,i'}\right\rangle_{\mathcal{X}_0}$.
Let $K_{in}$ denote the input Gram matrix such that $[K_{in}]_{j, j'} = \left\langle x_j, x_{j'}\right\rangle_{\mathcal{X}_0}$.
Finally, following notations of \Cref{sec:finite_dim} for $K_L$, and denoting $I_n$ the identity matrix on $\mathbb{R}^n$, equation \eqref{eq:F} may be rewritten as:
\begin{equation*}
[K_{in}]_{j, j'} = \sum_{i,i' = 1}^n [K_L + n\lambda_L I_n]_{j,i} [N_L]_{i, i'} [K_L + n\lambda_L I_n]_{i',j},
\end{equation*}
or equivalently:
\begin{equation*}
K_{in} = (K_L + n\lambda_L I_n)~N_L~(K_L + n\lambda_L I_n),
\end{equation*}
so that the computation of the desired linear products $\left\langle \varphi_{L,i},\varphi_{L,i'}\right\rangle_{\mathcal{X}_0}$ becomes straightforward:
\begin{equation}\label{eq:N_L}
N_L = (K_L + n\lambda_L I_n)^{-1}~K_{in}~(K_L + n\lambda_L I_n)^{-1}.
\end{equation}

% Differential

\subsection{Detail of Equation \eqref{eq:differential}}\label{sec:differential}
Since $\mathcal{X}_L$ is now infinite dimensional, $\textbf{Jac}_{x_i^{L}}(\varphi_{l_0, i_0})$ makes no more sense.
Nevertheless, $\varphi_{l, i}$ remains finite dimensional, and the distortion a scalar: a gradient does exist.
One is just forced to use the differential of $\|x_i - f_L \circ \ldots \circ f_1(x_i)\|_{\mathcal{X}_0}^2$ to make it appear.
As a reminder, the chain rule for the differentials reads : $d(g\circ f)(x) = dg(f(x)) \circ df(x)$.
Let us apply it with $g(\cdot) = \|\cdot\|_{\mathcal{X}_0}^2$ and $f: \varphi_{l_0, i_0} \mapsto x_i - x_i^{(L)}$.
Let $h \in \mathcal{X}_{l_0}$ and $h' \in \mathcal{X}_0$, we have:
\begin{equation*}
\Big(dg(y)\Big)(h') = 2\left\langle y,h'\right\rangle_{\mathcal{X}_0}.
\end{equation*}
\begin{align*}
\Big(df(\varphi_{l_0, i_0})\Big)(h) &= \left(d\left(x_i - \sum_{i'=1}^n k_L\left[x_i^{(L-1)}, x_{i'}^{(L-1)}\right]\varphi_{L, i'}\right)(\varphi_{l_0, i_0})\right)(h),\\
&= - \sum_{i'=1}^n \left(d\left(k_L\left[x_i^{(L-1)}, x_{i'}^{(L-1)}\right]\varphi_{L, i'}\right)(\varphi_{l_0, i_0})\right)(h),\\
&= - \sum_{i'=1}^n\left(d\left(k_L\left[x_i^{(L-1)}, x_{i'}^{(L-1)}\right]\right)(\varphi_{l_0, i_0})\right)(h)~\varphi_{L, i'},\\
\Big(df(\varphi_{l_0, i_0})\Big)(h) &= - \sum_{i'=1}^n \left\langle \nabla_{\varphi_{l_0, i_0}}k_L\left(x_i^{(L-1)},x_{i'}^{(L-1)}\right),h\right\rangle_{\mathcal{X}_{l_0}}\varphi_{L, i'}.
\end{align*}
Combining both expressions with $y = x_i - x_i^{(L)}$ gives:
\begin{align*}
\Big(d(\|x_i - f_L \circ \ldots \circ f_1(x_i)\|_{\mathcal{X}_0}^2)(\varphi_{l_0, i_0})\Big)(h) &= \Big(d(g\circ f)(\varphi_{l_0, i_0})\Big)(h),\\
&= \Big(dg\left(x_i - x_i^{(L)}\right)\Big)\circ\Big(df(\varphi_{l_0, i_0})\Big)(h),\\
&= 2\left\langle x_i - x_i^{(L)},- \sum_{i'=1}^n \left\langle \nabla_{\varphi_{l_0, i_0}}k_L\left(x_i^{(L-1)},x_{i'}^{(L-1)}\right),h\right\rangle_{\mathcal{X}_{l_0}}\varphi_{L, i'}\right\rangle_{\mathcal{X}_0},\\
&= -2\sum_{i'=1}^n \left\langle \nabla_{\varphi_{l_0, i_0}}k_L\left(x_i^{(L-1)},x_{i'}^{(L-1)}\right),h\right\rangle_{\mathcal{X}_{l_0}}\left\langle x_i - x_i^{(L)},\varphi_{L, i'}\right\rangle_{\mathcal{X}_0},\\
\Big(d(\|x_i - f_L \circ \ldots \circ f_1(x_i)\|_{\mathcal{X}_0}^2)(\varphi_{l_0, i_0})\Big)(h) &= \left\langle -2\sum_{i'=1}^n\left\langle x_i - x_i^{(L)},\varphi_{L, i'}\right\rangle_{\mathcal{X}_0}\nabla_{\varphi_{l_0, i_0}}k_L\left(x_i^{(L-1)},x_{i'}^{(L-1)}\right),h\right\rangle_{\mathcal{X}_{l_0}}.
\end{align*}
A direct identification leads to equation \eqref{eq:differential}.
\qed

\subsection{Solutions to Equations \eqref{eq:KRR} and Test Distortion}
Since we have assumed that $A_L$ is the identity operator on $\mathcal{X}_L$, equations \eqref{eq:KRR} simplify to:
\begin{equation}\label{eq:KRR_simplify}
\forall~i\in\llbracket n \rrbracket, \qquad \sum_{i'=1}^nW_{i,i'}~\varphi_{L,i'} = x_i,
\end{equation}
where $W = K_L + n\lambda_L I_n$.
It is then easy to show that the
\begin{equation*}
\varphi_{L, i'} = \sum_{i=1}^n \left[W^{-1}\right]_{i',i}~x_i \qquad \forall~i' \in \llbracket n \rrbracket
\end{equation*}
are solutions to equations \eqref{eq:KRR_simplify} and therefore to equations \eqref{eq:KRR}.
Note that using this expansion directly leads to equation \eqref{eq:N_L}.
But more interestingly, this new writing allows for computing the distortion on a test set.
Indeed, let $x \in \mathcal{X}_0$, one has:
\begin{align*}
\left\|x - f_L \circ \ldots \circ f_1(x)\right\|_{\mathcal{X}_0}^2 &= \left\|x - f_L\left(x^{(L-1)}\right)\right\|_{\mathcal{X}_0}^2,\\
&= \left\|x\right\|_{\mathcal{X}_0}^2 + \left\|f_L\left(x^{(L-1)}\right)\right\|_{\mathcal{X}_0}^2 - 2 \left\langle x , f_L\left(x^{(L-1)}\right)\right\rangle_{\mathcal{X}_0},\\
&= \left\|x\right\|_{\mathcal{X}_0}^2 + \left\|\sum_{i=1}^n k_L\left(x^{(L-1)}, x_i^{(L-1)}\right)\varphi_{L, i}\right\|_{\mathcal{X}_0}^2 - 2 \left\langle x , \sum_{i=1}^n k_L\left(x^{(L-1)}, x_i^{(L-1)}\right)\varphi_{L, i}\right\rangle_{\mathcal{X}_0},\\
&= \left\|x\right\|_{\mathcal{X}_0}^2 + \sum_{i,j=1}^n k_L\left(x^{(L-1)}, x_i^{(L-1)}\right) k_L\left(x^{(L-1)}, x_j^{(L-1)}\right) \left\langle \varphi_{L, i}, \varphi_{L, j}\right\rangle_{\mathcal{X}_0}\\
&\phantom{=~}- 2 \sum_{i=1}^n k_L\left(x^{(L-1)}, x_i^{(L-1)}\right) \left\langle x , \varphi_{L, i}\right\rangle_{\mathcal{X}_0},\\
\left\|x - f_L \circ \ldots \circ f_1(x)\right\|_{\mathcal{X}_0}^2 &= \left\|x\right\|_{\mathcal{X}_0}^2 + \sum_{i,j=1}^n k_L\left(x^{(L-1)}, x_i^{(L-1)}\right) k_L\left(x^{(L-1)}, x_j^{(L-1)}\right) \left\langle \varphi_{L, i}, \varphi_{L, j}\right\rangle_{\mathcal{X}_0}\\
&\phantom{=~}- 2\sum_{i,j=1}^n k_L\left(x^{(L-1)}, x_i^{(L-1)}\right) \left[W^{-1}\right]_{i,j} \left\langle x , x_j\right\rangle_{\mathcal{X}_0}.
\end{align*}
Just like in \Cref{sec:inf_dim} and \Cref{sec:lin_sys}, knowing the scalar products in $\mathcal{X}_0$ is the only thing we need to compute the test distortion (all other quantities are finite dimensional and thus computable).

%%%%%%%%%%%%%%%%%%%%%%%%%%%%
%  ADDITIONAL EXPERIMENTS  %
%%%%%%%%%%%%%%%%%%%%%%%%%%%%

\clearpage
\section{Additional Experiments}
\subsection{2D Data}\label{sec:add_expe}

\Cref{fig:1D_behaviour} gives a look on the algorithm behavior on 1D data. Results on  1D data are displayed and analyzed here as they are easily understandable. Indeed, one dimension of the plot (the $x$ axis) is used to display the original 1D points (the crosses), while their representations (the $f(x_i)$) vary along the $y$ axis. As soon as the original point or the representation needs more than 1 dimension to be plotted, a 2D plot lacks of dimensions to correctly display the behavior of the algorithm. Original data (to be represented) are sampled from 2 Gaussian distributions, of standard deviation 0.1, and with expected value 0 and 2 respectively.

\Cref{fig:enc_ev} and \Cref{fig:dec_ev} show the evolution of the encoding / decoding functions along the iterations of the algorithm. From the initial yellow representation function, obtained by uniform weights, the algorithm learns the black function, which seems satisfying in two ways. First, the representations of the two clusters are easily separable. Points from the first blue cluster (i.e. drawn from the Gaussian centered at 0) have positive representations, while points from the red one (i.e. drawn from the Gaussian centered at 2) have negative ones. If computed in a clustering purpose, the representation thus gives an easy criterion to distinguish the two clusters. Second, in order to be able to reconstruct any point, one must observe variability within each cluster. This way, the reconstruction function can easily reassign every point. On the contrary, the yellow representation function represents all points by almost the same value, which leads necessarily to a uniform (and bad) reconstruction.

\Cref{fig:1d_gaus_1d_enc} shows another 1D representation of the two clusters, while \Cref{fig:1d_gaus_2d_enc} shows a 2D encoding of these points. Interestingly, the two components of the 2D representation are highly correlated. This can be interpreted as the fact that a 2D descriptor is over-parameterizing a 1D point.

\begin{figure*}[hb]
\centering
\subfigure[Encoding evolution during fitting\label{fig:enc_ev}]{\includegraphics[width=0.45\columnwidth]{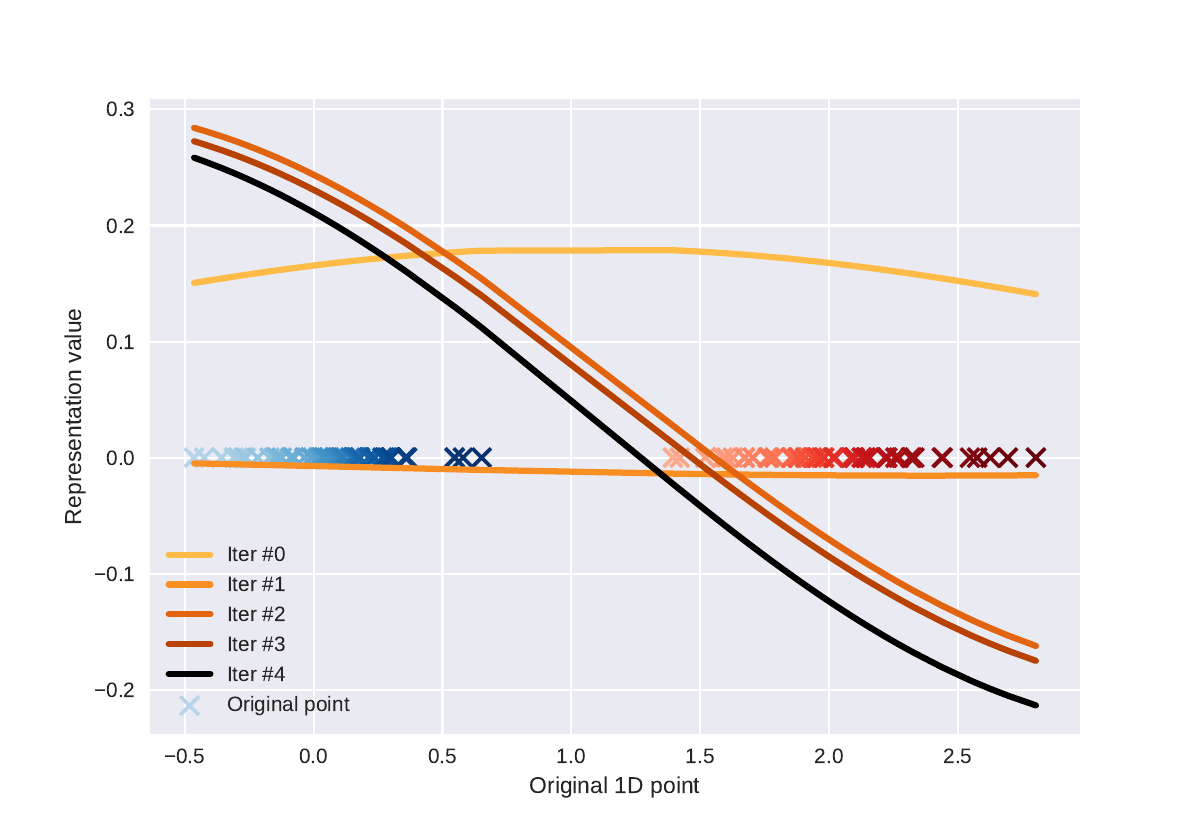}}
\subfigure[Decoding evolution during fitting\label{fig:dec_ev}]{\includegraphics[width=0.45\columnwidth]{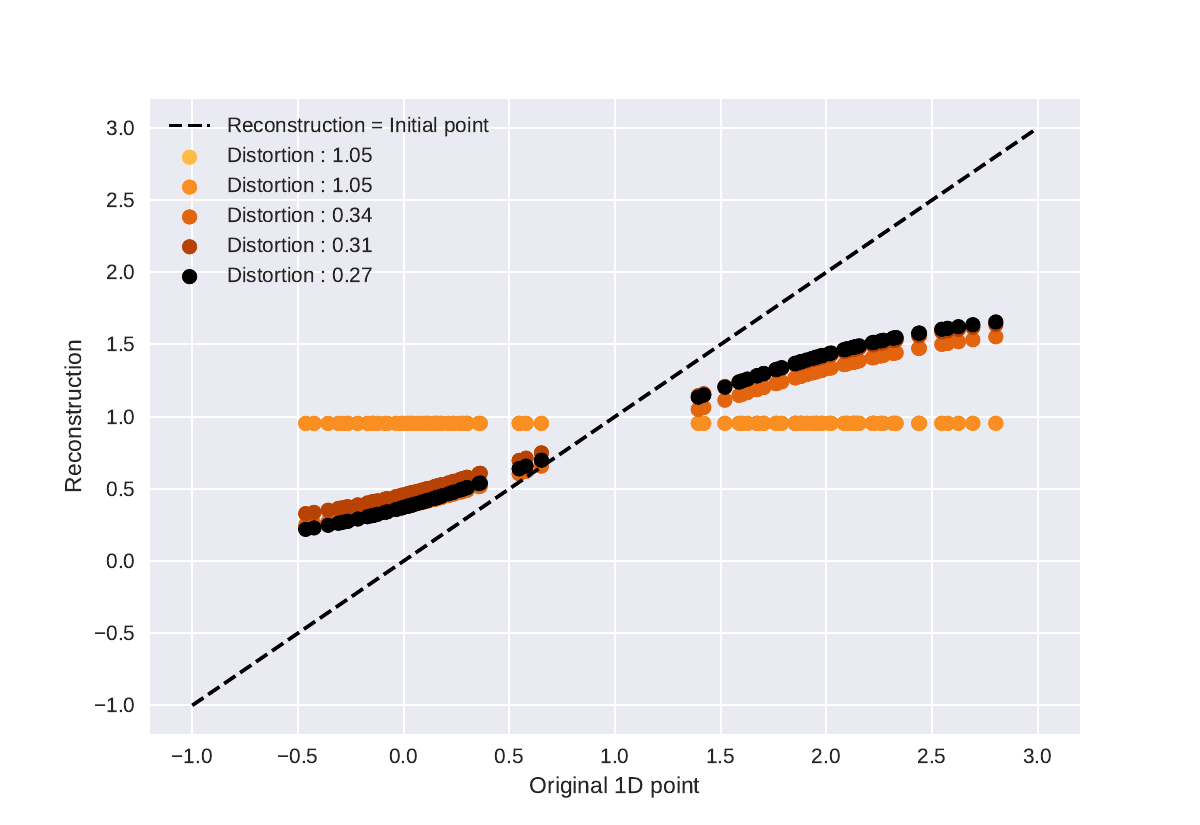}}
\subfigure[1D Gaussian clusters and 1D representation\label{fig:1d_gaus_1d_enc}]{\includegraphics[width=0.45\columnwidth]{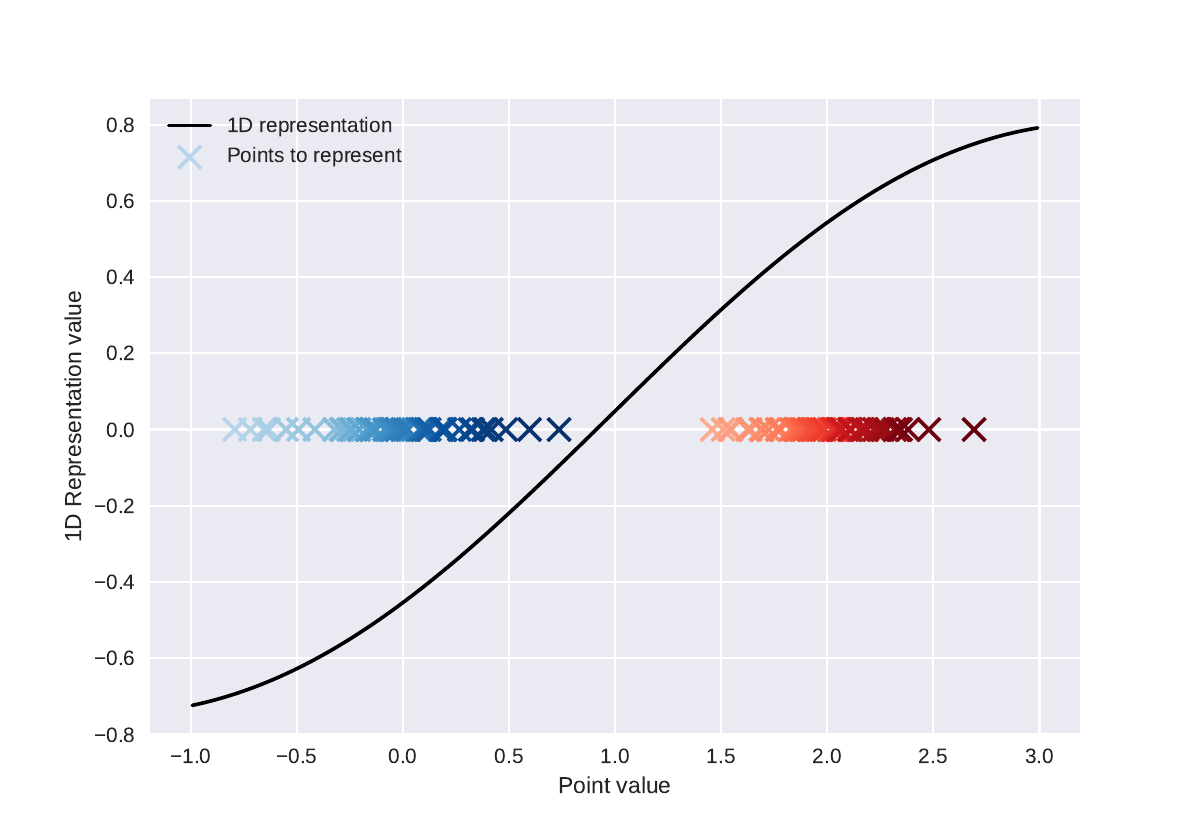}}
\subfigure[2D representation of the clusters\label{fig:1d_gaus_2d_enc}]{\includegraphics[width=0.45\columnwidth]{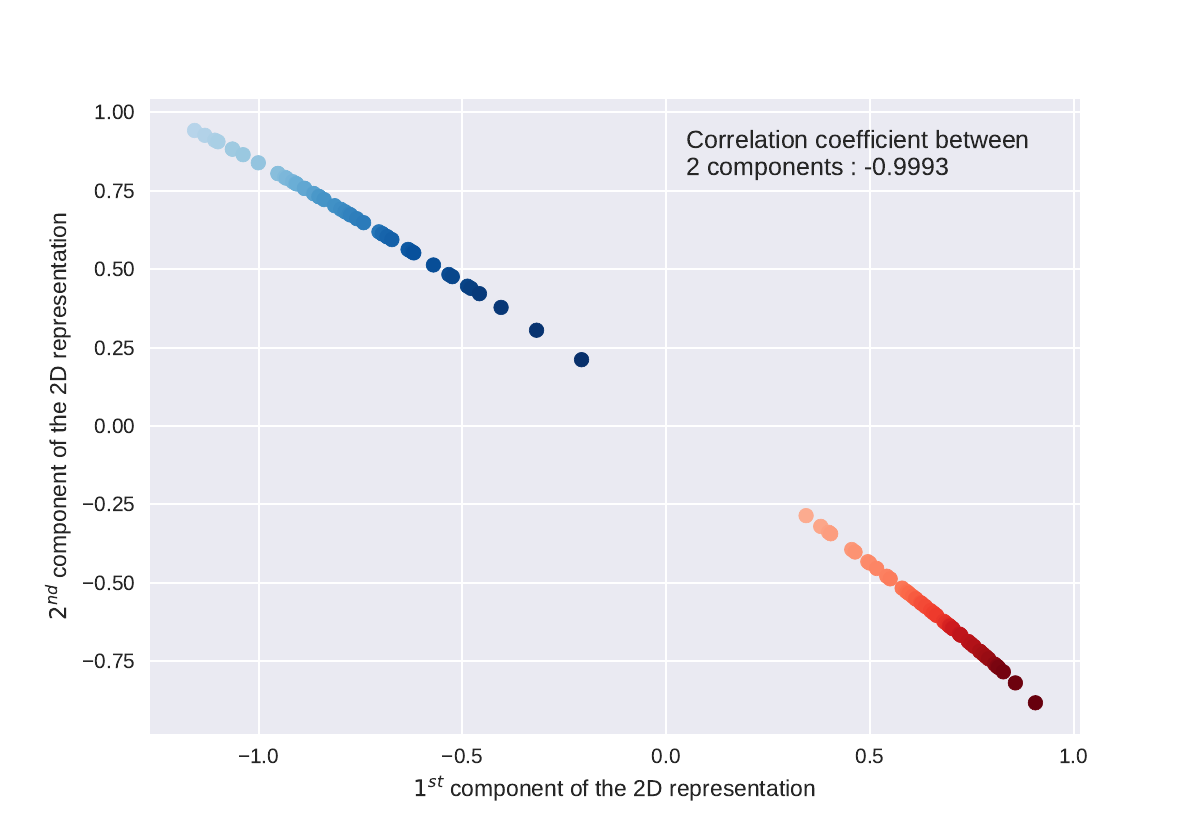}}
\caption{Algorithm behavior on 1D data}
\label{fig:1D_behaviour}
\end{figure*}
\clearpage

\Cref{fig:Gaussian_behaviour} shows the algorithm's behavior on Gaussian clusters. Whenever original points and their representations cannot be displayed on the same graph (\textit{i.e.} when whether the original data or its representation is of dimension more than 2), the colormap helps linking them. In \Cref{fig:2d_2gaus_ori}, the original 2D data are plotted, while \Cref{fig:2d_2gaus_enc} shows their 1D representations. The colormap has been established according to the value of this representation. First, the two clusters remain well separated in the representation space (positive/negative representations). But what is really interesting is how they are separated. The lighter the blue points are, the most negative representation they have, or in other terms, the \textit{most certain} they are to be in the blue cluster. Similarly, the darker the red points are, the most positive representation they have. When looking at these points on \Cref{fig:2d_2gaus_ori}, one sees that it matches the distribution: light blue points are the most distant from the red cluster, and conversely for the dark red ones. The algorithm has found the direction that discriminates the two clusters. Similar results are shown for 3 Gaussian clusters on \Cref{fig:2d_3gaus_ori} and \Cref{fig:2d_3gaus_1d_enc}.

\begin{figure*}[hb]
\centering
\subfigure[2D Gaussian clusters\label{fig:2d_2gaus_ori}]{\includegraphics[width=0.45\columnwidth]{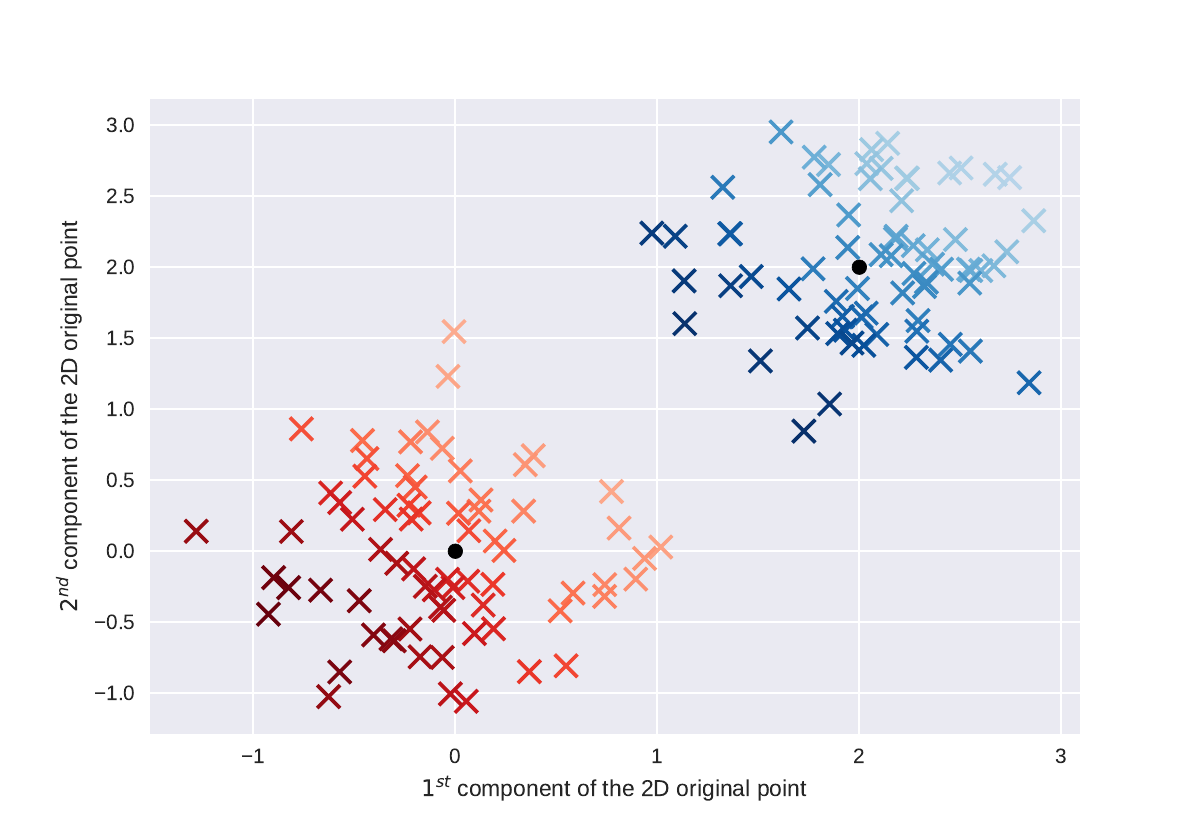}}
\subfigure[1D representation of the clusters\label{fig:2d_2gaus_enc}]{\includegraphics[width=0.45\columnwidth]{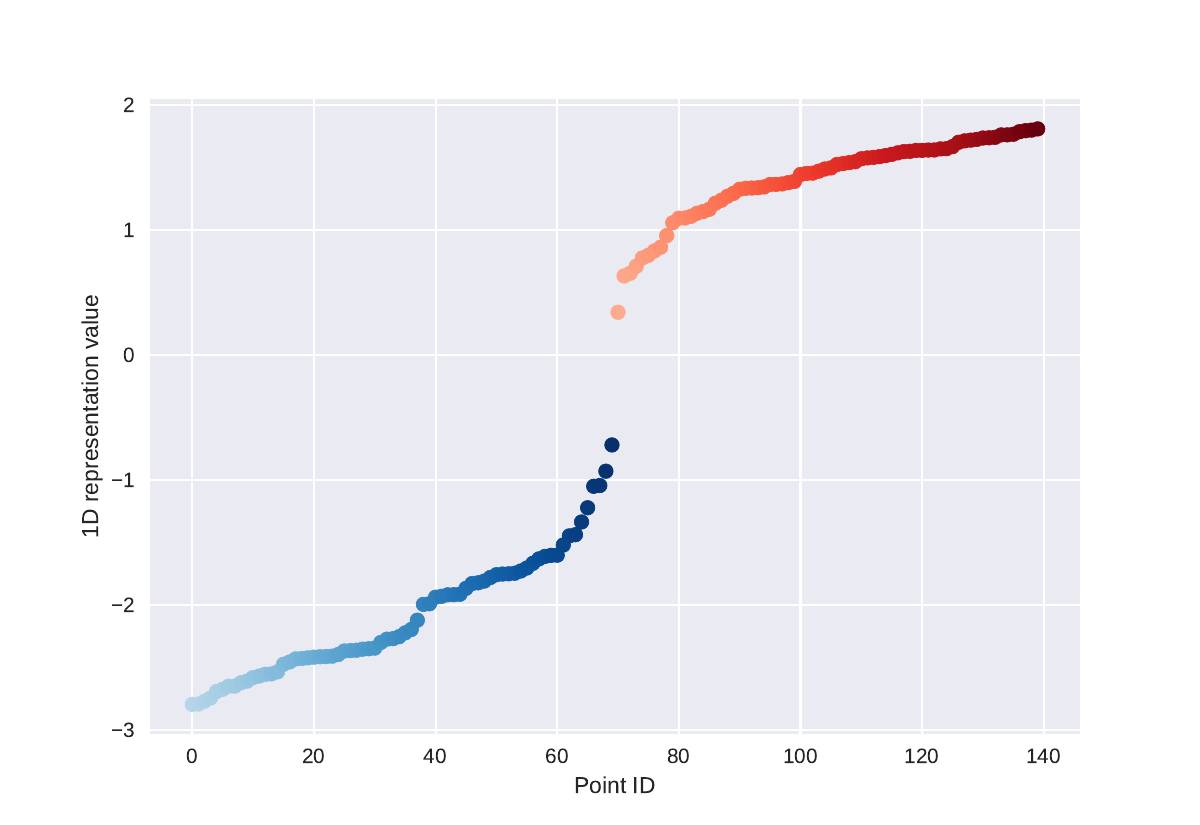}}
\subfigure[2D Gaussian clusters (3)\label{fig:2d_3gaus_ori}]{\includegraphics[width=0.45\columnwidth]{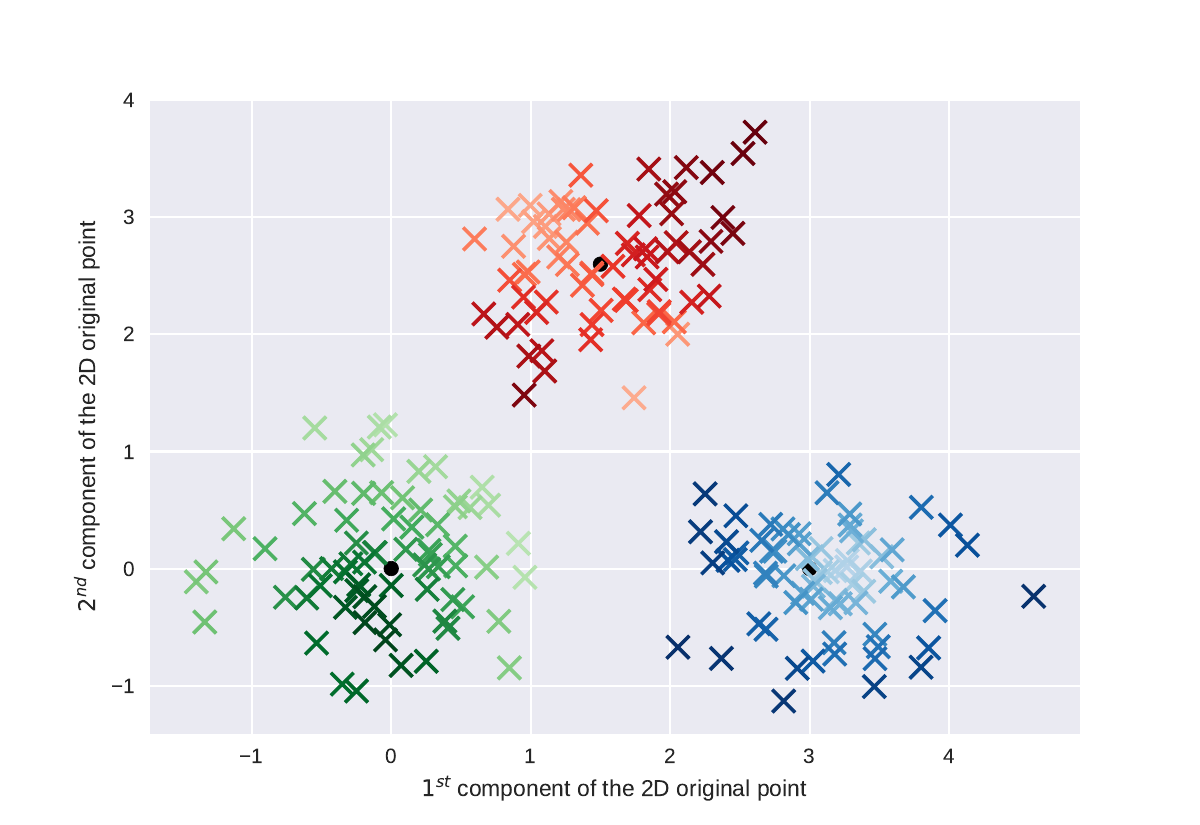}}
\subfigure[1D representation of the clusters\label{fig:2d_3gaus_1d_enc}]{\includegraphics[width=0.45\columnwidth]{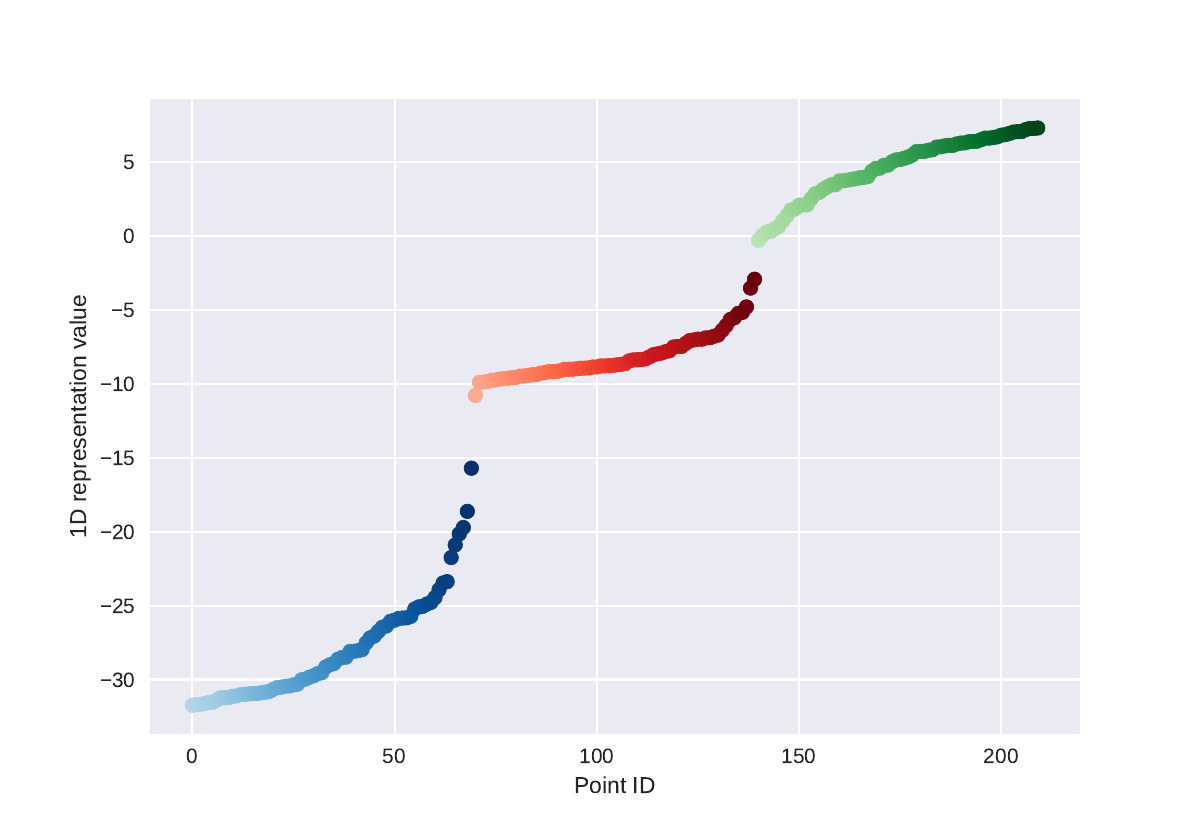}}
\caption{Algorithm behavior on Gaussian clusters}
\label{fig:Gaussian_behaviour}
\end{figure*}

\bigskip

Finally, \Cref{fig:2moons_behaviour} shows the algorithm's behavior on the so called \textit{two moons dataset}. 2D original points (\Cref{fig:2moons_ori1} and \Cref{fig:2moons_ori2}, colored differently according to the representation on their right) are first mapped to a 1D representation (\Cref{fig:2moons_1denc}). Just as for the 3 concentric circles example, this 1D representation is discriminative, also with intra-cluster variability in order to reconstruct properly. The 2D re-representation on \Cref{fig:2moons_2denc} shows again the disentangling properties of the KAE.

\begin{figure*}[ht]
\centering
\subfigure[Two moons dataset, colored w.r.t. its 1D representation\label{fig:2moons_ori1}]{\includegraphics[width=0.45\columnwidth]{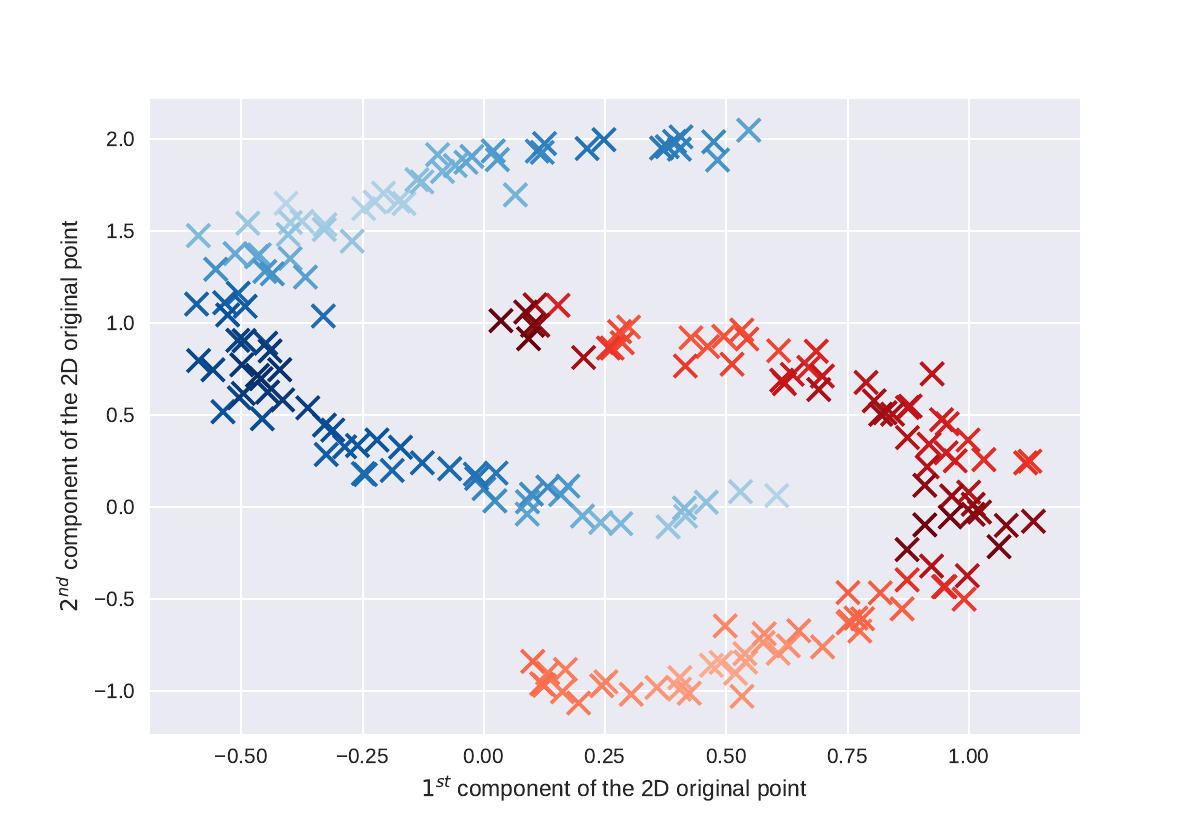}}
\subfigure[1D representation of the 2 moons\label{fig:2moons_1denc}]{\includegraphics[width=0.45\columnwidth]{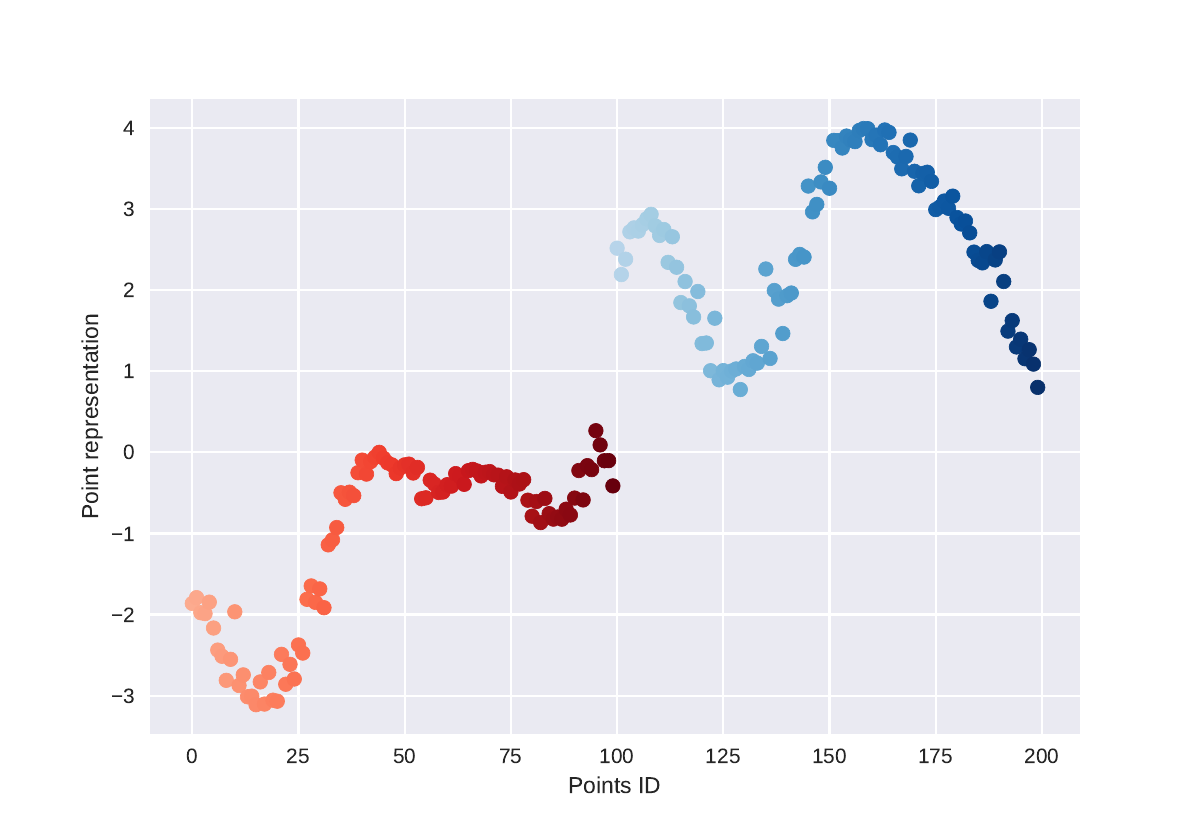}}
\subfigure[Two moons dataset, colored w.r.t. its 2D representation\label{fig:2moons_ori2}]{\includegraphics[width=0.45\columnwidth]{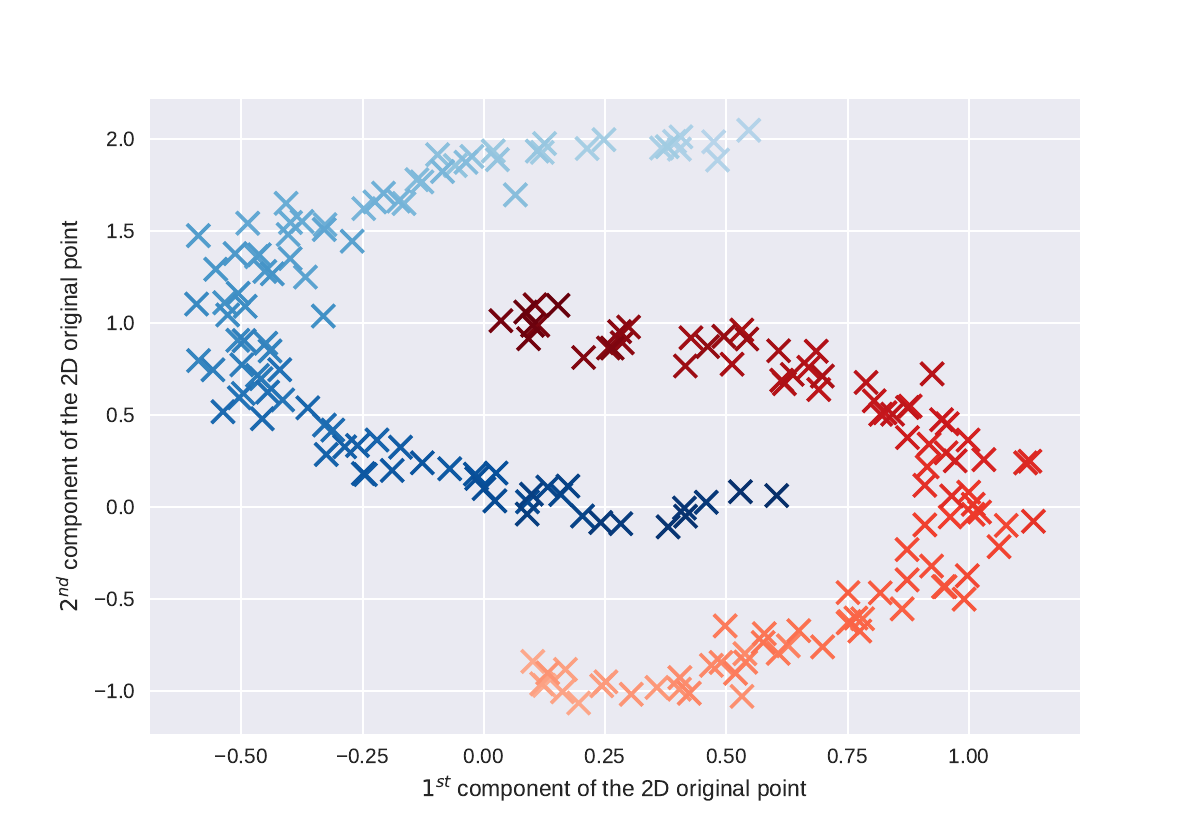}}
\subfigure[2D representation of the 2 moons\label{fig:2moons_2denc}]{\includegraphics[width=0.45\columnwidth]{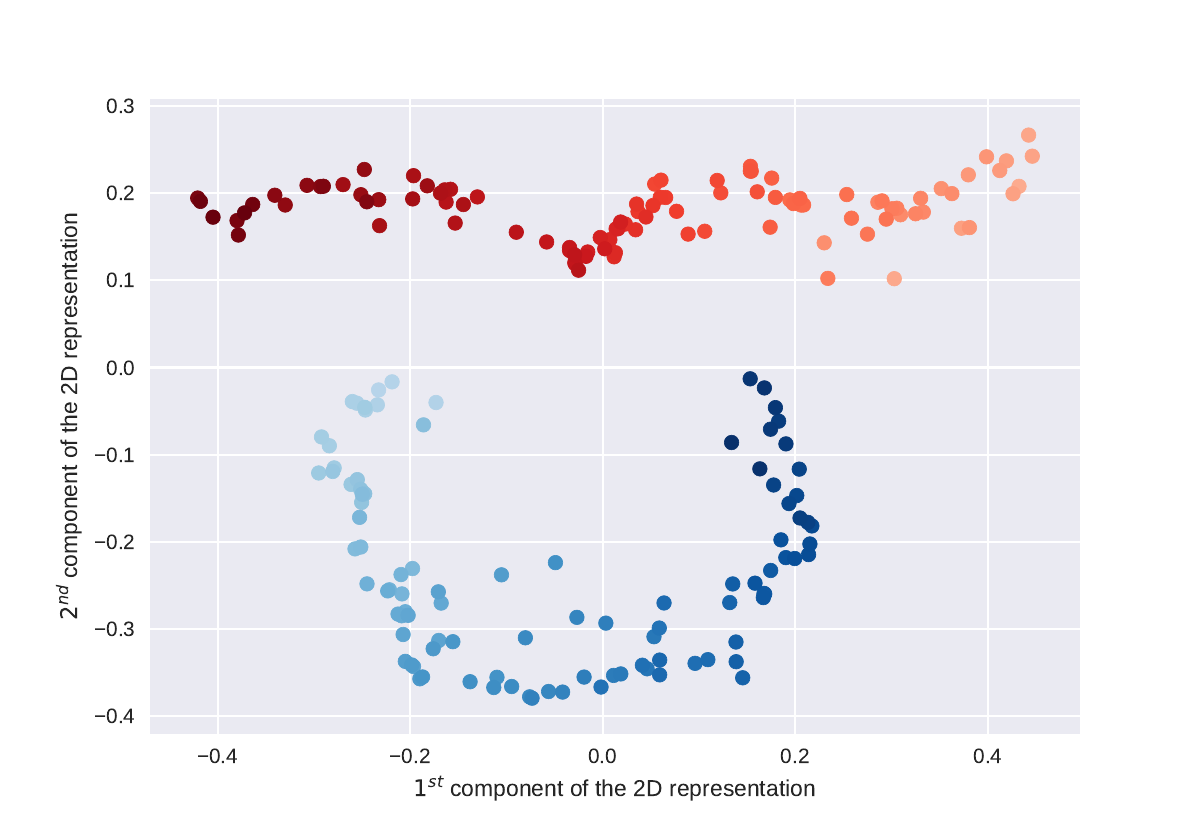}}
\caption{Algorithm behavior on the 2 moons dataset}
\label{fig:2moons_behaviour}
\end{figure*}
\clearpage

\subsection{NCI Data}\label{sec:add_nci}
\subsubsection{All strategies on 8 cancers graph}

\begin{figure*}[ht]
\centering
\includegraphics[height=0.64\columnwidth, angle=-90]{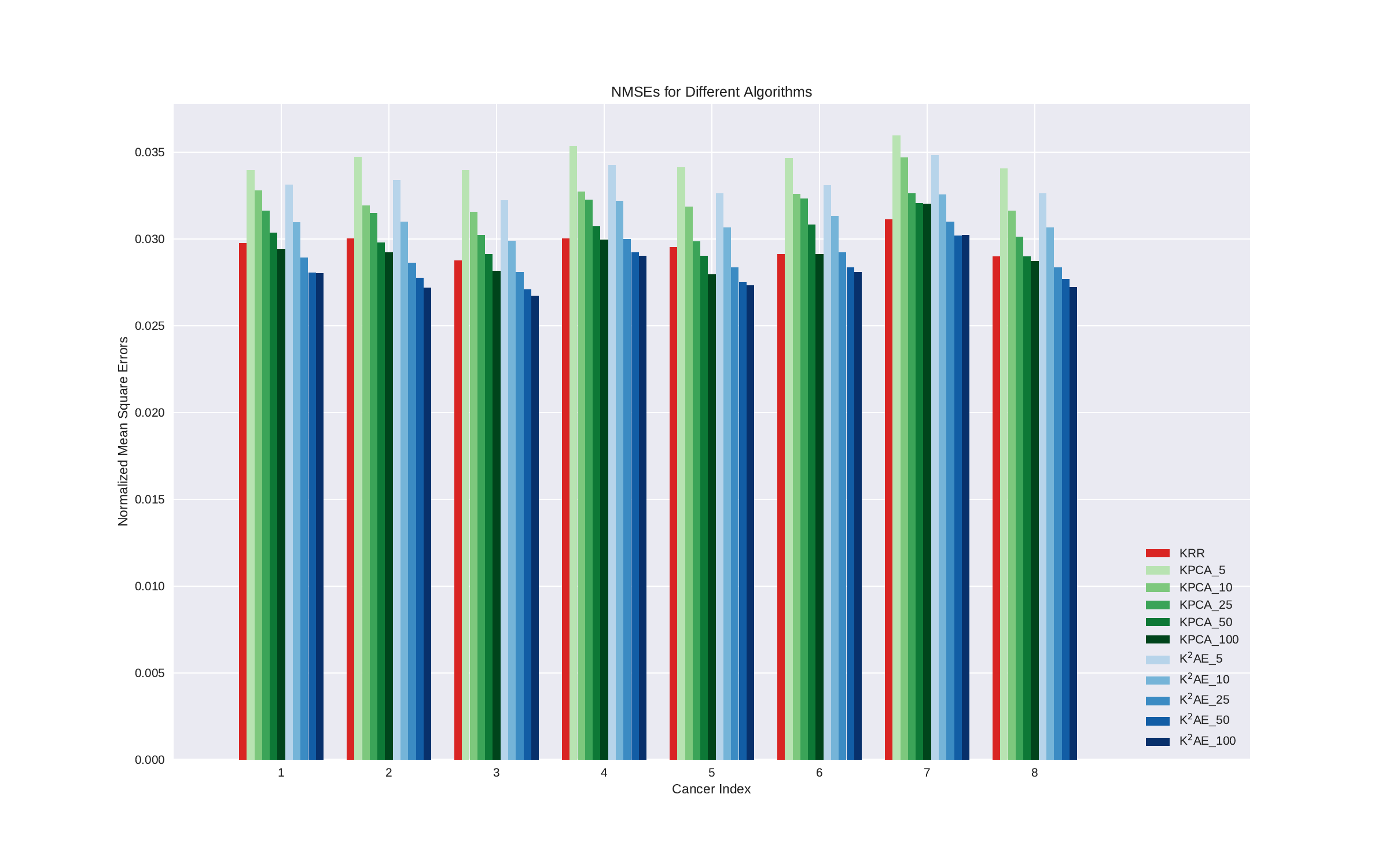}
\caption{Performance of the Different Strategies on 8 Cancers}
\label{fig:bar_plot}
\end{figure*}
\clearpage

\subsection{5 strategies on 59 cancers table}
\begin{table}[h]
\scriptsize
\caption{NMSEs on Molecular Activity for Different Types of Cancer}
\label{tab:NCI2}
\vskip 0.15in
\begin{center}
\begin{sc}
\begin{tabular}{cccccc}
\toprule
& KRR      & KPCA 10 + RF & KPCA 50 + RF & $\text{K}^2$AE 10 + RF & $\text{K}^2$AE 50 + RF\\
\midrule
Cancer 01 & 0.02978  & 0.03279  & 0.03035  & 0.03097  & \textbf{0.02808}\\
Cancer 02 & 0.03004  & 0.03194  & 0.02978  & 0.03099  & \textbf{0.02775}\\
Cancer 03 & 0.02878  & 0.03155  & 0.02914  & 0.02989  & \textbf{0.02709}\\
Cancer 04 & 0.03003  & 0.03274  & 0.03074  & 0.03218  & \textbf{0.02924}\\
Cancer 05 & 0.02954  & 0.03185  & 0.02903  & 0.03065  & \textbf{0.02754}\\
Cancer 06 & 0.02914  & 0.03258  & 0.03083  & 0.03134  & \textbf{0.02838}\\
Cancer 07 & 0.03113  & 0.03468  & 0.03207  & 0.03257  & \textbf{0.03018}\\
Cancer 08 & 0.02899  & 0.03162  & 0.02898  & 0.03065  & \textbf{0.02770}\\
Cancer 09 & 0.02860  & 0.02992  & 0.02804  & 0.02872  & \textbf{0.02627}\\
Cancer 10 & 0.02987  & 0.03291  & 0.03111  & 0.03170  & \textbf{0.02910}\\
Cancer 11 & 0.03035  & 0.03258  & 0.03095  & 0.03188  & \textbf{0.02900}\\
Cancer 12 & 0.03178  & 0.03461  & 0.03153  & 0.03253  & \textbf{0.02983}\\
Cancer 13 & 0.03069  & 0.03338  & 0.03104  & 0.03162  & \textbf{0.02857}\\
Cancer 14 & 0.03046  & 0.03340  & 0.03102  & 0.03135  & \textbf{0.02862}\\
Cancer 15 & 0.02910  & 0.03221  & 0.03066  & 0.03131  & \textbf{0.02806}\\
Cancer 16 & 0.02956  & 0.03220  & 0.02958  & 0.03060  & \textbf{0.02779}\\
Cancer 17 & 0.03004  & 0.03413  & 0.03140  & 0.03145  & \textbf{0.02869}\\
Cancer 18 & 0.02954  & 0.03195  & 0.03005  & 0.03108  & \textbf{0.02805}\\
Cancer 19 & 0.03003  & 0.03211  & 0.03079  & 0.03178  & \textbf{0.02832}\\
Cancer 20 & 0.02911  & 0.03179  & 0.03041  & 0.03085  & \textbf{0.02769}\\
Cancer 21 & 0.02963  & 0.03275  & 0.03023  & 0.03152  & \textbf{0.02837}\\
Cancer 22 & 0.03075  & 0.03391  & 0.03089  & 0.03263  & \textbf{0.02958}\\
Cancer 23 & 0.03006  & 0.03286  & 0.02983  & 0.03109  & \textbf{0.02760}\\
Cancer 24 & 0.03075  & 0.03398  & 0.03112  & 0.03242  & \textbf{0.02894}\\
Cancer 25 & 0.02977  & 0.03307  & 0.03054  & 0.03159  & \textbf{0.02824}\\
Cancer 26 & 0.03083  & 0.03358  & 0.03132  & 0.03206  & \textbf{0.02959}\\
Cancer 27 & 0.03083  & 0.03347  & 0.03116  & 0.03230  & \textbf{0.02974}\\
Cancer 28 & 0.03061  & 0.03256  & 0.03116  & 0.03185  & \textbf{0.02918}\\
Cancer 29 & 0.03056  & 0.03360  & 0.03147  & 0.03181  & \textbf{0.02892}\\
Cancer 30 & 0.03099  & 0.03288  & 0.03100  & 0.03181  & \textbf{0.02906}\\
Cancer 31 & 0.03082  & 0.03361  & 0.03161  & 0.03242  & \textbf{0.02986}\\
Cancer 32 & 0.03233  & 0.03562  & 0.03300  & 0.03422  & \textbf{0.03158}\\
Cancer 33 & 0.03065  & 0.03208  & 0.03045  & 0.03142  & \textbf{0.02909}\\
Cancer 34 & 0.03326  & 0.03668  & 0.03423  & 0.03486  & \textbf{0.03183}\\
Cancer 35 & 0.03292  & 0.03587  & 0.03393  & 0.03450  & \textbf{0.03146}\\
Cancer 36 & 0.03068  & 0.03389  & 0.03122  & 0.03249  & \textbf{0.02925}\\
Cancer 37 & 0.03023  & 0.03310  & 0.03061  & 0.03130  & \textbf{0.02878}\\
Cancer 38 & 0.03100  & 0.03487  & 0.03156  & 0.03327  & \textbf{0.02974}\\
Cancer 39 & 0.02989  & 0.03288  & 0.03149  & 0.03148  & \textbf{0.02865}\\
Cancer 40 & 0.03166  & 0.03525  & 0.03201  & 0.03352  & \textbf{0.03010}\\
Cancer 41 & 0.03139  & 0.03501  & 0.03203  & 0.03316  & \textbf{0.03025}\\
Cancer 42 & 0.03010  & 0.03251  & 0.03013  & 0.03072  & \textbf{0.02807}\\
Cancer 43 & 0.03042  & 0.03324  & 0.03062  & 0.03144  & \textbf{0.02806}\\
Cancer 44 & 0.02838  & 0.03045  & 0.02821  & 0.02927  & \textbf{0.02679}\\
Cancer 45 & 0.02910  & 0.03085  & 0.02895  & 0.02970  & \textbf{0.02651}\\
Cancer 46 & 0.02969  & 0.03258  & 0.02996  & 0.03111  & \textbf{0.02834}\\
Cancer 47 & 0.03148  & 0.03438  & 0.03346  & 0.03286  & \textbf{0.03056}\\
Cancer 48 & 0.03272  & 0.03640  & 0.03397  & 0.03425  & \textbf{0.03197}\\
Cancer 49 & 0.03305  & 0.03392  & 0.03329  & 0.03334  & \textbf{0.03148}\\
Cancer 50 & 0.03229  & 0.03637  & 0.03300  & 0.03404  & \textbf{0.03155}\\
Cancer 51 & 0.02943  & 0.03188  & 0.03028  & 0.03072  & \textbf{0.02857}\\
Cancer 52 & 0.03309  & 0.03420  & 0.03252  & 0.03335  & \textbf{0.03130}\\
Cancer 53 & 0.03170  & 0.03340  & 0.03105  & 0.03170  & \textbf{0.02843}\\
Cancer 54 & 0.03189  & 0.03439  & 0.03164  & 0.03345  & \textbf{0.03036}\\
Cancer 55 & 0.03082  & 0.03339  & 0.03146  & 0.03207  & \textbf{0.02892}\\
Cancer 56 & 0.03026  & 0.03327  & 0.03041  & 0.03185  & \textbf{0.02901}\\
Cancer 57 & 0.02962  & 0.03237  & 0.02990  & 0.03162  & \textbf{0.02855}\\
Cancer 58 & 0.02883  & 0.03200  & 0.02978  & 0.03058  & \textbf{0.02783}\\
Cancer 59 & 0.02936  & 0.03208  & 0.02914  & 0.03032  & \textbf{0.02750}\\\bottomrule
\end{tabular}
\end{sc}
\end{center}
\vskip -0.1in
\end{table}

%% file: 8-KAE_Rademacher.tex
%%%%%%%%%%%%%%%%%%%%
%  STANDARD BOUND  %
%%%%%%%%%%%%%%%%%%%%

\subsubsection{Standard Rademacher Generalization Bound}

Let loss $\ell$ denote the squared norm on $\mathcal{X}_0$: $\forall x \in \mathcal{X}_0, \ell(x) = \|x\|_{\mathcal{X}_0}^2$.
Notice that, on the set considered, the mapping $\ell$ is $2M$-Lipschitz, and: $\ell(x_i - h(x_i)) - \ell(x_{i'} - h(x_{i'})) \le 4M^2$.
Hence, by applying McDiarmid's inequality, together with standard arguments in the statistical learning literature (symmetrization/randomization tricks, see \textit{e.g.} Theorem 3.1 in \citet{mohri2012foundations}), one may show that, for any $\delta\in (0,1)$, we have with probability at least $1 - \delta$:
\begin{equation}\label{eq:Rade_1}
\frac{1}{2} \left(\epsilon(\hat{h}_n) - \epsilon^*\right) \le \sup_{h \in \mathcal{H}_{s, t}}| \epsilon(h) - \hat{\epsilon}_n(h) | \le 2 \widehat{\mathscr{R}}_n\Big(\big(\ell \circ (\text{id} - \mathcal{H}_{s, t})\big)(S)\Big) + 12M^2 \sqrt{\frac{\ln \frac{2}{\delta}}{2n}}.
\end{equation}

The subsequent results shall provide tools to bound the quantity $\widehat{\mathscr{R}}_n\Big(\big(\ell \circ (\text{id} - \mathcal{H}_{s, t})\big)(S)\Big)$ properly.

%%%%%%%%%%%%%%%%%%%%%%%%%%%%%%%%%%%%%%
%  OPERATIONS ON RADEMACHER AVERAGE  %
%%%%%%%%%%%%%%%%%%%%%%%%%%%%%%%%%%%%%%

\subsubsection{Operations on the Rademacher Average}

As a first go, we state a preliminary lemma that establishes a comparison between Rademacher and Gaussian averages.
\begin{lemma}\label{lem:comp}
We have: $\forall n\geq 1$,
\begin{equation*}\label{eq:comparison_RG}
\widehat{\mathscr{R}}_n(\mathcal{C}(S)) \le \sqrt{\frac{\pi}{2}}~\widehat{\mathscr{G}}_n(\mathcal{C}(S)).
\end{equation*}
\end{lemma}
\begin{proof}
The proof is based on the fact that $\gamma_{i, k}$ and $\sigma_{i, k} \left|\gamma_{i, k}\right|$ have the same distribution, combined with Jensen's inequality.
See also Lemma 4.5 in \citet{ledoux1991probability}.
\end{proof}

Hence, the application of the lemma above yields:
\begin{align}
\widehat{\mathscr{R}}_n\Big(\big(\ell \circ (\text{id} - \mathcal{H}_{s, t})\big)(S)\Big) &\le 2\sqrt{2}M~\widehat{\mathscr{R}}_n\Big(\big(\text{id} - \mathcal{H}_{s, t}\big)(S)\Big),\label{eq:foo1}\\
&\le 2\sqrt{2}M~\left[\widehat{\mathscr{R}}_n\Big(\{\text{id}\}(S)\Big) + \widehat{\mathscr{R}}_n\Big(\mathcal{H}_{s, t}(S)\Big)\right],\label{eq:foo2}\\
&\le 2\sqrt{2}M~\widehat{\mathscr{R}}_n\Big(\mathcal{H}_{s, t}(S)\Big),\label{eq:foo3}\\
\widehat{\mathscr{R}}_n\Big(\big(\ell \circ (\text{id} - \mathcal{H}_{s, t})\big)(S)\Big) &\le 2\sqrt{\pi}M~\widehat{\mathscr{G}}_n\Big(\mathcal{H}_{s, t}(S)\Big),\label{eq:foo4}
\end{align}

where \eqref{eq:foo1} directly results from Corollary 4 in \citet{maurer2016vector} (observing that, even if they do not take their values in $\ell_2(\mathbb{N})$ but in the separable Hilbert space $\mathcal{X}_0$, the functions $h(x)$ can replaced by the square-summable sequence $\left(\left\langle h(x), e_k\right\rangle\right)_{k \in \mathbb{N}}$) and
\eqref{eq:foo4} is a consequence of \Cref{lem:comp}.

It now remains to bound
$\widehat{\mathscr{G}}_n\Big(\mathcal{H}_{s, t}(S)\Big)$ using an extension of a result established in \citet{maurer2014chain} and applying to classes of functions valued in $\mathbb{R}^m$ only, while functions in $\mathcal{H}_{s, t}$ are Hilbert-valued.

%%%%%%%%%%%%%%%%%%%%%%%%%%%%%%%%%%%%%%
%  Extension of Maurer's Chain Rule  %
%%%%%%%%%%%%%%%%%%%%%%%%%%%%%%%%%%%%%%

\subsubsection{Extension of Maurer's Chain Rule}

The result stated below extends Theorem 2 in \citet{maurer2014chain} to the Hilbert-valued situation.

\begin{theorem}\label{thm:TIS}
Let $H$ be a Hilbert space, $X$ a $H$-valued Gaussian random vector, and $f: H \rightarrow \mathbb{R}$ a $L$-Lipschitz mapping.
We have:
\begin{equation*}
\forall t > 0, \qquad \mathbb{P}\Big(\left|f(X) - \mathbb{E}f(X)\right| > t\Big) \le \exp\left(- \frac{2 t^2}{\pi^2 L^2}\right).
\end{equation*}
\end{theorem}
\begin{proof}
It is a direct extension of Corollary 2.3 in \citet{pisier1986probabilistic}, which states the result for $H = \mathbb{R}^N$ only, observing that the proof given therein actually makes no use of the assumption of finite dimensionality of $H$, and thus remains valid in our case.
%
%The reason why authors did not establish this general version in their work is probably because they only needed the $\mathbb{R}^N$ version for their purpose.
%
Up to constants, it can also be viewed an extension of Theorem 4 in \citet{maurer2014chain}.
\end{proof}

We now introduce quantities involved in the rest of the analysis, see Definition 1 in \citet{maurer2014chain}.

\begin{definition}
Let $Y \subset \mathbb{R}^n$, $H$ be a Hilbert space, $Z \subset H$, and $\gamma$ be a $H$-valued standard Gaussian variable/process. We set:
\begin{equation*}
D(Y) = \sup_{y, y' \in Y} \|y - y'\|_{\mathbb{R}^n},
\end{equation*}
\begin{equation*}
G(Z) = \sup_{z \in Z} \mathbb{E}_\gamma \left[ \left\langle \gamma, z\right\rangle_H \right].
\end{equation*}
If $\mathcal{H}$ a class of functions from $Y$ to $H$, we set:
\begin{equation*}
L(\mathcal{H}, Y) = \sup_{h \in \mathcal{H}}~\sup_{y, y' \in Y,~y \ne y'} \frac{\|h(y) - h(y')\|_H}{\|y - y'\|_{\mathbb{R}^n}},
\end{equation*}
\begin{equation*}
R(\mathcal{H}, Y) = \sup_{y, y' \in Y,~y \ne y'} \mathbb{E}_{\bm{\gamma}} \left[ \sup_{h \in \mathcal{H}} \frac{\left\langle \bm{\gamma}, h(y) - h(y')\right\rangle_H}{\|y - y'\|_{\mathbb{R}^n}}\right].
\end{equation*}
\end{definition}
\bigskip

The next result establishes useful relationships between the quantities introduced above.

\begin{theorem}\label{thm:chain_rule}
Let $Y \subset \mathbb{R}^n$ be a finite set, $H$ a Hilbert space and $\mathcal{H}$ a finite class of functions $h: Y \rightarrow H$.
Then, there are universal constants $C_1$ and $C_2$ such that, for any $y_0 \in Y$:
\begin{equation*}
G(\mathcal{H}(Y)) \le C_1 L(\mathcal{H}, Y)G(Y) + C_2R(\mathcal{H}, Y)D(Y) + G(\mathcal{H}(y_0)).
\end{equation*}
\end{theorem}
\begin{proof}
This result is a direct extension of Theorem 2 in \citet{maurer2014chain} for $H$-valued functions.
The only part in the proof depending on the dimensionality of $H$ is Theorem 4 in the same paper, whose extension to any Hilbert space in \Cref{thm:TIS} is proved in the present paper.
Indeed, considering $X_y = (\sqrt{2}/\pi L(F, Y)) \sup_{f \in F} \left\langle \gamma, f(y)\right\rangle$ (using the same notation as in \citet{maurer2014chain} allows to finish the proof like in the finite dimensional case.
\end{proof}

Let $\mathcal{H}_{1, s}'$ be the set of functions from $(\mathcal{X}_0)^n$ to $\mathbb{R}^{np}$ that take as input $S = (x_1, \ldots, x_n)$ and return $(f(x_1), \ldots, f(x_n))$, $f \in \mathcal{H}_{1, s}$.
Let $Y = \mathcal{H}_{1, s}'(S) \subset \mathbb{R}^{np}$, and $H = (\mathcal{X}_0)^n$, which is a Hilbert space.
Let $\mathcal{H} = \mathcal{H}_{2, t}'$ be the set of functions from $\mathbb{R}^{np}$ to $(\mathcal{X}_0)^n$ that take as input $(y_1, \ldots, y_n)$ and return $(g(y_1), \ldots, g(y_n))$, $g \in \mathcal{H}_{2, t}$.
Finally, let $y_0 = (0_{\mathbb{R}^p}, \ldots, 0_{\mathbb{R}^p})$ (it actually belongs to $\mathcal{H}_{1, s}'(S)$ since the null function is in $\mathcal{H}_{1, s}'$).
\Cref{thm:chain_rule} entails that:
\begin{equation}
G\Big(\mathcal{H}_{2, t}'\big(\mathcal{H}_{1, s}'(S)\big)\Big) \le C_1 L\Big(\mathcal{H}_{2, t}', \mathcal{H}_{1, s}'(S)\Big)G\Big(\mathcal{H}_{1, s}'(S)\Big) + C_2R\Big(\mathcal{H}_{2, t}', \mathcal{H}_{1, s}'(S)\Big)D\Big(\mathcal{H}_{1, s}'(S)\Big) + G\Big(\mathcal{H}_{2, t}'(0)\Big),
\end{equation}
and
\begin{equation}\label{eq:chain_rule}
\widehat{\mathscr{G}}_n\Big(\mathcal{H}_{s, t}(S)\Big) \le C_1 L\Big(\mathcal{H}_{2, t}', \mathcal{H}_{1, s}'(S)\Big) \widehat{\mathscr{G}}_n\Big(\mathcal{H}_{1, s}(S)\Big) + \frac{C_2}{n} R\Big(\mathcal{H}_{2, t}', \mathcal{H}_{1, s}'(S)\Big)D\Big(\mathcal{H}_{1, s}'(S)\Big) + \frac{1}{n} G\Big(\mathcal{H}_{2, t}'(0)\Big).
\end{equation}

We now bound each term appearing on the right-hand side.

%%%%%%%%%%%%%%%%%%%%%%%%%%%%%%%%%%%%%
%  Bounding the Lipschitz Constant  %
%%%%%%%%%%%%%%%%%%%%%%%%%%%%%%%%%%%%%

\noindent {\bf Bounding $L\Big(\mathcal{H}_{2, t}', \mathcal{H}_{1, s}'(S)\Big)$.}
Consider the following hypothesis, denoting by $\|. \|_*$ the operator norm of any bounded linear operator.
\begin{assumption}\label{hyp:lip_2}
There exists a constant $L<+\infty$ such that: $\forall (y, y') \in \mathbb{R}^p$,
$$
\big\| \mathcal{K}_2(y, y) - 2 \mathcal{K}_2(y, y') + \mathcal{K}_2(y', y') \big\|_* \le L^2~\|y - y'\|_{\mathbb{R}^p}^2.
$$
\end{assumption}
This assumption is not too much compelling since it is enough for $\mathcal{K}_2$ to be the sum of $M$ decomposable kernels $k_m(\cdot, \cdot) A_m$ such that the scalar feature maps $\phi_m$ are $L_m$-Lipschitz (the feature map of the Gaussian kernel with bandwidth $1 / (2\sigma^2)$ has Lipschitz constant $1 / \sigma$ for instance), and the $A_m$ operators have finite operator norms $\sigma_m$.
Indeed, we would have then: $\forall z \in \mathcal{X}_0$,
\begin{align*}
\left\|\Big(\mathcal{K}_2(y, y) - 2 \mathcal{K}_2(y, y') + \mathcal{K}_2(y', y')\Big)z\right\|_{\mathcal{X}_0} &= \left\| \left(\sum_{m=1}^M \|\phi_m(y) - \phi_m(y')\|^2 A_m\right)z\right\|_{\mathcal{X}_0},\\
&\le \sum_{m=1}^M \|\phi_m(y) - \phi_m(y')\|^2 \sigma_m~\|z\|_{\mathcal{X}_0},\\
\left\|\Big(\mathcal{K}_2(y, y) - 2 \mathcal{K}_2(y, y') + \mathcal{K}_2(y', y')\Big)z\right\|_{\mathcal{X}_0} &\le \left(\sum_{m=1}^M L_m^2 \sigma_m\right) \| y - y'\|_{\mathbb{R}^p}^2~\|z\|_{\mathcal{X}_0},\\
\left\|\mathcal{K}_2(y, y) - 2 \mathcal{K}_2(y, y') + \mathcal{K}_2(y', y')\right\|_* &\le \left(\sum_{m=1}^M L_m^2 \sigma_m\right) \| y - y'\|_{\mathbb{R}^p}^2.
\end{align*}

Let $\mathcal{K}_2$ satisfy \Cref{hyp:lip_2}, $g \in \mathcal{H}_{2, t}'$ and $(\bm{y}, \bm{y'}) \in \mathbb{R}^{np}$. We have
\begin{align}
\left\|g(\bm{y}) - g(\bm{y'})\right\|_{(\mathcal{X}_0)^n}^2 &= \sum_{i=1}^n \left\| g(y_i) - g(y_i')\right\|_{\mathcal{X}_0}^2,\label{eq:bar0}\\
&= \sum_{i=1}^n \left\langle g(y_i) - g(y_i'), g(y_i) - g(y_i')\right\rangle_{\mathcal{X}_0},\label{eq:bar1}\\
&= \sum_{i=1}^n \left\langle {\mathcal{K}_2}_{y_i}(g(y_i) - g(y_i')), g\right\rangle_{\mathcal{H}_2} - \left\langle {\mathcal{K}_2}_{y_i'}(g(y_i) - g(y_i')), g\right\rangle_{\mathcal{H}_2},\label{eq:bar2}\\
&\le \|g\|_{\mathcal{H}_2} \sum_{i=1}^n~\left\|{\mathcal{K}_2}_{y_i}(g(y_i) - g(y_i')) - {\mathcal{K}_2}_{y_i'}(g(y_i) - g(y_i'))\right\|_{\mathcal{H}_2},\label{eq:bar3}\\
&\le t~\sum_{i=1}^n~\sqrt{\left\langle g(y_i) - g(y_i'), \big(\mathcal{K}_2(y_i, y_i) - 2 \mathcal{K}_2(y_i, y_i') + \mathcal{K}_2(y_i', y_i')\big)(g(y_i) - g(y_i'))\right\rangle_{\mathcal{X}_0}},\label{eq:bar4}\\
%
% &\le t~\sum_{i=1}^n~\left\| g(y_i) - g(y_i') \right\|_{\mathcal{X}_0} \sqrt{\left\| \big(\mathcal{K}_2(y_i, y_i) - 2 \mathcal{K}_2(y_i, y_i') + \mathcal{K}_2(y_i', y_i')\big)\right\|_*},\\
%
&\le Lt~\sum_{i=1}^n~\left\| g(y_i) - g(y_i') \right\|_{\mathcal{X}_0} \|y_i - y_i'\|_{\mathbb{R}^p},\label{eq:bar5}\\
%
% &\le Lt~\sqrt{\left(\sum_{i=1}^n~\left\| g(y_i) - g(y_i') \right\|_{\mathcal{X}_0}^2 \right) \left( \sum_{i=1}^n \|y_i - y_i'\|_{\mathbb{R}^p}^2 \right)},\\
%
\left\|g(\bm{y}) - g(\bm{y'})\right\|_{(\mathcal{X}_0)^n}^2 &\le Lt~\left\|g(\bm{y}) - g(\bm{y'})\right\|_{(\mathcal{X}_0)^n} \left\|\bm{y} - \bm{y'}\right\|_{\mathbb{R}^{np}},\label{eq:bar6}\\
\left\|g(\bm{y}) - g(\bm{y'})\right\|_{(\mathcal{X}_0)^n} &\le Lt~ \left\|\bm{y} - \bm{y'}\right\|_{\mathbb{R}^{np}},\label{eq:bar7}
\end{align}
where \eqref{eq:bar2} results from the reproducing property in vv-RKHSs (see Eq. (2.1) in \citet{micchelli2005learning}), \eqref{eq:bar3} follows from Cauchy-Schwarz inequality, \eqref{eq:bar4} is again a consequence of the reproducing property (Eq. (2.3) in \citet{micchelli2005learning}), \eqref{eq:bar5} can be deduced from \Cref{hyp:lip_2} and \eqref{eq:bar6} is a consequence of Cauchy-Schwarz inequality as well.
Hence, we finally have:
\begin{equation}\label{eq:lip}
L\Big(\mathcal{H}_{2, t}', \mathcal{H}_{1, s}'(S)\Big) \le L\Big(\mathcal{H}_{2, t}', \mathbb{R}^{np}\Big) \le Lt.
\end{equation}

%%%%%%%%%%%%%%%%%%%%%%%%%%%%%%%%%%%%%%%%
%  Bounding the Gaussian of a vv-RKHS  %
%%%%%%%%%%%%%%%%%%%%%%%%%%%%%%%%%%%%%%%%

\noindent {\bf Bounding $\widehat{\mathscr{G}}_n\Big(\mathcal{H}_{1, s}'(S)\Big)$.} Consider the assumption below.
\begin{assumption}\label{hyp:trace_1_supp}
There exists a constant $K<+\infty$ such that: $\forall x \in \mathcal{X}_0$,
$$
\text{\normalfont \textbf{Tr}}\Big(\mathcal{K}_1(x, x)\Big) \le Kp.
$$
\end{assumption}
This assumption is mild as well, since the sum of $M$ decomposable kernels $k_m(\cdot, \cdot) A_m$ such that the scalar kernels are bounded by $\kappa_m$ (as $X$ is supposed to be bounded, any continuous kernel is valid).
Indeed, we have: $\forall x \in \mathcal{X}_0$,
\begin{equation*}
\textbf{Tr}\Big(\mathcal{K}_1(x, x)\Big) =  \sum_{m=1}^M k_m(x, x)~\textbf{Tr}(A_m) \le \left(\sum_{m=1}^M \kappa_m \| A_m \|_{\infty}\right) p.
\end{equation*}
Let the OVK $\mathcal{K}_1$ satisfy \Cref{hyp:trace_1_supp} and be such that $\mathcal{H}_1$ is separable.
We then know that there exists $\Phi \in \mathcal{L}(\ell_2(\mathbb{N}), \mathbb{R}^p)$ such that: $\forall (x, x') \in \mathcal{X}_0,~\mathcal{K}_1(x, x') = \Phi(x)\Phi^*(x')$ and $\forall f \in \mathcal{H}_1, \exists u \in \ell_2(\mathbb{N})$ such that $f(\cdot) = \Phi(\cdot)u,~~\|f\|_{\mathcal{H}_1} = \|u\|_{\ell_2}$ (see \citet{micchelli2005learning}).
We have:
\begin{align}
n~\widehat{\mathscr{G}}_n\Big(\mathcal{H}_{1, s}'(S)\Big) &= \mathbb{E}_{\bm{\gamma}}\left[ \sup_{f \in \mathcal{H}_{1, s}} \sum_{i=1}^n \left\langle \bm{\gamma}_i, f(x_i)\right\rangle_{\mathbb{R}^p}\right],\label{eq:Rad_ovk_1}\\
&= \mathbb{E}_{\bm{\gamma}}\left[ \sup_{\|u\|_{\ell_2} \le s} \sum_{i=1}^n \sum_{k=1}^p \gamma_{i, k}, \left\langle \Phi(x_i)u, e_k\right\rangle_{\mathbb{R}^p}\right],\label{eq:Rad_ovk_2}\\
&= \mathbb{E}_{\bm{\gamma}}\left[ \sup_{\|u\|_{\ell_2} \le s} \left\langle u, \sum_{i=1}^n \sum_{k=1}^p \gamma_{i, k} \Phi^*(x_i) e_k\right\rangle_{\ell_2}\right],\label{eq:Rad_ovk_3}\\
&\le s~\mathbb{E}_{\bm{\gamma}}\left[ \left\|\sum_{i=1}^n \sum_{k=1}^p \gamma_{i, k} \Phi^*(x_i) e_k \right\|_{\ell_2} \right],\label{eq:Rad_ovk_4}\\
&\le s~\sqrt{\mathbb{E}_{\bm{\gamma}}\left[ \left\|\sum_{i=1}^n \sum_{k=1}^p \gamma_{i, k} \Phi^*(x_i) e_k \right\|_{\ell_2}^2 \right]},\label{eq:Rad_ovk_5}\\
&\le s~\sqrt{\sum_{i=1}^n \sum_{k=1}^p \left\langle \mathcal{K}(x_i, x_i) e_k, e_k \right\rangle_{\mathbb{R}^p}},\label{eq:Rad_ovk_6}\\
&\le s~\sqrt{\sum_{i=1}^n \textbf{Tr}\Big(\mathcal{K}_1(x_i, x_i)\Big)},\label{eq:Rad_ovk_7}\\
n~\widehat{\mathscr{G}}_n(\mathcal{H}_{1, s}'(S)) &\le s\sqrt{n K p},\label{eq:Rad_ovk_8}
\end{align}
where \eqref{eq:Rad_ovk_4} follows from Cauchy-Schwarz inequality, \eqref{eq:Rad_ovk_5} from Jensen's inequality, \eqref{eq:Rad_ovk_6} results from the orthogonality of the Gaussian variables introduced and \eqref{eq:Rad_ovk_8} from \Cref{hyp:trace_1_supp}.
Finally, we have:
\begin{equation}\label{eq:gauss}
\widehat{\mathscr{G}}_n\Big(\mathcal{H}_{1, s}'(S)\Big) \le s\sqrt{\frac{K p}{n}}.
\end{equation}

%%%%%%%%%%%%%%%%%%%%%%%%%%%%%%%
%  Bounding the Gaussian sup  %
%%%%%%%%%%%%%%%%%%%%%%%%%%%%%%%

\noindent {\bf Bounding $R\Big(\mathcal{H}_{2, t}', \mathcal{H}_{1, s}'(S)\Big)$.}
Consider the following hypothesis.
\begin{assumption}\label{hyp:trace_2_supp}
There exists a constant $L<+\infty$ such that: $\forall (y, y') \in \mathbb{R}^p$,
$$
\text{\normalfont \textbf{Tr}}\Big( \mathcal{K}_2(y, y) - 2 \mathcal{K}_2(y, y') + \mathcal{K}_2(y', y') \Big) \le L^2~\|y - y'\|_{\mathbb{R}^p}^2.
$$
\end{assumption}
Suppose that the OVK $\mathcal{K}_2$ is the sum of $M$ decomposable kernels $k_m(\cdot, \cdot) A_m$ such that the scalar feature maps $\phi_m$ are $L_m$-Lipschitz and the $A_m$ operators are trace class.
Then, we have: $\forall (y, y') \in \mathbb{R}^p$,
\begin{equation*}
\textbf{Tr}\Big(\mathcal{K}_2(y, y) - 2 \mathcal{K}_2(y, y') + \mathcal{K}_2(y', y') \Big) =  \sum_{m=1}^M \|\phi_m(y) - \phi_m(y')\|^2~\textbf{Tr}(A_m) \le \left(\sum_{m=1}^M L_m^2 \textbf{Tr}(A_m) \right) \|y - y'\|_{\mathbb{R}^p}^2.
\end{equation*}
Note also that \Cref{hyp:trace_2_supp} is stronger than \Cref{hyp:lip_2}, since $\|A\|_* \le \textbf{Tr}(A)$ for any trace class operator $A$.\\

Let the OVK $\mathcal{K}_2$ satisfy \Cref{hyp:trace_2_supp} and be such that $\mathcal{H}_2$ is separable.
We then know that there exists $\Psi \in \mathcal{L}(\ell_2(\mathbb{N}), \mathcal{X}_0)$ such that $\forall (y, y') \in \mathbb{R}^p,~\mathcal{K}_2(y, y') = \Psi(y)\Psi^*(y')$ and $\forall g \in \mathcal{H}_2, \exists v \in \ell_2(\mathbb{N})$ such that $g(\cdot) = \Psi(\cdot)v,~~\|g\|_{\mathcal{H}_2} = \|v\|_{\ell_2}$.
We have:
\begin{align}
\mathbb{E}_{\bm{\gamma}}\left[ \sup_{g \in \mathcal{H}_{2, t}} \left\langle \bm{\gamma}_i, g(\bm{y} - g(\bm{y'})\right\rangle_{\mathcal{X}_0^n}\right] &= \mathbb{E}_{\bm{\gamma}}\left[ \sup_{g \in \mathcal{H}_{2, t}} \sum_{i=1}^n \sum_{k=1}^\infty \gamma_{i, k} \left\langle (\Psi(y_i) - \Psi(y'_i))v, e_k\right\rangle_{\mathcal{X}_0}\right],\label{eq:Radd_ovk_1}\\
&= \mathbb{E}_{\bm{\gamma}}\left[ \sup_{g \in \mathcal{H}_{2, t}} \left\langle \sum_{i=1}^n \sum_{k=1}^\infty \gamma_{i, k} (\Psi^*(y_i) - \Psi^*(y'_i))e_k, v\right\rangle_{\ell_2}\right],\label{eq:Radd_ovk_2}\\
&\le t~\sqrt{\mathbb{E}_{\bm{\gamma}} \left\| \sum_{i=1}^n \sum_{k=1}^\infty \gamma_{i, k} (\Psi^*(y_i) - \Psi^*(y'_i))e_k\right\|_{\ell_2}^2},\label{eq:Radd_ovk_3}\\
&\le t~\sqrt{\sum_{i=1}^n \textbf{Tr}\Big(\mathcal{K}_2(y_i, y_i) - 2 \mathcal{K}_2(y_i, y'_i) + \mathcal{K}_2(y_i', y_i')\Big)},\label{eq:Radd_ovk_4}\\
\mathbb{E}_{\bm{\gamma}}\left[ \sup_{g \in \mathcal{H}_{2, t}} \left\langle \bm{\gamma}_i, g(\bm{y} - g(\bm{y'})\right\rangle_{\mathcal{X}_0^n}\right] &\le tL~\|\bm{y} - \bm{y'}\|_{\mathbb{R}^{np}},\label{eq:Radd_ovk_5}
\end{align}
where only \Cref{hyp:trace_2_supp} and arguments previously involved have been used.
Finally, we get:
\begin{equation}\label{eq:rrr}
R\Big(\mathcal{H}_{2, t}', \mathcal{H}_{1, s}'(S)\Big) \le R\Big(\mathcal{H}_{2, t}', \mathbb{R}^{np}\Big) \le tL.
\end{equation}

%%%%%%%%%%%%%%%%%%%%%%%%%%%
%  Bounding the diameter  %
%%%%%%%%%%%%%%%%%%%%%%%%%%%

\noindent {\bf Bounding $D\Big(\mathcal{H}_{1, s}'(S)\Big)$.}
Consider the assumption below.
\begin{assumption}\label{hyp:bound}
There exists $\kappa<+\infty$ such that: $\forall x \in S$,
$$
\left\| \mathcal{K}_1(x, x) \right\|_* \le \kappa^2.
$$
\end{assumption}
This assumption is easily fulfilled, since $X$ is almost surely bounded.
Indeed, any ov-kernel which is the (finite) sum of decomposable kernels with continuous scalar kernels fulfills it.
Note also that it is a weaker assumption than \Cref{hyp:trace_1_supp}, since one could choose $\kappa = \sqrt{Kp}$.\\

Let $\mathcal{K}_1$ satisfy \Cref{hyp:bound} and $(\bm{y}, \bm{y'}) \in \mathcal{H}_{1, s}'(S)$.
There exists $(f, f') \in \mathcal{H}_{1, s}$ such that $\bm{y} = (f(x_1), \ldots, f(x_n))$ and $\bm{y'} = (f'(x_1), \ldots, f'(x_n))$.
We have:
\begin{align}
\left\| \bm{y} - \bm{y'} \right\|_{\mathbb{R}^{np}}^2 &= \sum_{i=1}^n \left\| f(x_i) - f'(x_i) \right\|_{\mathbb{R}^p}^2,\\
& \le \sum_{i=1}^n \left(\left\| f(x_i) \right\|_{\mathbb{R}^p} + \left\| f'(x_i) \right\|_{\mathbb{R}^p} \right)^2,\\
&\le \sum_{i=1}^n \left( \left\| f \right\|_{\mathcal{H}_1} \left\| \mathcal{K}_1(x_i, x_i) \right\|_*^{1/2} + \left\| f' \right\|_{\mathcal{H}_1} \left\| \mathcal{K}_1(x_i, x_i) \right\|_*^{1/2} \right)^2,\label{eq:ovk_bound}\\
\left\| \bm{y} - \bm{y'} \right\|_{\mathbb{R}^{np}}^2 & \le 4 \kappa^2 s^2 n,
\end{align}
where \eqref{eq:ovk_bound} follows from Eq. (f) of Proposition 2.1 in \citet{micchelli2005learning}.
Finally, we get:
\begin{equation}\label{eq:diam}
D\Big(\mathcal{H}_{1, s}', S\Big) \le 2 \kappa s \sqrt{n}.
\end{equation}

%%%%%%%%%%%%%%%%%%%%%%%%%%%%%%%%
%  Bounding the last Gaussian  %
%%%%%%%%%%%%%%%%%%%%%%%%%%%%%%%%

\noindent {\bf Bounding $G\Big(\mathcal{H}_{2, t}'(0)\Big)$.}
We introduce the following assumption.
\begin{assumption}\label{hyp:trace_0}
$\mathcal{K}_2(0, 0)$ is trace class.
\end{assumption}
Then, using the same arguments as for \eqref{eq:Rad_ovk_7}, we get:
\begin{equation}
n~G\Big(\mathcal{H}_{2, t}'(0)\Big) \le t \sqrt{n~\textbf{Tr}\Big(\mathcal{K}_2(0, 0)\Big)}, \text{\qquad or \qquad} G\Big(\mathcal{H}_{2, t}'(0)\Big) \le t\sqrt{\frac{\textbf{Tr}\big(\mathcal{K}_2(0, 0)\big)}{n}}.
\end{equation}

Rather than shifting the kernel $\widetilde{\mathcal{K}}_2(y, y') = \mathcal{K}_2(y, y') - \mathcal{K}_2(0, 0)$, one could consider that \Cref{hyp:trace_0} is always satisfied.
In addition, we have $\textbf{Tr}\Big(\widetilde{\mathcal{K}}_2(0, 0)\Big) = 0$ and consequently $G\Big(\mathcal{H}_{2, t}'(0)\Big) \le 0$.

%%%%%%%%%%%%%%%%%
%  Final Bound  %
%%%%%%%%%%%%%%%%%

\subsubsection{Final Argument}

Now, combining inequalities \eqref{eq:Rade_1}, \eqref{eq:foo4}, \eqref{eq:chain_rule}, \eqref{eq:lip}, \eqref{eq:gauss}, \eqref{eq:rrr}, \eqref{eq:diam} and defining $C_0 \coloneqq 8 \sqrt{\pi} (C_1 + 2 C_2)$, for any $ \delta \in (0,1)$, we have with probability at least $1 - \delta$:
\begin{equation*}\label{eq:final_bound}
\epsilon(\hat{h}_n) - \epsilon^* \le C_0 L M s t \sqrt{\frac{K p}{n}} + 24 M^2 \sqrt{\frac{\ln \frac{2}{\delta}}{2n}}.
\end{equation*}
\qed